%% file: main.tex
\RequirePackage{etex}
\documentclass[10pt]{article}

\usepackage{arxiv_pkg}
\usepackage{general_def}
\usepackage{specific_def}

\title{Federated Composite Optimization}

\author{Honglin Yuan \\
Stanford University\thanks{Based on work performed at Google Research.} \\
\texttt{yuanhl@cs.stanford.edu} 
\and 
Manzil Zaheer \\
Google Research \\
\texttt{manzilz@google.com} \and
Sashank Reddi \\
Google Research \\
\texttt{sashank@google.com} 
}

\begin{document}
\date{}
\maketitle

\input{abstract}
\input{introduction}
\input{addl_related_work}
\input{preliminaries}

\input{algorithms}
\input{theoretical_results}
\input{proof_sketch}
\input{experiments}

\input{conclusion}
% \clearpage
\input{ack}

% \clearpage
\bibliographystyle{plainnat}
\bibliography{fed_cstr}

\clearpage

\begin{appendices}
\crefalias{section}{appendix}
\crefalias{subsection}{appendix}
\crefalias{subsubsection}{appendix}
\input{pre_appendix}
\begin{spacing}{1.05}
 \listofappendices
\end{spacing}
\clearpage
\input{expr_details}
\input{technicalities}

\input{analysis_bdd_grad}

\input{analysis_quad}
\input{analysis_small_lr}
\clearpage % don't remove
\end{appendices}

\end{document}

%% file: abstract.tex
\begin{abstract}
  Federated Learning (FL) is a distributed learning paradigm that scales on-device learning collaboratively and privately. 
  Standard FL algorithms such as \fedavg are primarily geared towards \emph{smooth unconstrained} settings. 
  In this paper, we study the \emph{Federated Composite Optimization} (FCO) problem, in which the loss function contains a non-smooth regularizer. 
  Such problems arise naturally in FL applications that involve sparsity, low-rank, monotonicity, or more general constraints. 
  We first show that straightforward extensions of primal algorithms such as \fedavg are not well-suited for FCO since they suffer from the ``curse of primal averaging,'' resulting in poor convergence.
  As a solution, we propose a new primal-dual algorithm, \emph{Federated Dual Averaging} (\feddualavg), which by employing a novel server dual averaging procedure
  circumvents the curse of primal averaging.
  Our theoretical analysis and empirical experiments demonstrate that \feddualavg outperforms the other baselines.
\end{abstract}

%% file: introduction.tex
% !TEX root = main.tex  
\section{Introduction}
Federated Learning (FL, \citealt{Konecny.McMahan.ea-NeurIPS15,McMahan.Moore.ea-AISTATS17}) is a novel distributed learning paradigm in which a large number of clients collaboratively train a shared model without disclosing their private local data. 
The two most distinct features of FL, when compared to classic distributed learning settings, are (1) heterogeneity in data amongst the clients and (2) very high cost to communicate with a client. Due to these aspects, classic distributed optimization algorithms have been rendered ineffective in FL settings \citep{Kairouz.McMahan.ea-arXiv19}. Several algorithms specifically catered towards FL settings have been proposed to address these issues. The most prominent amongst them is Federated Averaging (\fedavg) algorithm, which by employing local SGD updates, significantly reduces the communication overhead under moderate client heterogeneity. Several follow-up works have focused on improving the \fedavg in various ways (e.g., \citealt{Li.Sahu.ea-MLSys20,Karimireddy.Kale.ea-ICML20, Reddi.Charles.ea-20,Yuan.Ma-NeurIPS20}). 

\input{fig1}

Existing FL research primarily focuses on the \emph{unconstrained smooth} objectives;
however, many FL applications involve non-smooth objectives. 
Such problems arise naturally in the context of regularization (e.g., sparsity, low-rank, monotonicity, or additional constraints on the model). For instance, consider the problem of cross-silo biomedical FL, where medical organizations collaboratively aim to learn a global model on their patients' data without sharing. In such applications, sparsity constraints are of paramount importance due to the nature of the problem as it involves only a few data samples (e.g., patients) but with very high dimensions (e.g., fMRI scans). 
For the purpose of illustration, in \cref{fig:haxby:simplified}, we present results on a federated sparse ($\ell_1$-regularized) logistic regression task for an fMRI dataset \citep{Haxby-01}.
As shown, using a federated approach that can handle non-smooth objectives enables us to find a highly accurate sparse solution without sharing client data.

In this paper, we propose to study the  \emph{Federated Composite Optimization} (FCO) problem. As in standard FL, the losses  are distributed to $M$ clients. In addition, we assume all the clients share the same, possibly non-smooth, non-finite regularizer $\psi$. Formally, (FCO) is of the following form
\begin{equation}
\min_{w \in \reals^d} \Phi(w) := F(w) + \psi(w) := \frac{1}{M} \sum_{m=1}^M F_m(w) + \psi(w), 
\tag{FCO}
\label{FCO}
\end{equation}
where $F_m(w) := \expt_{\xi^m \sim \mathcal{D}_m} f(w; \xi^m)$ is the loss at the $m$-th client, assuming $\mathcal{D}_m$ is its local data distribution. We assume that each client $m$ can access $\nabla f(w; \xi^m)$ by drawing independent samples $\xi^m$ from its local distribution $\mathcal{D}_m$. 
Common examples of $\psi(w)$ include $\ell_1$-regularizer or more broadly $\ell_p$-regularizer, nuclear-norm regularizer (for matrix variable), total variation (semi-)norm, etc. 
The (FCO) reduces to the standard federated optimization problem if $\psi \equiv 0$. 
The (FCO) also covers the constrained federated optimization if one takes $\psi$ to be the following constraint characteristics
$\chi_{\cstr}(w) := %$ if $w \in \mathcal{C}$ or $+\infty$ otherwise.
  \begin{cases} 
      0 & \text{if $w \in \mathcal{C}$}, \\
      +\infty & \text{if $w \notin \mathcal{C}$}.
  \end{cases}
 $

Standard FL algorithms such as \fedavg (see \cref{alg:fedavg}) and its variants (e.g., \citealt{Li.Sahu.ea-MLSys20,Karimireddy.Kale.ea-ICML20}) are primarily tailored to \emph{smooth unconstrained} settings, and are therefore, not well-suited for FCO. 
% Recall that in \fedavg, each client runs a local thread of SGD and periodically synchronizes with the orchestration server. %(see \cref{sec:fedavg} for a formal review).
The most straightforward extension of \fedavg towards (FCO) is to apply local subgradient method \citep{Shor-85} in lieu of SGD. 
This approach is largely ineffective due to the intrinsic slow convergence of subgradient approach \citep{Boyd.Xiao.ea-03}, which is also demonstrated in \cref{fig:haxby:simplified} (marked \fedavg ($\partial$)).

A more natural extension of \fedavg is to replace the local SGD with proximal SGD (\citealt{Parikh.Boyd-FnT14}, a.k.a. projected SGD for constrained problems), or more generally, mirror descent \cite{Duchi.Shalev-shwartz.ea-COLT10}. 
We refer to this algorithm as \emph{\fedmidfull} (\fedmid, see \cref{alg:fedmid}). 
The most noticeable drawback of a primal-averaging method like \fedmid is the ``curse of primal averaging,'' where the desired regularization of FCO may be rendered completely ineffective due to the server averaging step typically used in FL. For instance, consider a $\ell_1$-regularized logistic regression setting. 
Although each client is able to obtain a sparse solution, simply averaging the client states will inevitably yield a dense solution.
See \cref{fig:curse_of_avg} for an illustrative example.

\input{fig2}

To overcome this challenge, we propose a novel primal-dual algorithm named \emph{\feddualavgfull} (\feddualavg, see \cref{alg:feddualavg}).
Unlike \fedmid (or its precursor \fedavg), the server averaging step of \feddualavg operates in the dual space instead of the primal. 
Locally, each client runs dual averaging algorithm \citep{Nesterov-MP09} by tracking of a pair of primal and dual states.
During communication, the dual states are averaged across the clients.
Thus, \feddualavg employs a novel double averaging procedure --- averaging of dual states across clients (as in \fedavg), and the averaging of gradients in dual space (as in the sequential dual averaging). 
Since both levels of averaging operate in the dual space, we can show that \feddualavg provably overcomes the curse of primal averaging.
Specifically, we prove that \feddualavg can attain significantly lower communication complexity when deployed with a large client learning rate. 

\paragraph{Contributions.} 
In light of the above discussion, let us summarize our key contributions below:
\begin{itemize}[leftmargin=*]
    \item We propose a generalized federated learning problem, namely \emph{Federated Composite Optimization} (FCO), with non-smooth regularizers and constraints.  
    \item We first propose a natural extension of \fedavg, namely \emph{\fedmidfull} (\fedmid). 
    We show that \fedmid can attain the mini-batch rate in the small client learning rate regime (\cref{subsec:small:client:lr}).
    We argue that \fedmid may suffer from the effect of ``curse of primal averaging,'' which results in poor convergence, especially in the large client learning rate regime (\cref{sec:curse}).
    \item We propose a novel primal-dual algorithm named \emph{\feddualavgfull} (\feddualavg), which provably overcomes the curse of primal averaging (\cref{sec:feddualavg}).
    Under certain realistic conditions, we show that by virtue of ``double averaging'' property, \feddualavg can have significantly lower communication complexity (\cref{sec:feddualavg-benefit}).
    \item We demonstrate the empirical performance of \fedmid and \feddualavg on various tasks, including $\ell_1$-regularization, nuclear-norm regularization, and various constraints in FL (\cref{sec:expr}).
\end{itemize}
\paragraph{Notations.}
We use $[n]$ to denote the set $\{1, \ldots, n\}$. 
We use $\langle \cdot , \cdot \rangle$ to denote the inner product, $\| \cdot\|$ to denote an arbitrary norm, and $\|\cdot\|_*$ to denote its dual norm, unless otherwise specified. 
We use $\| \cdot\|_2$ to denote the $\ell_2$ norm of a vector or the operator norm of a matrix, and $\|\cdot\|_A$ to denote the vector norm induced by positive definite matrix $A$, namely $\|w\|_A := \sqrt{ \left\langle w, Aw \right\rangle }$. 
For any convex function $g(w)$, we use $g^*(z)$ to denote its convex conjugate $g^*(z) := \sup_{w \in \reals^d}\{ \langle z, w \rangle - g(w) \}$.
We use $w^{\star}$ to denote the optimum of the problem (FCO).  
We use $\mathcal{O}, \Theta$ to hide multiplicative absolute constants only and $x \lesssim y$ to denote $x = \mathcal{O}(y)$.

%% file: fig1.tex
\begin{figure}
  \centering
  \centerline{\includegraphics[width=\columnwidth]{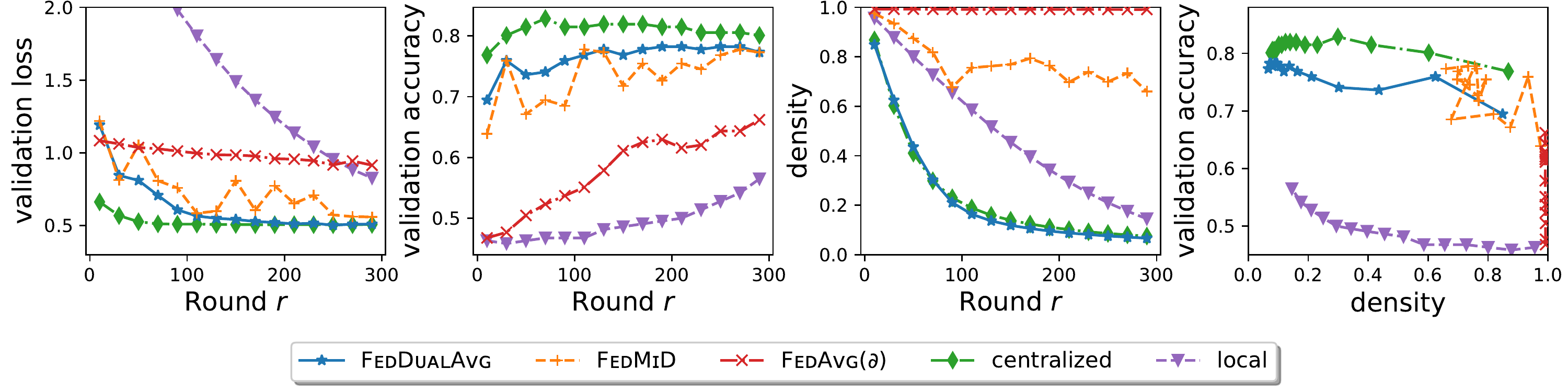}}
  \caption{
  \textbf{Results on sparse ($\ell_1$-regularized) logistic regression for a federated fMRI dataset based on  \citep{Haxby-01}.}
  \texttt{centralized} corresponds to training on the centralized dataset gathered from \textbf{all} the training clients.
  \texttt{local} corresponds to training on the local data from only \textbf{one} training client without communication. 
  \fedavg($\partial$) corresponds to running \fedavg algorithms with subgradient in lieu of SGD to handle the non-smooth $\ell_1$-regularizer. 
  \fedmid is another straightforward extension of \fedavg running local proximal gradient method (see \cref{sec:fedmid} for details). 
  We show that using our proposed algorithm \feddualavg, one can 1) achieve performance comparable to the \texttt{centralized} baseline without the need to gather client data, and 2) significantly outperforms the \texttt{local} baseline on the isolated data and the \fedavg baseline. 
  See \cref{mainsec:expr:fmri} for details.}
  \label{fig:haxby:simplified}
\end{figure}

%% file: fig2.tex
\begin{figure}
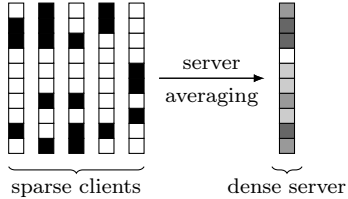
%[b]
  \centering
  \include{tikz_curse_of_avg_script}
  % \vspace{-2em}
\caption{\textbf{Illustration of ``curse of primal averaging''}. While each client of \fedmid can locate a sparse solution, simply averaging them will yield a much denser solution on the server side.}
\label{fig:curse_of_avg}
\end{figure}

%% file: tikz_curse_of_avg_script.tex
% !TEX root = main.tex  
\begin{tikzpicture}
  \filldraw[fill=black!0] (0.0, 0.0) rectangle (0.2, 0.2);
\filldraw[fill=black!100] (0.0, 0.2) rectangle (0.2, 0.4);
\filldraw[fill=black!0] (0.0, 0.4) rectangle (0.2, 0.6000000000000001);
\filldraw[fill=black!0] (0.0, 0.6000000000000001) rectangle (0.2, 0.8);
\filldraw[fill=black!0] (0.0, 0.8) rectangle (0.2, 1.0);
\filldraw[fill=black!0] (0.0, 1.0) rectangle (0.2, 1.2);
\filldraw[fill=black!0] (0.0, 1.2000000000000002) rectangle (0.2, 1.4000000000000001);
\filldraw[fill=black!100] (0.0, 1.4000000000000001) rectangle (0.2, 1.6);
\filldraw[fill=black!100] (0.0, 1.6) rectangle (0.2, 1.8);
\filldraw[fill=black!0] (0.0, 1.8) rectangle (0.2, 2.0);
\filldraw[fill=black!100] (0.4, 0.0) rectangle (0.6000000000000001, 0.2);
\filldraw[fill=black!0] (0.4, 0.2) rectangle (0.6000000000000001, 0.4);
\filldraw[fill=black!0] (0.4, 0.4) rectangle (0.6000000000000001, 0.6000000000000001);
\filldraw[fill=black!100] (0.4, 0.6000000000000001) rectangle (0.6000000000000001, 0.8);
\filldraw[fill=black!0] (0.4, 0.8) rectangle (0.6000000000000001, 1.0);
\filldraw[fill=black!0] (0.4, 1.0) rectangle (0.6000000000000001, 1.2);
\filldraw[fill=black!0] (0.4, 1.2000000000000002) rectangle (0.6000000000000001, 1.4000000000000001);
\filldraw[fill=black!100] (0.4, 1.4000000000000001) rectangle (0.6000000000000001, 1.6);
\filldraw[fill=black!100] (0.4, 1.6) rectangle (0.6000000000000001, 1.8);
\filldraw[fill=black!100] (0.4, 1.8) rectangle (0.6000000000000001, 2.0);
\filldraw[fill=black!100] (0.8, 0.0) rectangle (1.0, 0.2);
\filldraw[fill=black!100] (0.8, 0.2) rectangle (1.0, 0.4);
\filldraw[fill=black!0] (0.8, 0.4) rectangle (1.0, 0.6000000000000001);
\filldraw[fill=black!100] (0.8, 0.6000000000000001) rectangle (1.0, 0.8);
\filldraw[fill=black!0] (0.8, 0.8) rectangle (1.0, 1.0);
\filldraw[fill=black!0] (0.8, 1.0) rectangle (1.0, 1.2);
\filldraw[fill=black!0] (0.8, 1.2000000000000002) rectangle (1.0, 1.4000000000000001);
\filldraw[fill=black!100] (0.8, 1.4000000000000001) rectangle (1.0, 1.6);
\filldraw[fill=black!0] (0.8, 1.6) rectangle (1.0, 1.8);
\filldraw[fill=black!0] (0.8, 1.8) rectangle (1.0, 2.0);
\filldraw[fill=black!0] (1.2000000000000002, 0.0) rectangle (1.4000000000000001, 0.2);
\filldraw[fill=black!100] (1.2000000000000002, 0.2) rectangle (1.4000000000000001, 0.4);
\filldraw[fill=black!0] (1.2000000000000002, 0.4) rectangle (1.4000000000000001, 0.6000000000000001);
\filldraw[fill=black!0] (1.2000000000000002, 0.6000000000000001) rectangle (1.4000000000000001, 0.8);
\filldraw[fill=black!0] (1.2000000000000002, 0.8) rectangle (1.4000000000000001, 1.0);
\filldraw[fill=black!0] (1.2000000000000002, 1.0) rectangle (1.4000000000000001, 1.2);
\filldraw[fill=black!0] (1.2000000000000002, 1.2000000000000002) rectangle (1.4000000000000001, 1.4000000000000001);
\filldraw[fill=black!0] (1.2000000000000002, 1.4000000000000001) rectangle (1.4000000000000001, 1.6);
\filldraw[fill=black!100] (1.2000000000000002, 1.6) rectangle (1.4000000000000001, 1.8);
\filldraw[fill=black!100] (1.2000000000000002, 1.8) rectangle (1.4000000000000001, 2.0);
\filldraw[fill=black!0] (1.6, 0.0) rectangle (1.8, 0.2);
\filldraw[fill=black!0] (1.6, 0.2) rectangle (1.8, 0.4);
\filldraw[fill=black!100] (1.6, 0.4) rectangle (1.8, 0.6000000000000001);
\filldraw[fill=black!0] (1.6, 0.6000000000000001) rectangle (1.8, 0.8);
\filldraw[fill=black!100] (1.6, 0.8) rectangle (1.8, 1.0);
\filldraw[fill=black!100] (1.6, 1.0) rectangle (1.8, 1.2);
\filldraw[fill=black!0] (1.6, 1.2000000000000002) rectangle (1.8, 1.4000000000000001);
\filldraw[fill=black!0] (1.6, 1.4000000000000001) rectangle (1.8, 1.6);
\filldraw[fill=black!0] (1.6, 1.6) rectangle (1.8, 1.8);
\filldraw[fill=black!0] (1.6, 1.8) rectangle (1.8, 2.0);
\filldraw[fill=black!40.0] (3.6, 0.0) rectangle (3.8000000000000003, 0.2);
\filldraw[fill=black!60.0] (3.6, 0.2) rectangle (3.8000000000000003, 0.4);
\filldraw[fill=black!20.0] (3.6, 0.4) rectangle (3.8000000000000003, 0.6000000000000001);
\filldraw[fill=black!40.0] (3.6, 0.6000000000000001) rectangle (3.8000000000000003, 0.8);
\filldraw[fill=black!20.0] (3.6, 0.8) rectangle (3.8000000000000003, 1.0);
\filldraw[fill=black!20.0] (3.6, 1.0) rectangle (3.8000000000000003, 1.2);
\filldraw[fill=black!0.0] (3.6, 1.2000000000000002) rectangle (3.8000000000000003, 1.4000000000000001);
\filldraw[fill=black!60.0] (3.6, 1.4000000000000001) rectangle (3.8000000000000003, 1.6);
\filldraw[fill=black!60.0] (3.6, 1.6) rectangle (3.8000000000000003, 1.8);
\filldraw[fill=black!40.0] (3.6, 1.8) rectangle (3.8000000000000003, 2.0);

  \draw[decoration={brace,mirror,raise=5pt},decorate]
  (0,0) -- node[below=6pt] {\footnotesize sparse clients} (1.8,0);

  \draw[-latex] (2,1.0) -- node[above=0pt] {\footnotesize server} (3.4,1.0);
  \draw[-latex] (2,1.0) -- node[below=0pt] {\footnotesize averaging} (3.4,1.0);

  \draw[decoration={brace,mirror,raise=5pt},decorate]
  (3.5,0) -- node[below=6pt] {\footnotesize dense server} (3.9,0);
\end{tikzpicture}

%% file: addl_related_work.tex
% !TEX root = main.tex  
\subsection{Additional Related Work}
\label{sec:related_works}
\paragraph{Federated Learning.}
Recent years have witnessed a growing interest in various aspects of Federated Learning.
The early analysis of \fedavg preceded the inception of Federated Learning, which was studied under the names of parallel SGD and local SGD \citep{Zinkevich.Weimer.ea-NIPS10,Zhou.Cong-IJCAI18}
Early results on \fedavg mostly focused on the ``one-shot'' averaging case, in which the clients are only synchronized once at the end of the procedure (e.g., \citealt{Mcdonald.Mohri.ea-NIPS09,Shamir.Srebro-Allerton14,Rosenblatt.Nadler-16,Jain.Kakade.ea-COLT18,Godichon-Baggioni.Saadane-20}). 
The first analysis of general \fedavg  was established by \citep{Stich-ICLR19} for the homogeneous client dataset.
This result was improved by 
\citep{
Haddadpour.Kamani.ea-NeurIPS19,
Khaled.Mishchenko.ea-AISTATS20,
Woodworth.Patel.ea-ICML20,
Yuan.Ma-NeurIPS20} via tighter analysis and accelerated algorithms.
\fedavg has also been studied for non-convex objectives \citep{Zhou.Cong-IJCAI18,Haddadpour.Kamani.ea-ICML19,Wang.Joshi-18,Yu.Jin-ICML19,Yu.Jin.ea-ICML19,Yu.Yang.ea-AAAI19}. 
For heterogeneous clients, numerous recent papers 
\citep{
  Haddadpour.Kamani.ea-NeurIPS19,
  Khaled.Mishchenko.ea-AISTATS20,
  Li.Huang.ea-ICLR20,
  Koloskova.Loizou.ea-ICML20,
  Woodworth.Patel.ea-NeurIPS20} studied the convergence of \fedavg under various notions of heterogeneity measure.
Other variants of \fedavg have been proposed to overcome heterogeneity (e.g., \citealt{Mohri.Sivek.ea-ICML19,
    Zhang.Hong.ea-20,
    Li.Sahu.ea-MLSys20,
    Wang.Tantia.ea-ICLR20,
    Karimireddy.Kale.ea-ICML20,
    Reddi.Charles.ea-20,
    Pathak.Wainwright-NeurIPS20,
    Al-Shedivat.Gillenwater.ea-ICLR21}).
A recent line of work has studied the behavior of Federated algorithms for personalized multi-task objectives \citep{Smith.Chiang.ea-NIPS17,Hanzely.Hanzely.ea-NeurIPS20,T.Dinh.Tran.ea-NeurIPS20,Deng.Kamani.ea-20}
and meta-learning objectives \citep{Fallah.Mokhtari.ea-NeurIPS20,Chen.Luo.ea-19,Jiang.Konecny.ea-19}. 
Federated Learning techniques have been successfully applied in a broad range of practical applications \citep{Hard.Rao.ea-18,Hartmann.Suh.ea-19,Hard.Partridge.ea-INTERSPEECH20}. 
We refer readers to \citep{Kairouz.McMahan.ea-arXiv19} for a comprehensive survey of the recent advances in Federated Learning. 
However, none of the aforementioned work allows for non-smooth or constrained problems such as (FCO).
To the best of our knowledge, the present work is the first work that studies non-smooth or constrained problems in Federated settings.

Shortly after the initial preprint release of the present work, \citet{Tong.Liang.ea-20} proposed a related federated $\ell_0$-constrained problem (which does not belong to FCO due to the non-convexity of $\ell_0$), and two algorithms to solve (similar to \fedmid(\textsc{-OSP}) but with hard-thresholding instead). 
As in most hard-thresholding work, the convergence is weaker since it depends on the sparsity level $\tau$ (worsens as $\tau$ gets tighter).

\paragraph{Composite Optimization, Dual Averaging, and Mirror Descent.}
Composite optimization has been a classic problem in convex optimization, which covers a variety of statistical inference, machine learning, signal processing problems. 
Mirror Descent (MD, a generalization of proximal gradient method) and Dual Averaging (DA, a.k.a. lazy mirror descent) are two representative algorithms for convex composite optimization.
The \emph{Mirror Descent} (MD) method was originally introduced by \citet{Nemirovski.Yudin-83} for the constrained case and reinterpreted by  \citet{Beck.Teboulle-ORL03}. 
MD was generalized to the composite case by \citet{Duchi.Shalev-shwartz.ea-COLT10} under the name of \textsc{Comid}, though numerous preceding work had studied the special case of \textsc{Comid} under a variety of names such as gradient mapping \citep{Nesterov-MP13}, forward-backward splitting method (FOBOS,\citealt{Duchi.Singer-JMLR09}), iterative shrinkage and thresholding (ISTA, \citealt{Daubechies.Defrise.ea-04}), and truncated gradient \citep{Langford.Li.ea-JMLR09}.
The \emph{Dual Averaging} (DA) method was introduced by \citet{Nesterov-MP09} for the constrained case, which is also known as \emph{Lazy Mirror Descent} in the literature \citep{Bubeck-15}. The DA method was generalized to the composite (regularized) case by \citep{Xiao-JMLR10,Dekel.Gilad-Bachrach.ea-JMLR12} under the name of Regularized Dual Averaging, and extended by recent works \citep{Flammarion.Bach-COLT17,Lu.Freund.ea-SIOPT18} to account for non-Euclidean geometry induced by an arbitrary distance-generating function $h$. 
DA also has its roots in online learning \citep{Zinkevich-ICML03}, and is related to the follow-the-regularized-leader (FTRL) algorithms \citep{McMahan-AISTATS11}. 
Other variants of MD or DA (such as delayed / skipped proximal step) have been investigated to mitigate the expensive proximal oracles \citep{Mahdavi.Yang.ea-NeurIPS12,Yang.Lin.ea-ICML17}. 
% Composite optimization has also been studied for non-convex objective \citep{Attouch.Bolte.ea-MP13,Chouzenoux.Pesquet.ea-JOTA14,Bredies.Lorenz.ea-JOTA15,Li.Pong-SIOPT15}. 
% These works are typically limited to special cases due to the hardness of non-convex composite optimization, which is in sharp constrast to smooth non-convex settings.
% Unlike smooth unconstrained cases in which simple algorithms like SGD can readily work, the non-convex composite optimization are particularly challenging. 
% \citep{Agarwal.Duchi-NIPS11,Feyzmahdavian.Aytekin.e}
% In addition to MD and DA, there are other algorithms that are popular for certain types of composite optimization problems. For example, Frank-Wolfe method \citep{Frank.Wolfe-56,Jaggi-ICML13} solves constrained optimization with a linear optimization oracle, which is different from the proximal oracle applied by MD and DA. 
We refer readers to \citep{Flammarion.Bach-COLT17,Diakonikolas.Orecchia-SIOPT19} for more detailed discussions on the recent advances of MD and DA.

\paragraph{Classic Decentralized Consensus Optimization.}
A related distributed setting is the \emph{decentralized consensus optimization}, also known as 
\emph{multi-agent optimization} or 
\emph{optimization over networks} in the literature \citep{Nedich-FnT15}. 
Unlike the federated settings, in decentralized consensus optimization, each client can communicate every iteration, but the communication is limited to its graphic neighborhood.
Standard algorithms for unconstrained consensus optimization include decentralized (sub)gradient methods \citep{Nedic.Ozdaglar-TACON09,Yuan.Ling.ea-SIOPT16} 
and EXTRA \citep{Shi.Ling.ea-SIOPT15,Mokhtari.Ribeiro-JMLR16}. 
For constrained or composite consensus problems, people have studied both mirror-descent type methods (with primal consensus), {e.g.}, \citep{SundharRam.Nedic.ea-JOTA10,Shi.Ling.ea-TSP15,Rabbat-CAMSAP15,Yuan.Hong.ea-Automatica18,Yuan.Hong.ea-TACON20}; and dual-averaging type methods (with dual consensus), {e.g.,} \citep{Duchi.Agarwal.ea-TACON12,Tsianos.Lawlor.ea-CDC12,Tsianos.Rabbat-ACC12,Liu.Chen.ea-TSP18}.
In particular, the distributed dual averaging \citep{Duchi.Agarwal.ea-TACON12} has gained great popularity since its dual consensus scheme elegantly handles the constraints, and overcomes the technical difficulties of primal consensus, as noted by the original paper.
We identify that while the federated settings share certain backgrounds with the decentralized consensus optimization, the motivations, techniques, challenges, and results are quite dissimilar due to the fundamental difference of communication protocol, as noted by \citep{Kairouz.McMahan.ea-arXiv19}.
We refer readers to \citep{Nedich-FnT15} for a more detailed introduction to the classic decentralized consensus optimization.

%% file: preliminaries.tex
\section{Preliminaries}
\label{sec:prelim}
In this section, we review the necessary background for composite optimization and federated learning. 
A detailed technical exposition of these topics is relegated to \cref{sec:background}.

\subsection{Composite Optimization}
\label{sec:co}
Composite optimization covers a variety of statistical inference, machine learning, signal processing problems.
Standard (non-distributed) composite optimization is defined as
\begin{equation}
    \min_{w \in \reals^d} \quad \expt_{\xi \sim \mathcal{D}}f(w;\xi) + \psi(w),
    \tag{CO}
    \label{eq:CO}
\end{equation}
where $\psi$ is a non-smooth, possibly non-finite regularizer.

\paragraph{Proximal Gradient Method.}
A natural extension of SGD for (CO) is the following \emph{proximal gradient method} (PGM):
\begin{align}
    w_{t+1} \gets \prox_{\eta \psi} \left(w_t - \eta \nabla f(w_t;\xi_t)  \right)
    = \argmin_{w} \left( \eta \langle \nabla f(w_t;\xi_t) , w \rangle + \frac{1}{2} \|w - w_t\|_2^2   + \eta \psi(w)  \right). 
    \label{eq:pgm}
\end{align}
The sub-problem \cref{eq:pgm} can be motivated by optimizing a quadratic upper bound of $f$ together with the original $\psi$. 
This problem can often be efficiently solved by virtue of the special structure of $\psi$. 
For instance, one can verify that PGM reduces to projected gradient descent if $\psi$ is a constraint characteristic $ {\chi}_{\mathcal{C}}$, soft thresholding if $\psi(w) = \lambda \|w\|_1$, or weight decay if $\psi(w) := \lambda \|w\|_2^2$. 

\paragraph{Mirror Descent / Bregman-PGM.}
PGM can be generalized to the Bregman-PGM if one replaces the Euclidean proximity term by the general Bregman divergence, namely
\begin{equation}
    w_{t+1} \gets \argmin_{w} \left( \eta  \left\langle \nabla f(w_t; \xi_t), w \right \rangle + \eta \psi(w) + D_{h}(w, w_t)  \right),
     \label{eq:bpgm}
\end{equation}
where $h$ is a strongly convex distance-generating function, $D_h$ is the Bregman divergence which reduces to Euclidean distance if one takes $h(w) = \frac{1}{2} \|w\|_2^2$. 
We will still refer to this step as a proximal step for ease of reference.
 %$D_h(w,u) : = h(w) - h(u) - \langle \nabla h(u), w-u\rangle$. 
This general formulation \eqref{eq:bpgm} enables an equivalent primal-dual interpretation:
\begin{equation}
    w_{t+1} \gets \nabla (h + \eta \psi)^*(\nabla h(w_t) - \nabla f(w_t;\xi_t)).
    \label{eq:md:oneline}
\end{equation}
A common interpretation of \eqref{eq:md:oneline} is to decompose it into the following three sub-steps \citep{Nemirovski.Yudin-83}: 
\begin{enumerate}[(a),leftmargin=*]
    \item Apply $\nabla h$ to carry $w_t$ to a dual state (denoted as $z_t$)
    \item Update $z_t$ to $y_{t+1}$ with the gradient queried at $w_t$.
    \item Map $y_{t+1}$ back to primal via $\nabla (h+ \eta\psi)^*$
\end{enumerate}
This formulation is known as the \emph{composite objective mirror descent} (\textsc{Comid}, \citealt{Duchi.Shalev-shwartz.ea-COLT10}), or simply \emph{mirror descent} in the literature \citep{Flammarion.Bach-COLT17}. 
% We will refer to this step as ``proximal'' step or ``mirror descent'' step interchangeably hereinafter.

\paragraph{Dual Averaging.} 
An alternative approach for (CO) is the following \emph{dual averaging} algorithm \citep{Nesterov-MP09}:
\begin{equation}
    z_{t+1} \gets z_t - \eta \nabla f \left(\nabla (h + \eta t \psi)^*(z_t); \xi_t \right).
    \label{eq:da}
\end{equation}
Similarly, we can decompose \eqref{eq:da} into two sub-steps:
\begin{enumerate}[(a),leftmargin=*]
    \item Apply $\nabla (h + \eta t \psi)^*$ to map dual state $z_t$ to primal $w_t$. 
        Note that this sub-step can be reformulated into
        \begin{equation}
            w_t = \argmin_{w} \left( \left\langle - z_t, w  \right\rangle + \eta t \psi(w) + h(w)\right),
        \end{equation}
        which allows for efficient computation for many $\psi$, as in PGM.
    \item Update $z_t$ to $z_{t+1}$ with the gradient queried at $w_t$.
\end{enumerate}
Dual averaging is also known as the \emph{``lazy'' mirror descent} algorithm \citep{Bubeck-15} since it skips the forward mapping $(\nabla h)$ step. 
Theoretically, mirror descent and dual averaging often share the similar convergence rates for sequential \eqref{eq:CO} (e.g., for smooth convex $f$, c.f. \citealt{Flammarion.Bach-COLT17}).

\begin{remark}
    There are other algorithms that are popular for certain types of \eqref{eq:CO} problems. 
    For example, \emph{Frank-Wolfe} method \citep{Frank.Wolfe-56,Jaggi-ICML13} solves constrained optimization with a linear optimization oracle.
    Smoothing method \citep{Nesterov-MP05} can also handle non-smoothness in objectives, but is in general less efficient than specialized CO algorithms such as dual averaging (c.f., \citealt{Nesterov-18}). 
    In this work, we mostly focus on Mirror Descent and Dual Averaging algorithms since they only employ simple proximal oracles such as projection and soft-thresholding.
% FW does not directly extend to general additive composite other than hard constraints.
    % We refer readers to \cref{sec:literature:co} for additional related work in composite optimization.
\end{remark}

\subsection{Federated Averaging}
\label{sec:fedavg}

Federated Averaging (\textsc{FedAvg}, \citealt{McMahan.Moore.ea-AISTATS17}) is the \emph{de facto} standard algorithm for Federated Learning with unconstrained smooth objectives (namely $\psi = 0$ for (FCO)). 
In this work, we follow the exposition of \citep{Reddi.Charles.ea-20} which splits the client learning rate and server learning rate, offering more flexibility (see \cref{alg:fedavg}).

\fedavg involves a series of \emph{rounds} in which each round consists of a client update phase and server update phase. 
We denote the total number of rounds as $R$. 
At the beginning of each round $r$, a subset of clients $\mathcal{S}_r$ are sampled from the client pools of size $M$. 
The server state is then broadcast to the sampled client as the client initialization.
During the client update phase (highlighted in \textcolor{blue}{blue} shade), each sampled client runs local SGD for $K$ steps with client learning rate $\eta_{\client}$ with their own data. 
We use $w_{r,k}^m$ to denote the $m$-th client state at the $k$-th local step of the $r$-th round. 
During the server update phase, the server averages the updates of the sampled clients and treats it as a pseudo-anti-gradient $\Delta_r$ (Line 9). 
The server then takes a server update step to update its server state with server learning rate $\eta_{\server}$ and the pseudo-anti-gradient $\Delta_r$ (Line 10). 

\input{alg_fedavg}

%% file: alg_fedavg.tex
% !TEX root = main.tex  
\begin{algorithm}
  \caption{\fedavgfull (\fedavg)}
  \label{alg:fedavg}
  \begin{algorithmic}[1]
  \STATE {\textbf{procedure}} \fedavg ($w_0, \eta_{\client}, \eta_{\server}$)
  \FOR {$r=0, \ldots, R-1$}
    \STATE sample a subset of clients $\mathcal{S}_r \subseteq [M]$
      \FORALL {\kern-6.5em\tikzmark{fedavg:begin}\kern+6.5em $m \in \mathcal{S}_r$ {\bf in parallel}}
        \STATE $w_{r,0}^m \gets w_{r}$ \COMMENT{broadcast client initialization}
        \FOR {$k = 0, \ldots, K-1$}
          \STATE $g_{r,k}^m \gets \nabla f(w_{r,k}^m; \xi_{r,k}^m)$ \COMMENT{query gradient}
          \STATE $w_{r,k+1}^m \gets w_{r,k}^m - \eta_{\client} \cdot g_{r,k}^m$ 
          \COMMENT{client update\tikzmark{fedavg:end}} 
          \drawCodeBox{fill=blue}{fedavg:begin}{fedavg:end}
        \ENDFOR
      \ENDFOR
  \STATE $\Delta_r = \frac{1}{|\mathcal{S}_r|} \sum_{m \in \mathcal{S}_r} (w_{r, K}^m - w_{r, 0}^m)$ 
  \STATE $w_{r+1} \gets w_{r} + \eta_{\server} \cdot \Delta_r$ \COMMENT{server update}
  \ENDFOR
\end{algorithmic}
\end{algorithm}

%% file: algorithms.tex
% !TEX root = main.tex  
\section{Proposed Algorithms for FCO}
\label{sec:alg}
In this section, we explore the possible solutions to approach (FCO).
As mentioned earlier, existing FL algorithms such as \fedavg and its variants do not solve (FCO). Although it is possible to apply \fedavg to non-smooth settings by using subgradient in place of the gradient, such an approach is usually ineffective owing to the intrinsic slow convergence of subgradient methods \citep{Boyd.Xiao.ea-03}. 

\subsection{Federated Mirror Descent (\fedmid)}
\label{sec:fedmid}
A more natural extension of \fedavg towards (FCO) is to replace the local SGD steps in \fedavg with local proximal gradient (mirror descent) steps \eqref{eq:md:oneline}.
% The first proposal is to replace the SGD steps in \fedavg with the proximal (mirror descent) steps to handle the non-smooth composite $\psi$.
The resulting algorithm, which we refer to as \emph{\fedmidfull} (\fedmid)\footnote{Despite sharing the same term ``prox'', \fedmid is fundamentally different from \fedprox \citep{Li.Sahu.ea-MLSys20}. 
The proximal step in \fedprox was to regularize the client drift caused by heterogeneity, whereas the proximal step in this work is to overcome the non-smoothness of $\psi$. 
The problems approached by the two methods are also different -- \fedprox still solves an unconstrained smooth problem, whereas ours concerns with approaches (FCO).
}, is outlined in \cref{alg:fedmid}. 
\input{alg_fedmid}

Specifically, we make two changes compared to \fedavg: 
\begin{itemize}[leftmargin=*]%,partopsep=0pt,parsep=0pt]
    \item The client local SGD steps in \fedavg are replaced with proximal gradient steps (Line 8).
    \item The server update step is replaced with another proximal step (Line 10). 
    % The coefficient of $\psi$ is taken to be $\eta_{\server}\eta_{\client}K$ to keep consistency with the magnitude of $\Delta_r$.
\end{itemize}
As a sanity check, for constrained (FCO) with $\psi = \chi_{\mathcal{C}}$, if one takes server learning rate $\eta_{\server}=1$ and Euclidean distance $h(w) = \frac{1}{2}\|w\|_2^2$, \fedmid will simply reduce to the following parallel projected SGD with periodic averaging:
\begin{enumerate}[(a),leftmargin=*]
    \item Each sampled client runs $K$ steps of projected SGD following $w_{r,k+1}^m \gets \mathbf{Proj}_{\mathcal{C}}(w_{r,k}^m - \eta_{\client} g_{r,k}^m)$.
    \item After $K$ local steps, the server simply average the client states following $w_{r+1} \gets \frac{1}{|\mathcal{S}_r|} \sum_{m \in \mathcal{S}_r} w_{r,K}^m$.
\end{enumerate}

\subsection{Limitation of \fedmid: Curse of Primal Averaging}
\label{sec:curse}
Despite its simplicity, \fedmid exhibits a major limitation, which we refer to as ``curse of primal averaging'': the server averaging step in \fedmid may severely impede the optimization progress.
To understand this phenomenon, let us consider the following two illustrative examples:
\begin{itemize}[leftmargin=*]
    \item {\bf Constrained problem}: Suppose the optimum of the aforementioned constrained problem resides on a non-flat boundary $\mathcal{C}$. Even when each client is able to obtain a local solution \emph{on} the boundary, the average of them will almost surely be \emph{off} the boundary (and hence away from the optimum) due to the curvature. 
    \item {\bf Federated $\ell_1$-regularized logistic regression problem}: Suppose each client obtains a local \emph{sparse} solution, simply averaging them across clients will invariably yield a non-sparse solution.
\end{itemize}
As we will see theoretically (\cref{sec:theory}) and empirically (\cref{sec:expr}), the ``curse of primal averaging'' indeed hampers the performance of \fedmid.
% Note that \fedmid requires proximal steps at both the client and the server. To reduce the computation cost at the client-side, we also consider a variant where the proximal step is only performed at the server (see \cref{sec:expr} for details). As we shall see later, this variant of \fedmid performs well in certain settings. 

\subsection{Federated Dual Averaging (\feddualavg)}
\label{sec:feddualavg}
Before we look into the solution of the curse of primal averaging, let us briefly investigate the cause of this effect.
Recall that in standard smooth FL settings, server averaging step is helpful because it implicitly pools the stochastic gradients and thereby reduces the variance \citep{Stich-ICLR19}. 
In \fedmid, however, the server averaging operates on the post-proximal \textbf{primal} states, but the gradient is updated in the \textbf{dual} space (recall the primal-dual interpretation of mirror descent in \cref{sec:co}). 
% Since averaging and proximal are in general not commutable, t
This primal/dual mismatch creates an obstacle for primal averaging to benefit from the pooling of stochastic gradients in dual space. 
% This observation motivates the importance of aligning the 
This thought experiment suggests the importance of aligning the gradient update and server averaging.

Building upon this intuition, we propose a novel primal-dual algorithm, named \emph{\feddualavgfull} (\feddualavg, \cref{alg:feddualavg}), which provably addresses the curse of primal averaging. 
The major novelty of \feddualavg, in comparison with \fedmid or its precursor \fedavg, is to operate the server averaging in the dual space instead of the primal. 
This facilitates the server to aggregate the gradient information since the gradients are also accumulated in the dual space.

Formally, each client maintains a pair of primal and dual states $(w_{r,k}^m, z_{r,k}^m)$. 
At the beginning of each client update round, the client dual state is initialized with the server dual state. 
During the client update stage, each client runs  dual averaging steps following \eqref{eq:da} to update its primal and dual state (highlighted in \textcolor{blue}{blue} shade). 
The coefficient of $\psi$, namely $\tilde{\eta}_{r,k}$, is to balance the contribution from $F$ and $\psi$.
At the end of each client update phase, the \emph{dual updates} (instead of primal updates) are returned to the server. 
The server then averages the dual updates of the sampled clients and updates the server dual state.
\input{alg_feddualavg}
We observe that the averaging in \feddualavg is two-fold: (1) averaging of gradients in dual space within a client and (2) averaging of dual states across clients at the server. As we shall see shortly in our theoretical analysis, this novel ``double'' averaging of \feddualavg in the non-smooth case enables lower communication complexity and faster convergence of \feddualavg under realistic assumptions.

%% file: alg_fedmid.tex
% !TEX root = main.tex  
\begin{algorithm}
  \caption{\fedmidfull (\fedmid)}
  \label{alg:fedmid}
  \begin{algorithmic}[1]
  \STATE {\textbf{procedure}} \fedmid ($w_0, \eta_{\client}, \eta_{\server}$)
  \FOR {$r=0, \ldots, R-1$}
    \STATE sample a subset of clients $\mathcal{S}_r \subseteq [M]$
      \FORALL {\kern-6.5em\tikzmark{fedmid:begin}\kern+6.5em $m \in \mathcal{S}_r$ {\bf in parallel}}
        \STATE $w_{r,0}^m \gets w_{r}$ \COMMENT{broadcast \emph{primal} initialization}
        \FOR {$k = 0, \ldots, K-1$}
          \STATE $g_{r,k}^m \gets \nabla f(w_{r,k}^m; \xi_{r,k}^m)$ \COMMENT{query gradient}
          \STATE $w_{r,k+1}^m \gets \nabla (h + \eta_{\client} \psi)^*(\nabla h(w_{r,k}^m) - \eta_{\client} \cdot g_{r,k}^m)$
            \hfill  \COMMENT{client update} \tikzmark{fedmid:end}
          \drawCodeBox{fill=blue}{fedmid:begin}{fedmid:end}
        \ENDFOR
      \ENDFOR
  \STATE $\Delta_r = \frac{1}{|\mathcal{S}_r|} \sum_{m \in \mathcal{S}_r} (w_{r, K}^m - w_{r, 0}^m)$ 
  \STATE $w_{r+1} \gets \nabla (h + \eta_{\server}\eta_{\client}K \psi)^*(\nabla h(w_{r}) + \eta_{\server} \cdot \Delta_r)$ 
          \COMMENT{server update} 
  \ENDFOR
\end{algorithmic}
\end{algorithm}

%% file: alg_feddualavg.tex
% !TEX root = main.tex  
\begin{algorithm}
  \caption{\feddualavgfull (\feddualavg)}
  \label{alg:feddualavg}
  \begin{algorithmic}[1]
  \STATE {\textbf{procedure}} \feddualavg ($w_0, \eta_{\client}, \eta_{\server}$)
  \STATE $z_{0} \gets \nabla h(w_0)$ \COMMENT{server dual initialization}
  \FOR {$r=0, \ldots, R-1$}
    \STATE sample a subset of clients $\mathcal{S}_r \subseteq [M]$
      \FORALL {\kern-6.5em\tikzmark{feddualavg:begin}\kern+6.5em $m \in \mathcal{S}_r$ {\bf in parallel}}
        \STATE $z_{r,0}^m \gets z_{r}$ \COMMENT{broadcast \emph{dual} initialization}
        \FOR {$k = 0, \ldots, K-1$}
          \STATE $\tilde{\eta}_{r,k} \gets \eta_{\server} \eta_{\client} r K + \eta_{\client} k$ 
          \STATE $w_{r,k}^m \gets \nabla (h + \tilde{\eta}_{r,k} \psi)^*(z_{r,k}^m)$
            \COMMENT{retrieve primal}
          \STATE $g_{r,k}^m \gets \nabla f (w_{r,k}^m; \xi_{r,k}^m) $
            \COMMENT{query gradient}
          \STATE $z_{r,k+1}^m \gets z_{r,k}^m - \eta_{\client} g_{r,k}^m$
            \COMMENT{client \emph{dual} update\tikzmark{feddualavg:end}}
          % \STATE \COMMENT{client update\tikzmark{fedmid:end}} 
          \drawCodeBox{fill=blue}{feddualavg:begin}{feddualavg:end}
        \ENDFOR
      \ENDFOR
  \STATE $\Delta_r = \frac{1}{|\mathcal{S}_r|} \sum_{m \in \mathcal{S}_r} (z_{r, K}^m - z_{r, 0}^m)$
  \STATE $z_{r+1} \gets z_{r} + \eta_{\server} \Delta_r$
            \COMMENT{server \emph{dual} update}
  \STATE $w_{r+1} \gets \nabla (h + \eta_{\server} \eta_{\client} (r+1) K \psi)^* (z_{r+1})$
            \COMMENT{(optional) retrieve server primal state}
  \ENDFOR
\end{algorithmic}
\end{algorithm}

%% file: theoretical_results.tex
% !TEX root = main.tex  
\section{Theoretical Results}
\label{sec:theory}
In this section, we demonstrate the theoretical results of \fedmid and \feddualavg. 
We assume the following assumptions throughout the paper. 
The convex analysis definitions in \cref{a1} are reviewed in \cref{sec:background}.
\begin{assumption} Let $\|\cdot\|$ be a norm and $\|\cdot\|_*$ be its dual.
  \label{a1}
  \begin{enumerate}[(a),leftmargin=*]
    \item $\psi: \reals^d \to \reals \cup \{+ \infty\}$ is a closed convex function with closed $\dom \psi$. 
    Assume that $\Phi(w) = F(w) + \psi(w)$ attains a finite optimum at $w^{\star} \in \dom \psi$.
    \item  $h: \reals^d \to \reals \cup \{+\infty\}$ is a Legendre function that is 1-strongly-convex w.r.t. $\|\cdot\|$. 
    Assume $\dom h \supset \dom \psi$.
    \item
    $f(\cdot, \xi): \reals^{d} \to \reals$ is a closed convex function that is differentiable on $\dom \psi$ for any fixed $\xi$.
    In addition, $f(\cdot, \xi)$ is $L$-smooth w.r.t. $\|\cdot\|$ on $\dom \psi$, % for any fixed $\xi$.
      namely for any $u, w \in \dom \psi$, 
      \begin{equation}
        f(u;\xi) \leq f(w;\xi) + \left\langle \nabla f(w;\xi),  u - w \right\rangle + \frac{1}{2} L \|u - w\|^2.
      \end{equation}
    \item $\nabla f$ has $\sigma^2$-bounded variance over $\mathcal{D}_m$ under $\|\cdot\|_*$  within $\dom \psi$, namely for any $w \in \dom \psi$,
      \begin{equation}
          \expt_{\xi \sim \mathcal{D}_m} \left\| \nabla f(w, \xi) - \nabla F_m(w) \right\|_*^2 \leq \sigma^2, \text{ for any $m \in [M]$}
      \end{equation}
    \item Assume that all the $M$ clients participate in the client updates for every round, namely $\mathcal{S}_r = [M]$.
  \end{enumerate}
 
\end{assumption}
\cref{a1}(a) \& (b) are fairly standard for composite optimization analysis (c.f. \citealt{Flammarion.Bach-COLT17}). \cref{a1}(c) \& (d) are standard assumptions in stochastic federated optimization literature \citep{Khaled.Mishchenko.ea-AISTATS20,Woodworth.Patel.ea-ICML20}. (e) is assumed to simplify the exposition of the theoretical results. All results presented can be easily generalized to the partial participation case.

\begin{remark}
  This work focuses on convex settings because the non-convex composite optimization (either $F$ or $\psi$ non-convex) is noticeably challenging and under-developed \textbf{even for non-distributed settings}.  
This is in sharp contrast to non-convex smooth optimization for which simple algorithms such as SGD can readily work.
Existing literature on non-convex CO (e.g., \citealt{Attouch.Bolte.ea-MP13,
Chouzenoux.Pesquet.ea-JOTA14,
Li.Pong-SIOPT15,
Bredies.Lorenz.ea-JOTA15}) typically relies on non-trivial additional assumptions (such as K-Ł conditions) and sophisticated algorithms.
Hence, it is beyond the scope of this work to study non-convex FCO. \footnote{However, we conjecture that for simple non-convex settings (e.g., optimize non-convex $f$ on a convex set, as tested in \cref{sec:emnist}), it is possible to show the convergence and obtain similar advantageous results for \textsc{FedDualAvg}.}
\end{remark}

\subsection{\fedmid and \feddualavg: Small Client Learning Rate Regime}
\label{subsec:small:client:lr}
We first show that both \fedmid and \feddualavg are (asymptotically) at least as good as stochastic mini-batch algorithms with $R$ iterations and batch-size $MK$ when client learning rate $\eta_{\client}$ is sufficiently small. 

\begin{theorem}[Simplified from \cref{small_lr}]
  \label{thm:0}
  Assuming \cref{a1}, then for sufficiently small client learning rate $\eta_{\client}$, and server learning rate $\eta_{\server} = \Theta (\min \{\frac{1} {\eta_{\client} K L}, \frac{B^{\frac{1}{2}} M^{\frac{1}{2}} }{ \eta_{\client} K^{\frac{1}{2}} R^{\frac{1}{2}} \sigma} \} )$, both \feddualavg and \fedmid can output $\hat{w}$ such that
  \begin{equation}
    \expt \left[ \Phi (\hat{w}) \right] - \Phi(w^{\star})  
    \lesssim
    \frac{L B}{R} 
    +
    \frac{\sigma B^{\frac{1}{2}}}{\sqrt{MKR}},
    \label{eq:thm:0}
  \end{equation}
  where $B := D_h(w^{\star}, w_0)$.
\end{theorem}
The intuition is that when $\eta_{\client}$ is small, the client update will not drift too far away from its initialization of the round. Due to space constraints, the proof is relegated to \cref{sec:small_lr}. 

\subsection{\feddualavg with a Larger Client Learning Rate: Usefulness of Local Step}
\label{sec:feddualavg-benefit}
In this subsection, we show that \feddualavg may attain stronger results with a larger client learning rate. 
In addition to possible faster convergence, \cref{thm:1:simplified,thm:2:simplified} also indicate that \feddualavg allows for much broader searching scope of efficient learning rates configurations, which is of key importance for practical purpose.
% We establish faster convergence rates in two important machine learning settings: (i) loss functions with bounded gradients and (ii) ridge regression (or more general quadratic functions)
% Specifically, we focus on the setting when the server learning rate $\eta_{\server} = 1$. 

\paragraph{Bounded Gradient.}
We first consider the setting with bounded gradient.
Unlike unconstrained, the gradient bound may be particularly useful when the constraint is finite.
\begin{theorem}[Simplified from \cref{thm:1}]
  \label{thm:1:simplified}
  Assuming \cref{a1} and $\sup_{w \in \dom \psi} \|\nabla f(w, \xi)\|_* \leq G$, then for \feddualavg with $\eta_{\server} = 1$ and  $\eta_{\client} \leq \frac{1}{4L}$, 
  considering
    \begin{equation}
      \hat{w} := \frac{1}{KR} \sum_{r=0}^{R-1} \sum_{k=1}^{K} 
      \left[ \nabla \left( h + \tilde{\eta}_{r,k} \psi \right)^* \left( \frac{1}{M} \sum_{m=1}^M z_{r,k}^m \right) \right],
      \label{eq:w_hat}
    \end{equation}
  the following inequality holds
  \begin{equation}
    \expt \left[ \Phi \left( \hat{w} \right) \right] - \Phi(w^{\star})  
    \lesssim
    \frac{B}{\eta_{\client} KR}  + \frac{\eta_{\client} \sigma^2}{M} +  \eta_{\client}^2 L K^2 G^2,
  \end{equation}
  where $B := D_h(w^{\star},w_0)$. 
  Moreover, there exists $\eta_{\client}$ such that
  \begin{equation}
    \expt \left[ \Phi (\hat{w}) \right] - \Phi(w^{\star})  
    \lesssim
    \frac{L B}{KR} 
    +
    \frac{\sigma B^{\frac{1}{2}}}{\sqrt{MKR}}
    +
    \frac{L^{\frac{1}{3}} B^{\frac{2}{3}} G^{\frac{2}{3}}}{R^{\frac{2}{3}}}.
    \label{eq:thm:1:simplified:2}
  \end{equation}
\end{theorem}
We refer the reader to \cref{sec:proof:thm:1} for complete proof details of \cref{thm:1:simplified}.
\begin{remark}
The result in \cref{thm:1:simplified} not only matches the rate by \citet{Stich-ICLR19} for smooth, unconstrained \fedavg but also allows for a general non-smooth composite $\psi$, general Bregman divergence induced by $h$, and arbitrary norm $\|\cdot\|$. 
Compared with the small learning rate result \cref{thm:0}, the first term in \cref{eq:thm:1:simplified:2} is improved from $\frac{LB}{R}$ to $\frac{LB}{KR}$, whereas the third term incurs an additional loss regarding infrequent communication. 
One can verify that the bound \cref{eq:thm:1:simplified:2} is better than \cref{eq:thm:0} if $R \lesssim \frac{L^2B}{G^2}$. 
Therefore, the larger client learning rate may be preferred when the communication is not too infrequent.
\end{remark}

\paragraph{Bounded Heterogeneity.}
Next, we consider the settings with bounded heterogeneity. 
For simplicity, we focus on the case when the loss $F$ is quadratic, as shown in \cref{a3}.
We will discuss other options to relax the quadratic assumption in \cref{sec:proof_sketch}.

\begin{assumption}[Bounded heterogeneity, quadratic]
  \label{a3}
  \begin{enumerate}[(a),leftmargin=*]
      \item The heterogeneity of $\nabla F_m$ is bounded, namely
      \begin{equation}
        \sup_{w \in \dom \psi} \|\nabla F_m(w) - \nabla F(w) \|_* \leq  \zeta^2,
        \text{ for any $m \in [M]$}
      \end{equation}
    \item $F(w) := \frac{1}{2} w^{\top} Q w + c^{\top} w$ for some $Q \succ 0$.
    \item Assume \cref{a1} is satisfied in which the norm $\|\cdot\|$ is taken to be the $\frac{Q}{\|Q\|_2}$-norm, namely $\|w\| = \sqrt{\frac{w^{\top} Q w}{\|Q\|_2}}$.
  \end{enumerate}
\end{assumption}
\begin{remark}
\cref{a3}(a) is a standard assumption to bound the heterogeneity among clients (e.g., \citealt{Woodworth.Patel.ea-NeurIPS20}).
Note that \cref{a3} only assumes the objective $F$ to be  quadratic. We do not impose any stronger assumptions on either the composite function $\psi$ or the distance-generating function $h$. 
Therefore, this result still applies to a broad class of problems such as \textsc{Lasso}.
\end{remark}
The following results hold under \cref{a3}. 
We sketch the proof in \cref{sec:proof_sketch} and defer the details to \cref{sec:proof:thm:2}.
\begin{theorem}[Simplified from \cref{thm:2}]
  \label{thm:2:simplified}
  Assuming \cref{a3}, then for \feddualavg with $\eta_{\server} = 1$ and $\eta_c \leq \frac{1}{4L}$, the following inequality holds
  \begin{equation}
    \expt [\Phi(\hat{w})] - \Phi(w^{\star}) 
    \lesssim
      \frac{B}{\eta_{\client} KR }
    + \frac{\eta_{\client}  \sigma^2}{M}
    + \eta_{\client}^2 L K \sigma^2 
    + \eta_{\client}^2 L K^2 \zeta^2,
    \label{eq:thm:2:simplified}
  \end{equation}
  where $\hat{w}$ is the same as defined in \cref{eq:w_hat}, and $B := D_h(w^{\star},w_0)$.
  Moreover, there exists $\eta_{\client}$ such that 
    \begin{equation}
     \expt \left[ \Phi (\hat{w}) \right] - \Phi(w^{\star})  
    \lesssim 
    \frac{LB}{KR} 
    +
    \frac{\sigma B^{\frac{1}{2}}}{\sqrt{MKR}}
    +
    \frac{L^{\frac{1}{3}} B^{\frac{2}{3}} \sigma^{\frac{2}{3}}}{K^{\frac{1}{3}} R^{\frac{2}{3}}}
    +
    \frac{L^{\frac{1}{3}} B^{\frac{2}{3}} \zeta^{\frac{2}{3}}}{R^{\frac{2}{3}}}.
    \label{eq:thm:2:simplified:2}
    \end{equation}
\end{theorem}
\begin{remark}
    The result in \cref{thm:2:simplified} matches the best-known convergence rate for smooth, unconstrained \fedavg \citep{Khaled.Mishchenko.ea-AISTATS20,Woodworth.Patel.ea-NeurIPS20}, while our results allow for general composite $\psi$ and non-Euclidean distance. 
    Compared with \cref{thm:1:simplified}, the overhead in \cref{eq:thm:2:simplified:2} involves variance $\sigma$ and heterogeneity $\zeta$ but no longer depends on $G$. 
    The bound \cref{eq:thm:2:simplified:2} could significantly outperform the previous ones when the variance $\sigma$ and heterogeneity $\zeta$ are relatively mild.
\end{remark}

%% file: proof_sketch.tex
% !TEX root = main.tex  
\subsection{Proof Sketch and Discussions}
\label{sec:proof_sketch}
In this subsection, we demonstrate our proof framework by sketching the proof for \cref{thm:2:simplified}. 
% which provides stronger guarantee for \feddualavg with larger client learning rate $\eta_{\client}$ and unit server learning rate $\eta_{\server}=1$.
\paragraph{Step 1: Convergence of Dual Shadow Sequence.}
We start by characterizing the convergence of the dual shadow sequence $\overline{z_{r,k}} := \frac{1}{M} \sum_{m=1}^M z_{r,k}^m$. 
The key observation for \feddualavg when $\eta_{s} = 1$ is the following relation
\begin{equation}
  \overline{z_{r,k+1}} = \overline{z_{r,k}} - \eta_{\client} \cdot \frac{1}{M} \sum_{m=1}^M \nabla f(w_{r,k}^m; \xi_{r,k}^m).
  \label{eq:shadow}
\end{equation}
This suggests that the shadow sequence $\overline{z_{r,k}}$ almost executes a dual averaging update \eqref{eq:da}, but with some perturbed gradient $\frac{1}{M} \sum_{m=1}^M \nabla f(w_{r,k}^m; \xi_{r,k}^m)$. 
To this end, we extend the perturbed iterate analysis framework \citep{Mania.Pan.ea-SIOPT17} to the dual space.
Theoretically we show the following \cref{pia:general:simplified}, with proof relegated to \cref{sec:pia:general}.
\begin{lemma}[Convergence of dual shadow sequence of \feddualavg, simplified version of \cref{pia:general}]
  \label{pia:general:simplified}
  Assuming $\cref{a1}$, then for \feddualavg with $\eta_{\server} = 1$ and $\eta_{\client} \leq \frac{1}{4L}$,  the following inequality holds
  \begin{small}
    \begin{equation}
      \expt \left[ \Phi \left( \frac{1}{KR} \sum_{r=0}^{R-1} \sum_{k=1}^{K} 
      \nabla \left( h + \tilde{\eta}_{r,k} \psi \right)^* \left( \overline{z_{r,k}}\right) \right) \right] - \Phi(w^{\star}) 
      \leq 
      \underbrace{\frac{B}{\eta_{\client} KR} + \frac{\eta_{\client}  \sigma^2}{M} + }_{\substack{\text{\footnotesize Rate if synchronized} \\ 
      \text{\footnotesize every iteration}}}
      \underbrace{\frac{L}{M KR}   \left[ \sum_{r=0}^{R-1}  \sum_{k=0}^{K-1} \sum_{m=1}^M   \expt \| \overline{z_{r,k}} - z_{r,k}^m \|_*^2 \right]}_{\text{\footnotesize  Discrepancy overhead}}.
      \label{eq:pia:quad:simplified:1}
    \end{equation} 
  \end{small}
\end{lemma}
The first two terms correspond to the rate when \feddualavg is synchronized every step. 
% This rate matches the (centralized) minibatch stochastic dual averaging with batch-size $M$ per iteration and $KR$ iterations in total.
The last term corresponds to the overhead for not synchronizing every step, which we call ``discrepancy overhead''.
% We refer to this term as ``discrepancy overhead'' since it represents the difference among clients. 
\cref{pia:general:simplified} can serve as a general interface towards the convergence of \feddualavg as it only assumes the blanket \cref{a1}.

\begin{remark}
Note that the relation \eqref{eq:shadow} is not satisfied by \fedmid due to the incommutability of the proximal operator and the the averaging operator, which thereby breaks \cref{pia:general:simplified}. 
Intuitively, this means \fedmid fails to pool the gradients properly (up to a high-order error) in the absence of communication. 
\feddualavg overcomes the incommutability issue because all the gradients are accumuluated and averaged  in the dual space, whereas the proximal step only operates at the interface from dual to primal.
% problem because the inter-client averaging is performed in the dual space. 
% Algebraically, the incommutability issue is resolved 
This key difference explains the ``curse of primal averaging'' from the theoretical perspective.
% because of the curse of primal averaging.
\end{remark}

\paragraph{Step 2: Bounding Discrepancy Overhead via Stability Analysis.}
The next step is to bound the discrepancy term introduced in \cref{eq:pia:quad:simplified:1}. 
Intuitively, this term characterizes the \emph{stability} of \feddualavg, in the sense that how far away a single client can deviate from the average (in dual space) if there is no synchronization for $k$ steps. 

However, unlike the smooth convex unconstrained settings in which the stability of SGD is known to be well-behaved \citep{Hardt.Recht.ea-ICML16}, the stability analysis for composite optimization is challenging and absent from the literature. 
We identify that the main challenge originates from the asymmetry of the Bregman divergence. 
In this work, we provide a set of simple conditions, namely \cref{a3}, such that the stability of \feddualavg is well-behaved. 
\begin{lemma}[Dual stability of \feddualavg under \cref{a3}, simplified version of \cref{quad:stability}]
  \label{quad:stability:simplified}
  Under the same settings of \cref{thm:2:simplified}, the following inequality holds
  \begin{equation}
    \frac{1}{M} \sum_{m=1}^M \expt \left \| \overline{z_{r,k}} - z_{r,k}^m \right\|_*^2
     \lesssim \eta_{\client}^2 K \sigma^2 + \eta_{\client}^2 K^2 \zeta^2.
  \end{equation}
\end{lemma}
% The proof is extended from \citep[Lemma 7]{Flammarion.Bach-COLT17}, which is deferred to \cref{sec:quad:stability}.

\paragraph{Step 3: Deciding $\eta_{\client}$.}
The final step is to plug in the bound in step 2 back to step 1,
and find appropriate $\eta_{\client}$ to optimize such upper bound. 
We defer the details to \cref{sec:proof:thm:2}.

%% file: experiments.tex
\section{Numerical Experiments}
\label{sec:expr}
  In this section, we validate our theory and demonstrate the efficiency of the algorithms via numerical experiments. 
  We mostly compare \feddualavg with \fedmid since the latter serves a natural baseline.
  We do not present subgradient-\fedavg in this section due to its consistent ineffectiveness, as demonstrated in \cref{fig:haxby:simplified} (marked \fedavg($\partial$)). 
  To examine the necessity of client proximal step, we also test two less-principled versions of \fedmid and \feddualavg, in which the proximal steps are only performed on the server-side.
  We refer to these two versions as \textsc{FedMiD-OSP} and \textsc{FedDualAvg-OSP}, where ``OSP'' stands for ``only server proximal,'' with pseudo-code provided in \cref{sec:addl_expr_setup}.
  We provide the complete setup details in \cref{sec:additional_expr}, including but not limited to hyper-parameter tuning, dataset processing and evaluation metrics.
  The source code is available at \url{https://github.com/hongliny/FCO-ICML21}.
  
  \subsection{Federated LASSO for Sparse Feature Recovery}
  \label{mainsec:expr:lasso}
  In this subsection, we consider the LASSO ($\ell_1$-regularized least-squares) problem on a synthetic  dataset, motivated by models from biomedical and signal processing literature~(e.g., \citealt{Ryali.Supekar.ea-10,Chen.Lin.ea-NIPS12}).   
  The goal is to recover the sparse signal $w$ from noisy observations $(x,y)$.
  \begin{equation}
      \min_{w \in \reals^d, b \in \reals} ~~ \frac{1}{M} \sum_{m=1}^M \expt_{(x,y) \sim \mathcal{D}_m} (x^{\top} w + b - y)_2^2 + \lambda \|w\|_1.
      \label{eq:flasso}
  \end{equation}
  To generate the synthetic dataset, we first fix a sparse ground truth $w_{\mathrm{real}} \in \reals^{d}$ and some bias $b_{\mathrm{real}} \in \reals$, and then sample the dataset $(x, y)$ following $y = x^{\top} w_{\mathrm{real}} + b_{\mathrm{real}} + \varepsilon$ for some noise $\varepsilon$. 
  We let the distribution of $(x,y)$ vary over clients to simulate the heterogeneity.
  We select $\lambda$ so that the centralized solver (on gathered data) can successfully recover the sparse pattern. 
  Since the ground truth $w_{\mathrm{real}}$ is known, we can assess the quality of the sparse features recovered by comparing it with the ground truth. 
  We evaluate the performance by recording precision, recall, sparsity density, and F1-score. 
  We tune the client learning rate $\eta_{\client}$ and server learning rate $\eta_{\server}$ only to attain the best F1-score.
  The results are presented in \cref{fig:lasso_p1024_nnz512_cl64_1x4}.
  The best learning rates configuration is $\eta_{\client} = 0.01, \eta_{\server} = 1$ for  \feddualavg, and  $\eta_{\client} = 0.001, \eta_{\server} = 0.3$ for other algorithms (including \fedmid). This matches our theory that \feddualavg can benefit from larger learning rates.
  We defer the rest of the setup details and further experiments to \cref{sec:expr:lasso}.
  \input{fig_expr1}

  % \vspace{-2em}
  \subsection{Federated Low-Rank Matrix Estimation via Nuclear-Norm Regularization}
  \label{mainsec:expr:nuclear}
  In this subsection, we consider a low-rank matrix estimation problem via the nuclear-norm regularization
  \begin{equation}
    \min_{W \in \reals^{d_1 \times d_2}, b\in \reals}
     \frac{1}{M} \sum_{m=1}^M \mathbb{E}_{(X,y) \sim \mathcal{D}_m} \left( \langle X, W \rangle + b - y \right)^2 + \lambda \|W\|_{\mathrm{nuc}},
    \label{eq:fed:matrix:estimation}
  \end{equation}
  where $\|W\|_{\mathrm{nuc}}$ denotes the matrix nuclear norm.
  The goal is to recover a low-rank matrix $W$ from noisy observations $(X,y)$. 
  This formulation captures a variety of problems such as low-rank matrix completion and recommendation systems  \citep{Candes.Recht-09}.
  Note that the proximal operator with respect to the nuclear-norm regularizer reduces to singular-value thresholding operation \citep{Cai.Candes.ea-SIOPT10}.
  
  We evaluate the algorithms on a synthetic federated dataset with known low-rank ground truth $W_{\mathrm{real}} \in \reals^{d_1 \times d_2}$ and bias $b_{\mathrm{real}} \in \reals$, similar to the above LASSO experiments. 
  We focus on four metrics for this task: the training (regularized) loss, the validation mean-squared-error, the recovered rank, and the recovery error in Frobenius norm $\|W_{\mathrm{output}} - W_{\mathrm{real}}\|_{\mathrm{F}}$.
  We tune the client learning rate $\eta_{\client}$ and server learning rate $\eta_{\server}$ only to attain the best recovery error.
  We also record the results obtained by the deterministic solver on centralized data, marked as \texttt{optimum}. 
  The results are presented in \cref{fig:nuclear_row32_rank16_cl64_1x4}.
  We provide the rest of the setup details and more experiments in \cref{sec:expr:nuclear}.
  \input{fig_expr2}

  % \vspace{-2em}
  \subsection{Sparse Logistic Regression for fMRI Scan}
  \label{mainsec:expr:fmri}
  In this subsection, we consider the cross-silo setup of learning a binary classifier on fMRI scans.
  For this purpose, we use the data collected by \citet{Haxby-01}, to understand the pattern of response in the ventral temporal (vt) area of the brain given a visual stimulus. 
%   There were six subjects doing image recognition in a block-design experiment over 11 to 12 sessions, with a total of 71 sessions. 
%   Each session consists of 18 fMRI scans under the stimuli of a picture of either a house or a face.
%   We use the \texttt{nilearn} package \citep{Abraham.Pedregosa.ea-14} to normalize and transform the four-dimensional raw fMRI scan data into an array with 39,912 volumetric pixels (voxels) using the standard  mask.
  We plan to learn a sparse ($\ell_1$-regularized) binary logistic regression on the voxels to classify the stimuli given the voxels input.
  Enforcing sparsity is crucial for this task as it allows domain experts to understand which part of the brain is differentiating between the stimuli.
%   We select five (out of six) subjects as the training set and the last subject as the held-out validation set. 
%   We treat each session as a client, with a total of 59 training clients and 12 validation clients, where each client possesses the voxel data of 18 scans.
%   As in the previous experiment, we tune the client learning rate $\eta_{\client}$ and server learning rate $\eta_{\server}$ only. 
%   We set the $\ell_1$-regularization strength to be $10^{-3}$.
%   For each setup, we run the federated algorithms for 300 communication rounds.
  We compare the algorithms with two non-federated baselines: 
  1) \texttt{centralized} corresponds to training on the centralized dataset gathered from \textbf{all} the training clients; 2) \texttt{local} corresponds to training on the local data from only \textbf{one} training client without communication. 
  The results are shown in \cref{fig:haxby59_lambd_1e-03_1x4}. 
  In \cref{sec:lr:config}, we provide another presentation of this experiment to visualize the progress of federated algorithms and understand the robustness to learning rate configurations. 
  The results suggest \feddualavg not only recovers sparse and accurate solutions, but also behaves most robust to learning-rate configurations.
  We defer the rest of the setup details to \cref{sec:expr:fmri}.

  In \cref{sec:emnist},  we provide another set of experiments on federated constrained optimization for Federated EMNIST dataset \citep{Caldas.Duddu.ea-NeurIPS19}.
  \input{fig_expr3}

%% file: fig_expr1.tex
% !TEX root = main.tex  
\begin{figure}[ht]
  \centering
  \includegraphics[width=\columnwidth]{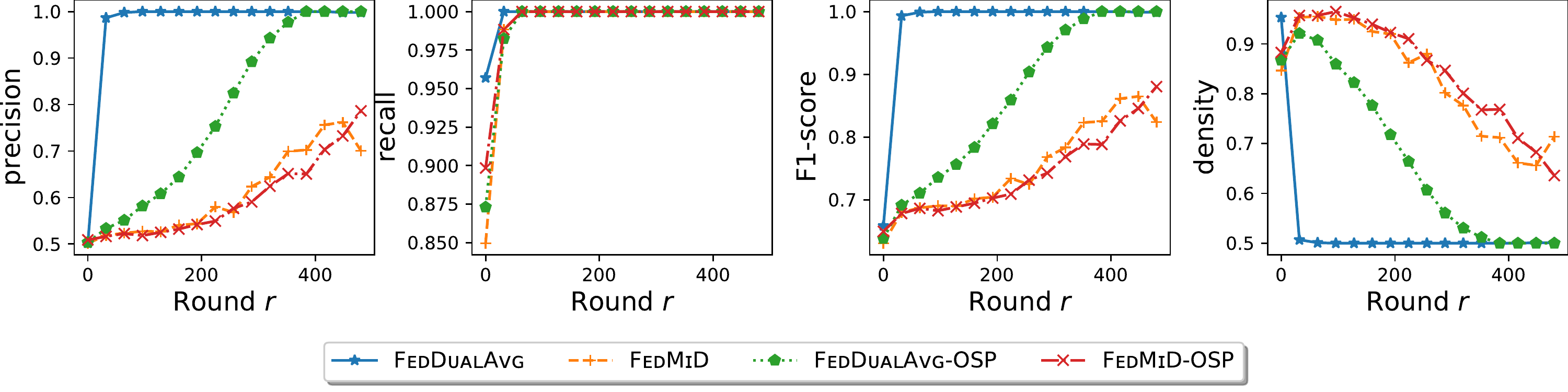}
  \caption{
   \textbf{Sparsity recovery on a synthetic LASSO problem with 50\% sparse ground truth.} 
  Observe that \feddualavg not only identifies most of the sparsity pattern but also is fastest. 
%   Apart from the expected good performance of \feddualavg, 
  It is also worth noting that the less-principled \textsc{FedDualAvg-OSP} is also very competitive.
  The poor performance of \fedmid can be attributed to the ``curse of primal averaging'', as the server averaging step ``smooths out'' the sparsity pattern, which is corroborated empirically by the least sparse solution obtained by \fedmid.
  }
  \label{fig:lasso_p1024_nnz512_cl64_1x4}
\end{figure}

%% file: fig_expr2.tex
% !TEX root = main.tex  
\begin{figure}[ht]
    \centering
    \includegraphics[width=\columnwidth]{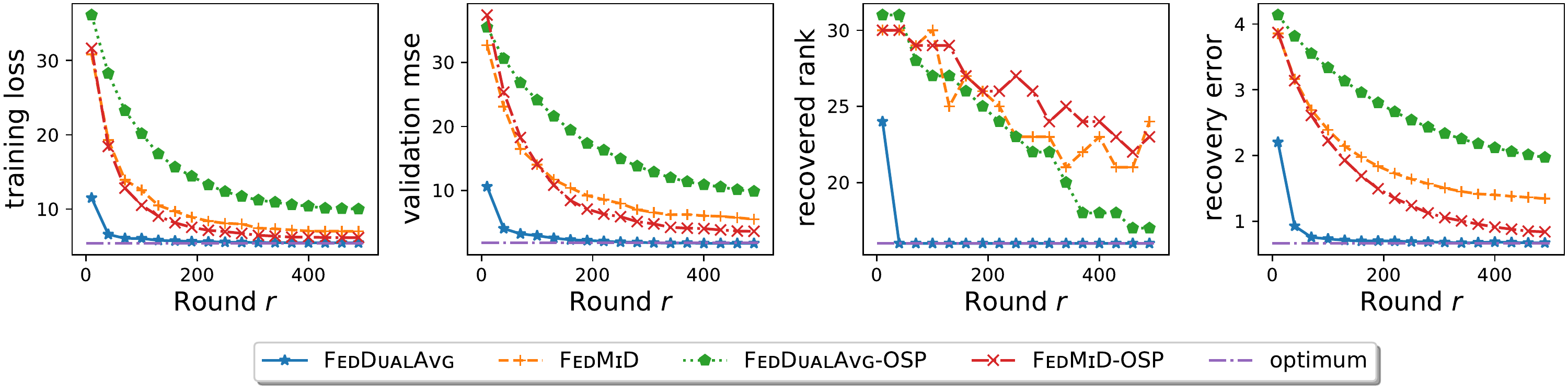}
    \caption{
      \textbf{Low-rank matrix estimation comparison on a synthetic dataset with the ground truth of rank 16.}
  We observe that \feddualavg finds the solution with exact rank in less than 100 communication rounds.
  \fedmid and \fedmid-OSP converge slower in loss and rank.
  The unprincipled \feddualavg-OSP can generate low-rank solutions but is far less accurate. 
    }
    \label{fig:nuclear_row32_rank16_cl64_1x4}
  \end{figure}

%% file: fig_expr3.tex
\begin{figure}[!hbp]
  \centering
  \includegraphics[width=\columnwidth]{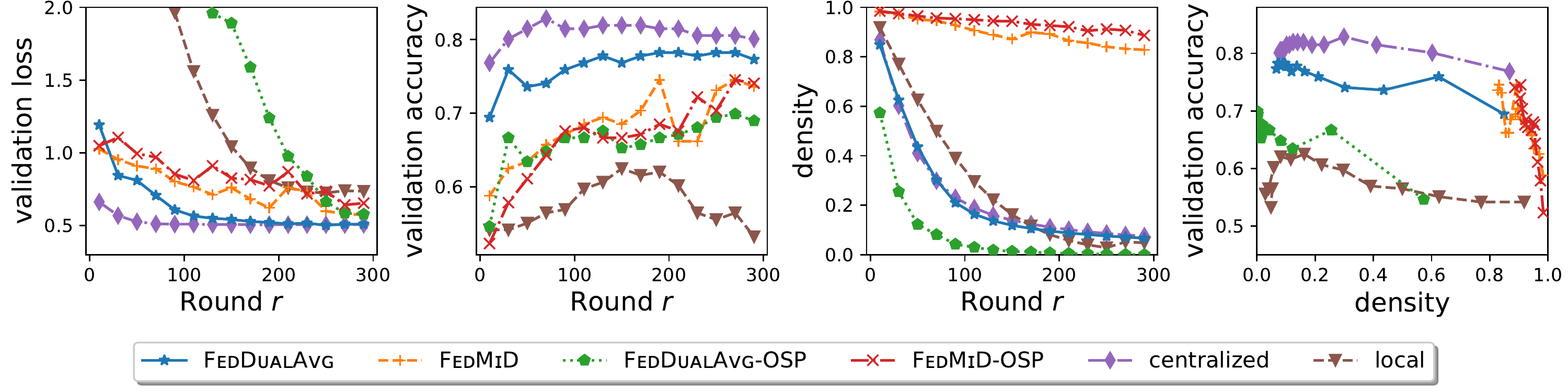}
  \caption{
    \textbf{Results on $\ell_1$-regularized logistic regression for fMRI data from \citep{Haxby-01}}. 
    %  We observe that the result of \feddualavg is comparable with the \texttt{centralized} baseline on gathered data and significantly outperforms the \texttt{local} baseline on isolated data.
    We observe that $\feddualavg$ yields sparse and accurate solutions that are comparable with the centralized baseline. 
    \fedmid and \fedmid-OSP provides denser solutions that are relatively less accurate.
    The unprincipled \feddualavg-OSP can provide sparse solutions but far less accurate.
  }
  \label{fig:haxby59_lambd_1e-03_1x4}
\end{figure}

%% file: conclusion.tex
\section{Conclusion}
In this paper, we have shown the shortcomings of primal FL algorithms for FCO and proposed a primal-dual method (\feddualavg) to tackle them. Our theoretical and empirical analysis provide strong evidence to support the superior performance of \feddualavg over natural baselines. 
Potential future directions include control variates and acceleration based methods for FCO, and applying FCO to personalized settings.

%% file: ack.tex
% !TEX root = main.tex  
\section*{Acknowledgement}
We would like to thank Zachary Charles, Zheng Xu, Andrew Hard, Ehsan Amid, Amr Ahmed, Aranyak Mehta, Qian Li,  Junzi Zhang, Tengyu Ma, and TensorFlow Federated team for helpful discussions at various stages of this work. 
Honglin Yuan would like to thank the support by the TOTAL Innovation Scholars program. 
We would like to thank the anonymous reviewers for their suggestions and comments.

%% file: pre_appendix.tex
The appendices are structured as follows. 
% In \cref{sec:related_works}, we discuss the additional related literature of this work.
In \cref{sec:additional_expr}, we include additional experiments and detailed setups. 
In \cref{sec:background}, we provide the necessary backgrounds for our theoretical results. 
We prove two of our main results, namely \cref{thm:1:simplified,thm:2:simplified}, in \cref{sec:proof:thm:1,sec:proof:thm:2}, respectively.
The proof of \cref{thm:0} is sketched in \cref{sec:small_lr}.

%% file: expr_details.tex
% !TEX root = main.tex  
\section{Additional Experiments and Setup Details}
\label{sec:additional_expr}
\subsection{General Setup}
\label{sec:addl_expr_setup}
\paragraph{Algorithms.} 
In this paper we mainly test four Federated algorithms, namely \fedmidfull (\fedmid, see \cref{alg:fedmid}), \feddualavgfull (\feddualavg, see \cref{alg:feddualavg}), as well as two less-principled algorithms which skip the client-side proximal operations. 
We refer to these two algorithms as \fedmid-OSP and \feddualavg-OSP, where ``OSP'' stands for ''only server proximal''. 
We formally state these two OSP algorithms in \cref{alg:fedmid-osp,alg:feddualavg-osp}. 
We study these two OSP algorithms mainly for ablation study purpose, thouse they might be of special interest if the proximal step is computationally intensive. 
For instance, in \fedmid-OSP, the client proximal step is replaced by $w_{r,k+1}^m \gets \nabla h^*(\nabla h (w_{r,k}^m) - \eta_{\client} g_{r,k}^m)$ with no $\psi$ involved (see line 8 of \cref{alg:fedmid-osp}). This step reduces to the ordinary SGD $w_{r,k+1}^m \gets w_{r,k}^m - \eta_{\client} g_{r,k}^m$ if $h(w) = \frac{1}{2}\|w\|_2^2$ in which case both $\nabla h$ and $\nabla h^*$ are identity mapping.
% However, we stress that there is \textbf{no} theoretical guarantee on the convergence of either \fedmid-OSP or \feddualavg-OSP.
Theoretically, it is not hard to establish similar rates of \cref{thm:0} for \fedmid-OSP with finite $\psi$. For infinite $\psi$, we need extension of $f$ outside $\textbf{dom}\psi$ to fix regularity. To keep this paper focused, we will not establish these results formally.
%We do not establish these results formally to keep this paper focused.
There is no theoretical guarantee on the convergence of \feddualavg-OSP.
% \footnote{It is not hard to establish similar rates of \cref{thm:0} for finite $\psi$. For infinite $\psi$, we need extension of $f$ outside $\textbf{dom}\psi$ to fix regularity. }

\input{alg_fedmid_osp}
\input{alg_feddualavg_osp}

\paragraph{Environments.} We simulate the algorithms in the TensorFlow Federated (TFF) framework \citep{Ingerman.Ostrowski-19}. 
The implementation is based on the Federated Research repository.\footnote{\url{https://github.com/google-research/federated}}
\paragraph{Tasks.} We experiment the following four tasks in this work. 
\begin{enumerate}
    \item Federated Lasso ($\ell_1$-regularized least squares) for sparse feature selection, see \cref{sec:expr:lasso}.
    \item Federated low-rank matrix recovery via nuclear-norm regularization, see \cref{sec:expr:nuclear}.
    \item Federated sparse ($\ell_1$-regularized) logistic regression for fMRI dataset \citep{Haxby-01}, see \cref{sec:expr:fmri}.
    \item Federated constrained optimization for Federated EMNIST dataset \citep{Caldas.Duddu.ea-NeurIPS19}, see \cref{sec:emnist}.
\end{enumerate}
We take the distance-generating function $h$ to be $h(w) := \frac{1}{2}\|w\|_2^2$ for all the four tasks.
The detailed setups of each experiment are stated in the corresponding subsections.

\subsection{Federated LASSO for Sparse Feature Selection}
\label{sec:expr:lasso}
\subsubsection{Setup Details}
In this experiment, we consider the federated LASSO ($\ell_1$-regularized least squares) on a synthetic dataset inspired by models from biomedical and signal processing literature (e.g., \citealt{Ryali.Supekar.ea-10,Chen.Lin.ea-NIPS12})
  \begin{equation}
      \min_{w,b} \quad \frac{1}{M} \sum_{m=1}^M \expt_{(x,y) \sim \mathcal{D}_m} (x^{\top} w + b - y)_2^2 + \lambda \|w\|_1.
  \end{equation}
The goal is to retrieve sparse features of $w$ from noisy observations.
\paragraph{Synthetic Dataset Descriptions.}
We first generate the ground truth $w_{\mathrm{real}}$ with $d_1$ ones and $d_0$ zeros for some $d_1 + d_0 = d$, namely
\begin{equation}
   w_{\mathrm{real}}  = \begin{bmatrix}
  \textbf{1}_{d_1}
  \\
  \textbf{0}_{d_0}
\end{bmatrix} \in \reals^{d},
\end{equation}
and ground truth $b_{\mathrm{real}} \sim \mathcal{N}(0,1)$. 

The observations $(x,y)$ are generated as follows to simulate the heterogeneity among clients. 
Let $(x_{m}^{(i)}, y_{m}^{(i)})$ denotes the $i$-th observation of the $m$-th client. 
For each client $m$, we first generate and fix the mean $\mu_m \sim \mathcal{N}(0, I_{d \times d})$. 
Then we sample $n_m$ pairs of observations following
\begin{align}
  x_m^{(i)} = \mu_m + \delta_m^{(i)}, \quad & \text{where $\delta_m^{(i)} \sim \mathcal{N}(\mathbf{0}_d, I_{d \times d})$ are \text{i.i.d.}, for $i=1,\ldots,n_m$;}
  \\
  y_m^{(i)} = w_{\mathrm{real}}^\top x_m^{(i)} + b_{\mathrm{real}} + \varepsilon_m^{(i)}, \quad & \text{where $\varepsilon_m^{(i)} \sim \mathcal{N}(0,1)$ are i.i.d., for $i=1,\ldots,n_m$}.
\end{align}

We test four configurations of the above synthetic dataset.
\begin{enumerate}
  \item [(I)] The ground truth $w_{\mathrm{real}}$ has $d_1 = 512$ ones and $d_0 = 512$ zeros. 
  We generate $M=64$ training clients where each client possesses $128$ pairs of samples.
  There are 8,192 training samples in total.
  \item [(II)] (sparse ground truth) The ground truth $w_{\mathrm{real}}$ has $d_1 = 64$ ones and $d_0 = 960$ zeros. The rest of the configurations are the same as dataset (I).
  \item [(III)] (sparser ground truth) The ground truth $w_{\mathrm{real}}$ has $d_1 = 8$ ones and $d_0 = 1016$ zeros. The rest of the configurations are the same as dataset (I).
  \item [(IV)] (more distributed data) The ground truth is the same as (I). We generate $M=256$ training clients where each client possesses $32$ pairs of samples. The total number of training examples are the same.
\end{enumerate}

\paragraph{Evaluation Metrics.}
Since the ground truth of the synthetic dataset is known, we can evaluate the quality of the sparse features retrieved by comparing it with the ground truth.
To numerically evaluate the sparsity, we treat all the features in $w$ with absolute values smaller than $10^{-2}$ as zero elements, and non-zero otherwise.
We evaluate the performance by recording precision, recall, F1-score, and sparse density.

\paragraph{Hyperparameters.}
For all algorithms, we tune the client learning rate $\eta_{\client}$ and server learning rate $\eta_{\server}$ only. 
We test 49 different combinations of $\eta_{\client}$ and $\eta_{\server}$.  
$\eta_{\client}$ is selected from $\{0.001, 0.003, 0.01, 0.03, 0.1, 0.3, 1 \}$, and $\eta_{\server}$ is selected from $\{0.01, 0.03, 0.1, 0.3, 1, 3, 10\}$. 
All methods are tuned to achieve the best averaged recovery error over the last 100 communication rounds.
We claim that the best learning rate combination falls in this range for all the algorithms tested.
We draw 10 clients uniformly at random at each communication round and let the selected clients run local algorithms with batch size 10 for one epoch (of its local dataset) for this round. 
We run 500 rounds in total, though \feddualavg usually converges to almost perfect solutions in much fewer rounds.

The \cref{fig:nuclear_row32_rank16_cl64_1x4} presented in the main paper (\cref{mainsec:expr:lasso}) is for the synthetic dataset (I). Now we test the performance on the other three datasets.

\subsubsection{Results on Synthetic Dataset (II) and (III) with Sparser Ground Truth}
\label{sec:expr:lasso:2:3}
We repeat the experiments on the dataset (II) and (III) with $1/2^{4}$ and $1/{2^7}$ ground truth density, respectively.
The results are shown in \cref{fig:lasso_p1024_nnz64_cl64_1x4,fig:lasso_p1024_nnz8_cl64_1x4}.
We observe that \feddualavg converges to the perfect F1-score in less than 100 rounds, which outperforms the other baselines by a margin. 
The F1-score of \feddualavg-OSP converges faster on these sparser datasets than (I), which makes it comparably more competitive.
The convergence of \fedmid and \fedmid-OSP remains slow.

\begin{figure}[!hbp]
  \centering
   \includegraphics[width=\textwidth]{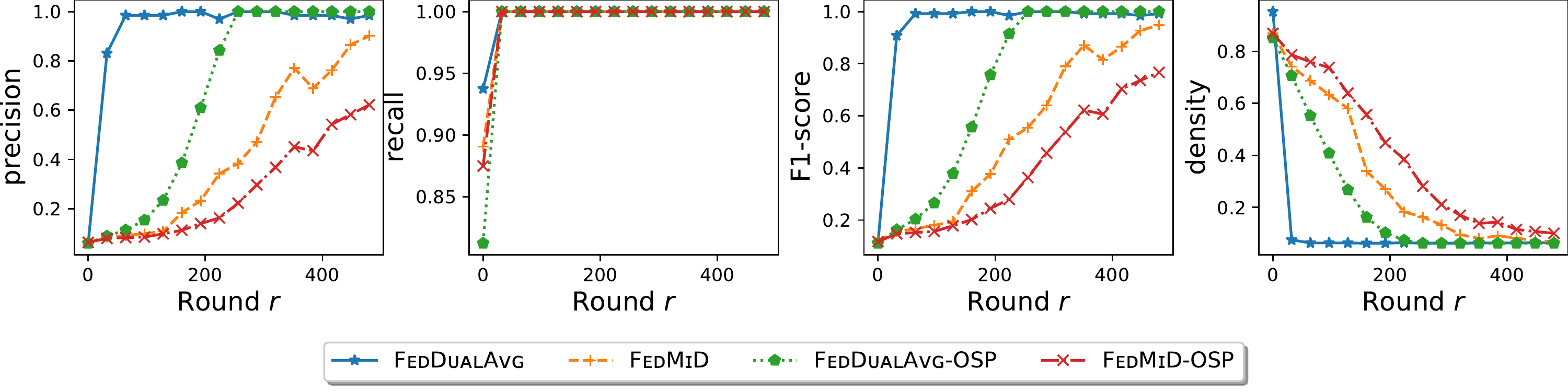}
  \caption{\textbf{Results on Dataset (II): $1/2^{4}$ Ground Truth Density.} See \cref{sec:expr:lasso:2:3} for discussions.}
  \label{fig:lasso_p1024_nnz64_cl64_1x4}
\end{figure}

\begin{figure}[!hbp]
  \centering
  \includegraphics[width=\textwidth]{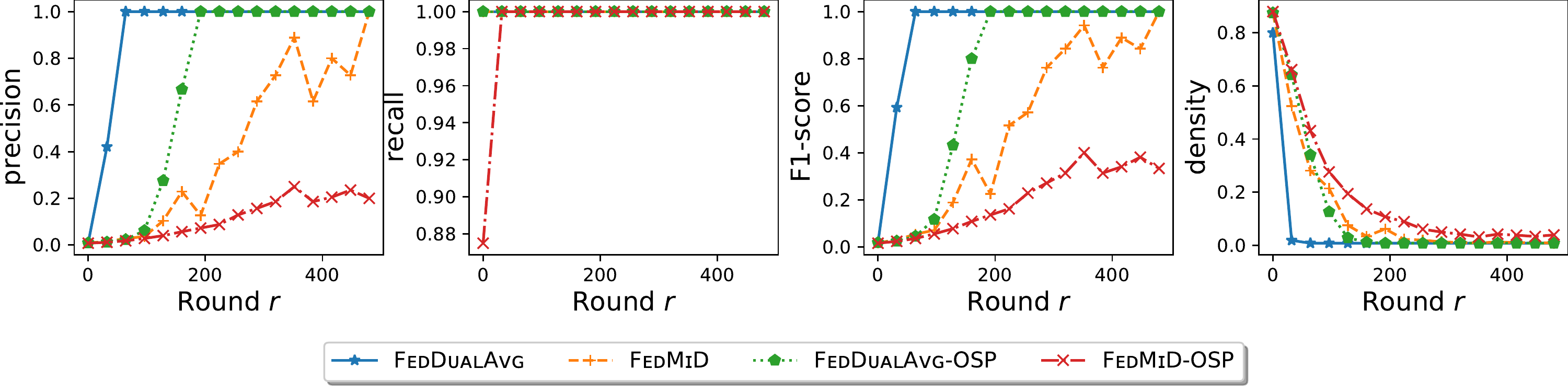}
  \caption{\textbf{Results on Dataset (III): $1/2^{7}$ Ground Truth Density.} See \cref{sec:expr:lasso:2:3} for discussions.}
  \label{fig:lasso_p1024_nnz8_cl64_1x4}
\end{figure}

\subsubsection{Results on Synthetic Dataset (IV): More Distributed Data (256 clients)}
\label{sec:expr:lasso:4}
We repeat the experiments on the dataset (IV) with more distributed data (256 clients). 
The results are shown in \cref{fig:lasso_p1024_nnz512_cl256_1x4}. 
We observe that all the four algorithms take more rounds to converge in that each client has fewer data than the previous configurations. 
\feddualavg manages to find perfect F1-score in less than 200 rounds, which outperforms the other algorithms significantly.
\feddualavg-OSP can recover an almost perfect F1-score after 500 rounds, but is much slower than on the less distributed dataset (I).
\fedmid and \fedmid-OSP have very limited progress within 500 rounds. 
This is because the server averaging step in \fedmid and \fedmid-OSP fails to aggregate the sparsity patterns properly. 
Since each client is subject to larger noise due to the limited amount of local data, simply averaging the primal updates will ``smooth out'' the sparsity pattern.

\begin{figure}[!hbp]
  \centering
  \includegraphics[width=\textwidth]{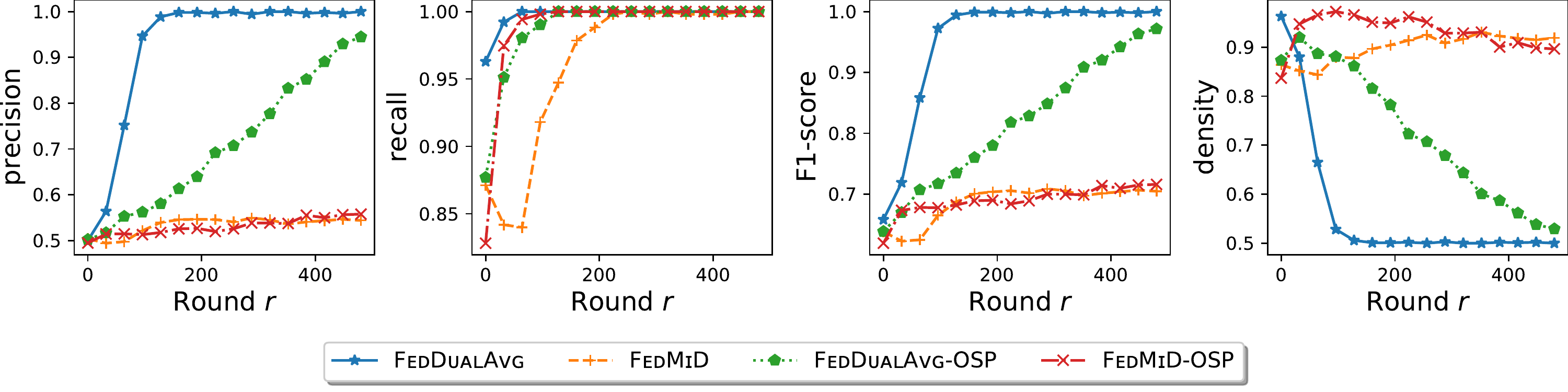}
  \caption{\textbf{Results on Dataset (IV): More Distributed Data.} See \cref{sec:expr:lasso:4} for discussions.}
  \label{fig:lasso_p1024_nnz512_cl256_1x4}
\end{figure}

\subsection{Nuclear-Norm-Regularization for Low-Rank Matrix Estimation}
\label{sec:expr:nuclear}
\subsubsection{Setup Details}
In this subsection, we consider a low-rank matrix estimation problem via the nuclear-norm regularization
\begin{equation}
  \min_{W \in \reals^{d_1 \times d_2}, b\in \reals} \frac{1}{M} \sum_{m=1}^M \mathbb{E}_{(X,y) \sim \mathcal{D}_m} \left( \langle X, W \rangle + b - y \right)^2 + \lambda \|W\|_{\mathrm{nuc}},
  \label{eq:fed:matrix:estimation}
\end{equation}
where $\|W\|_{\mathrm{nuc}} := \sum_i \sigma_i(W)$ denotes the nuclear norm (a.k.a. trace norm) defined by the summation of all the singular values. 
The goal is to recover a low-rank matrix $W$ from noisy observations $(X,y)$. 
This formulation captures a variety of problems, such as low-rank matrix completion and recommendation systems (c.f. \citealt{Candes.Recht-09}).
Note that the proximal operator with respect to the nuclear-norm regularizer $\|\cdot\|_{\mathrm{nuc}}$ reduces to the well-known singular-value thresholding operation \citep{Cai.Candes.ea-SIOPT10}.

\paragraph{Synthetic Dataset Descriptions.}
We first generate the following ground truth $W_{\mathrm{real}} \in \reals^{d \times d}$ of rank $r$ 
\begin{equation}
  W_{\mathrm{real}} = \begin{bmatrix}
    I_{r \times r} & \mathbf{0}_{r \times (d - r)} \\
    \mathbf{0}_{(d-r) \times r} & \mathbf{0}_{(d-r) \times (d-r)}
  \end{bmatrix},
\end{equation}
and ground truth $b_{\mathrm{real}} \sim \mathcal{N}(0,1)$. 

The observations $(X,y)$ are generated as follows to simulate the heterogeneity among clients. 
Let $(X_{m}^{(i)}, y_{m}^{(i)})$ denotes the $i$-th observation of the $m$-th client. 
For each client $m$, we first generate and fix the mean $\mu_m \in \reals^{d \times d}$ where all coordinates are i.i.d. standard Gaussian $\mathcal{N}(0,1)$. 
Then we sample $n_m$ pairs of observations following
\begin{align}
  & X_m^{(i)} = \mu_m + \delta_m^{(i)}, \text{where $\delta_m^{(i)} \in \reals^{d \times d}$ is a matrix with all coordinates from standard Gaussian;}
  \\
  & y_m^{(i)} = \langle w_{\mathrm{real}}, X_m^{(i)} \rangle + b_{\mathrm{real}} + \varepsilon_m^{(i)}, \text{where $\varepsilon_m^{(i)} \sim \mathcal{N}(0,1)$ are i.i.d.}
\end{align}

We tested four configurations of the above synthetic dataset.
\begin{enumerate}
  \item [(I)] The ground truth $W_{\mathrm{real}}$ is a matrix of dimension $32 \times 32$ with rank $r=16$. 
  We generate $M=64$ training clients where each client possesses $128$ pairs of samples.
  There are 8,192 training samples in total.
  \item [(II)] (rank-4 ground truth) The ground truth $W_{\mathrm{real}}$ has rank $r=4$. The other configurations are the same as the dataset (I).
  \item [(III)] (rank-1 ground truth) The ground truth $W_{\mathrm{real}}$ has rank $r=1$. The other configurations are the same as the dataset (I).
  \item [(IV)] (more distributed data) The ground truth is the same as (I). We generate $M=256$ training clients where each client possesses $32$ samples. The total number of training examples remains the same.
\end{enumerate}

\paragraph{Evaluation Metrics.}
We focus on four metrics for this task: the training (regularized) loss, the validation mean-squared-error, the recovered rank, and the recovery error in Frobenius norm $\|W_{\mathrm{output}} - W_{\mathrm{real}}\|_{\mathrm{F}}$.
To numerically evaluate the rank, we count the number of singular values that are greater than $10^{-2}$.

\paragraph{Hyperparameters.}
For all algorithms, we tune the client learning rate $\eta_{\client}$ and server learning rate $\eta_{\server}$ only. 
We test 49 different combinations of $\eta_{\client}$ and $\eta_{\server}$.  
$\eta_{\client}$ is selected from $\{0.001, 0.003, 0.01, 0.03, 0.1, 0.3, 1 \}$, and $\eta_{\server}$ is selected from $\{0.01, 0.03, 0.1, 0.3, 1, 3, 10\}$. 
All methods are tuned to achieve the best averaged recovery error on the last 100 communication rounds.
We claim that the best learning rate combination falls in this range for all algorithms tested.
We draw 10 clients uniformly at random at each communication round and let the selected clients run local algorithms with batch size 10 for one epoch (of its local dataset) for this round. 
We run 500 rounds in total, though \feddualavg usually converges to perfect F1-score in much fewer rounds.

The \cref{fig:nuclear_row32_rank16_cl64_1x4} presented in the main paper (\cref{mainsec:expr:nuclear}) is for the synthetic dataset (I). Now we test the performance of the algorithms on the other three datasets.

\subsubsection{Results on Synthetic Dataset (II) and (III) with Ground Truth of Lower Rank}
\label{sec:expr:nuclear:2:3}
We repeat the experiments on the dataset (II) and (III) with 4 and 1 ground truth rank, respectively.
The results are shown in \cref{fig:nuclear_row32_rank4_cl64_1x4,fig:nuclear_row32_rank1_cl64_1x4}. 
The results are qualitatively reminiscent of the previous experiments on the dataset (I). 
\feddualavg can recover the exact rank in less than 100 rounds, which outperforms the other baselines by a margin. 
\feddualavg-OSP can recover a low-rank solution but is less accurate.
The convergence of \fedmid and \fedmid-OSP remains slow.
\subsubsection{Results on Synthetic Dataset (IV): More Distributed Data (256 clients)}
\label{sec:expr:nuclear:4}
We repeat the experiments on the dataset (IV) with more distributed data. 
The results are shown in \cref{fig:nuclear_row32_rank16_cl256_1x4}. 
We observe that all four algorithms take more rounds to converge in that each client has fewer data than the previous configurations. 
The other messages are qualitatively similar to the previous experiments -- 
\feddualavg manages to find exact rank in less than 200 rounds, which outperforms the other algorithms significantly.
% This is because the server averaging step in \fedmid and \fedmid-OSP fails to aggregate the sparsity patterns properly. 
% Since each client is subject to larger noise due to the limited amount of local data, simply averaging the primal updates will ``smooth out'' the sparsity pattern.
\begin{figure}[ht]
  \centering
  \includegraphics[width=\textwidth]{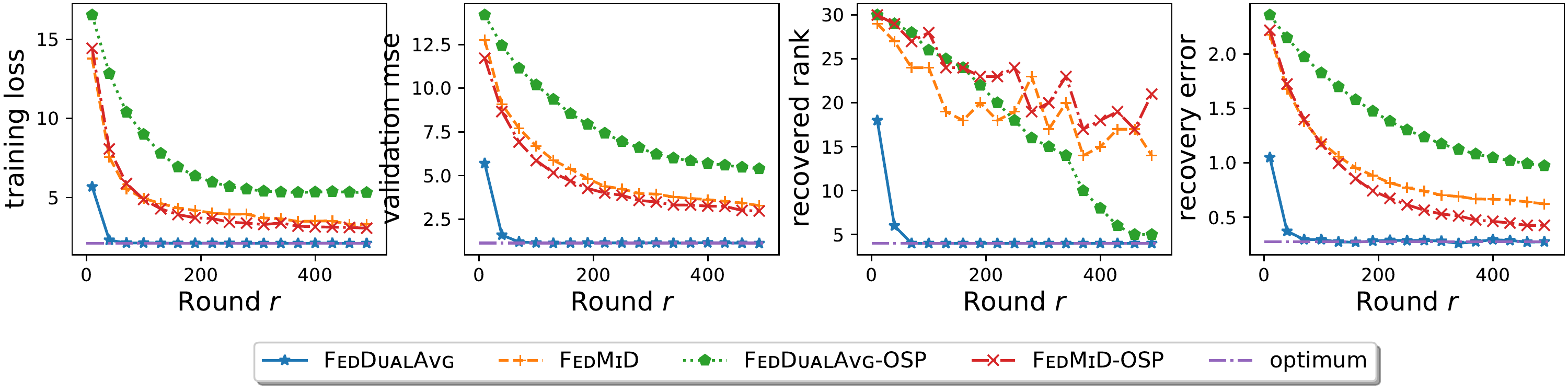}
  \caption{\textbf{Results on Dataset (II): Ground Truth Rank 4.} See \cref{sec:expr:nuclear:2:3} for discussions.}
  \label{fig:nuclear_row32_rank4_cl64_1x4}
\end{figure}
\begin{figure}[ht]
  \centering
  \includegraphics[width=\textwidth]{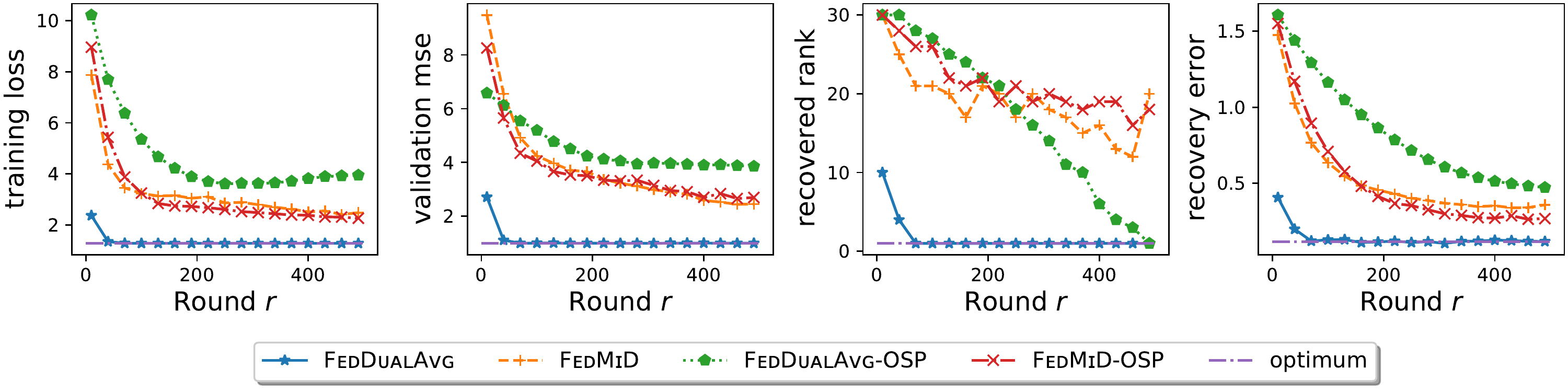}
  \caption{\textbf{Results on Dataset (III): Ground Truth Rank 1.} See \cref{sec:expr:nuclear:2:3} for discussions.}
  \label{fig:nuclear_row32_rank1_cl64_1x4}
\end{figure}
\begin{figure}[ht]
    \centering
    \includegraphics[width=\textwidth]{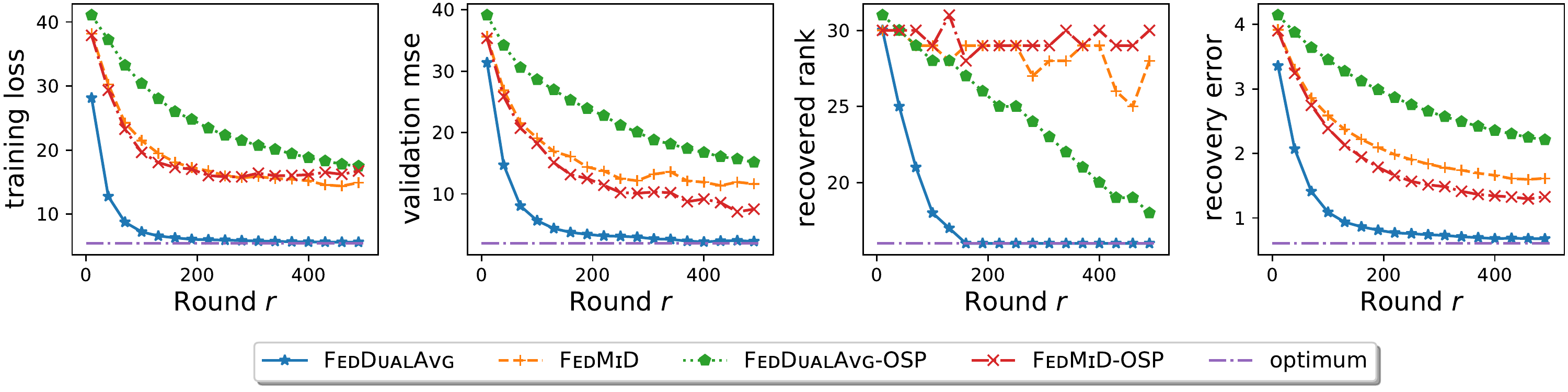}
    \caption{\textbf{Results on Dataset (IV): More Distributed Data.} See \cref{sec:expr:nuclear:4} for discussions.}
    \label{fig:nuclear_row32_rank16_cl256_1x4}
\end{figure}

\subsection{Sparse Logistic Regression for fMRI}
\label{sec:expr:fmri}
\subsubsection{Setup Details}
In this subsection, we provide the additional setup details for the fMRI experiment presented in \cref{fig:haxby59_lambd_1e-03_1x4}.
The goal is to understand the pattern of response in ventral temporal area of the brain given a visual stimulus.
Enforcing sparsity is important as it allows domain experts to understand which part of the brain is differentiating between the stimuli.
We apply $\ell_1$-regularized logistic regression on the voxels to classify the visual stimuli.

\paragraph{Dataset Descriptions and Preprocessing.}
  We use data collected by \citet{Haxby-01}. 
  There were 6 subjects doing binary image recognition (from a horse and a face) in a block-design experiment over 11-12 sessions per subject, in which each session consists of 18 scans.
  We use \texttt{nilearn} package \citep{Abraham.Pedregosa.ea-14} to normalize and transform the 4-dimensional raw fMRI scan data into an array with 39,912 volumetric pixels (voxels) using the standard mask.
  We choose the first 5 subjects as training set and the last subject as validation set.
  To simulate the cross-silo federated setup, we treat each session as a client. There are 59 training clients and 12 test clients, where each client possesses the voxel data of 18 scans.
  
\paragraph{Evaluation Metrics.} We focus on three metrics for this task: validation (regularized) loss, validation accuracy, and (sparsity) density.
To numerically evaluate the density, we treat all weights with absolute values smaller than $10^{-4}$ as zero elements.
The density is computed as non-zero parameters divided by the total number of parameters.

\paragraph{Hyperparameters.} For all algorithms, we adjust only client learning rate $\eta_{\client}$ and server learning rate $\eta_{\server}$. 
For each federated setup, we tested 49 different combinations of $\eta_{\client}$ and $\eta_{\server}$.  $\eta_{\client}$ is selected from $\{0.001, 0.003, 0.01, 0.03, 0.1, 0.3, 1\}$, and $\eta_{\server}$ is selected from $\{0.01, 0.03, 0.1, 0.3, 1, 3, 10\}$. 
We let each client run its local algorithm with batch-size one for one epoch per round. 
At the beginning of each round, we draw 20 clients uniformly at random.
We run each configuration for 300 rounds and present the configuration with the lowest validation (regularized) loss at the last round. 

We also tested two non-federated baselines for comparison, marked as \texttt{centralized} and \texttt{local}. 
\texttt{centralized} corresponds to training on the centralized dataset gathered from \textbf{all} the 59 training clients.
\texttt{local} corresponds to training on the local data from only \textbf{one} training client without communication. 
We run proximal gradient descent for these two baselines for 300 epochs.
The learning rate is tuned from $\{0.0001, 0.0003, 0.001, 0.003, 0.01, 0.03, 0.1, 0.3, 1\}$ to attain the best validation loss at the last epoch.
The results are presented in \cref{fig:haxby59_lambd_1e-03_1x4}.

\subsubsection{Progress Visualization across Various Learning Rate Configurations}
\label{sec:lr:config}
In this subsection, we present an alternative viewpoint to visualize the progress of federated algorithms and understand the robustness to hyper-parameters.
To this end, we run four algorithms for various learning rate configurations (we present all the combinations of learning rates mentioned above such that $\eta_{\client} \eta _{\server} \in [0.003, 0.3]$) and record the validation accuracy and (sparsity) density after 10th, 30th, 100th, and 300th round. 
The results are presented in \cref{fig:haxby59_lambd_1e-03_multi_1x4}. 
Each dot stands for a learning rate configuration (client and server). 
We can observe that most \feddualavg configurations reach the upper-left region of the box, which indicates sparse and accurate solutions. 
\feddualavg-OSP reaches to the mid-left region of the box, which indicates sparse but less accurate solutions.
The majority of \fedmid and \fedmid-OSP lands on the right side region box, which reflects the hardness for \fedmid and \fedmid-OSP to find the sparse solutions.

\begin{figure}[!hbp]
    \centering
    \includegraphics[width=\textwidth]{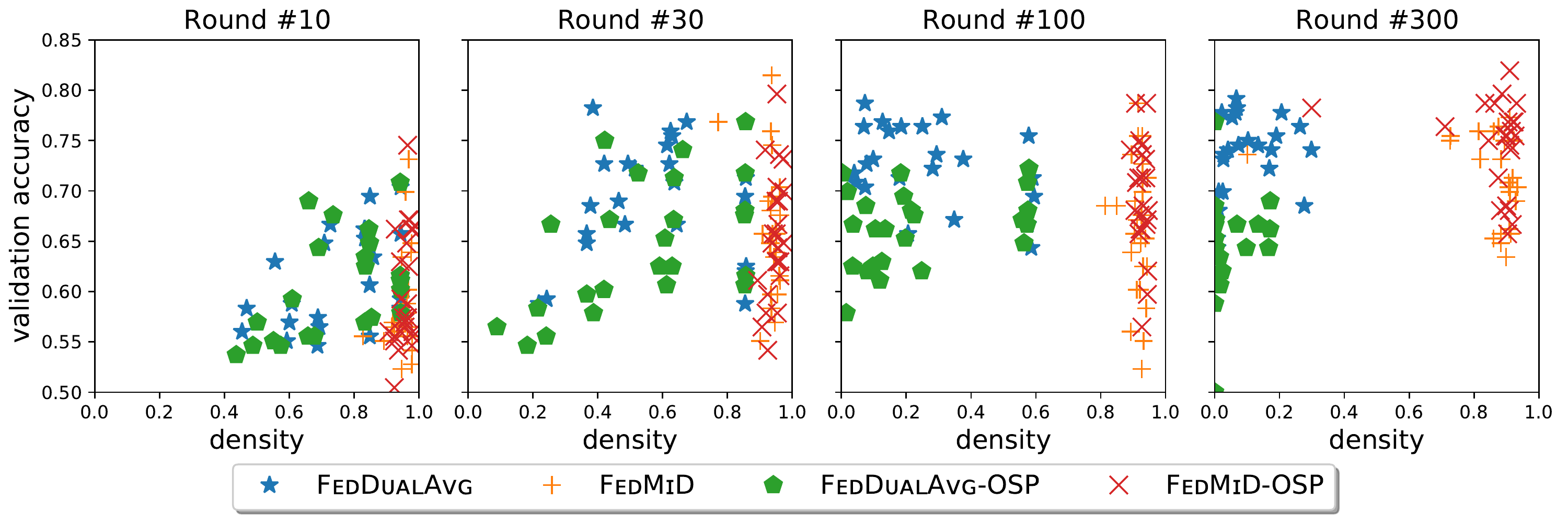}
    \caption{\textbf{Progress of Federated Algorithms Under Various Learning Rate Configurations for fMRI.} Each dot stands for a learning rate configuration (client and server). \feddualavg recovers sparse and accurate solutions, and is robust to learning-rate configurations.
    }
    \label{fig:haxby59_lambd_1e-03_multi_1x4}
\end{figure}

\subsection{Constrained Federated Optimization for Federated EMNIST}
\label{sec:emnist}
\subsubsection{Setup Details}
In this task we test the performance of the algorithms when the composite term $\psi$ is taken to be convex characteristics
$\chi_{\cstr}(w) := 
  \begin{cases} 
      0 & \text{if $w \in \mathcal{C}$}, \\
      +\infty & \text{if $w \notin \mathcal{C}$}.
  \end{cases}
$
which encodes a hard constraint. 
  
\paragraph{Dataset Descriptions and Models.} 
We tested on the Federated EMNIST (FEMNIST) dataset provided by TensorFlow Federated, which was derived from the Leaf repository \citep{Caldas.Duddu.ea-NeurIPS19}. 
EMNIST is an image classification dataset that extends MNIST dataset by incorporating alphabetical classes. 
The Federated EMNIST dataset groups the examples from EMNIST by writers.

We tested two versions of FEMNIST in this work:
\begin{enumerate}
    \item[(I)] FEMNIST-10: digits-only version of FEMNIST which contains 10 label classes. %There are 3,383 clients, with a total of 40,832 examples.
    We experiment the logistic regression models with $\ell_1$-ball-constraint or $\ell_2$-ball-constraint on this dataset. 
    Note that for this task we only trained on 10\% of the examples in the original FEMNIST-10 dataset because the original FEMNIST-10 has an unnecessarily large number (340k) of examples for the logistic regression model.
    \item[(II)] FEMNIST-62: full version of FEMNIST which contains 62 label classes (including 52 alphabetical classes and 10 digital classes). 
     We test a two-hidden-layer fully connected neural network model where all fully connected layers are simultaneously subject to $\ell_1$-ball-constraint. 
    Note that there is no theoretical guarantee for either of the four algorithms on non-convex objectives. We directly implement the algorithms as if the objectives were convex. 
    We defer the study of \fedmid and \feddualavg for non-convex objectives to the future work.
    % There are 3,400 clients, with a total of 671,585 examples.
\end{enumerate}

\paragraph{Evaluation Metrics.}
We focused on three metrics for this task: training error, training accuracy, and test accuracy. 
Note that the constraints are always satisfied because all the trajectories of all the four algorithms are always in the feasible region.

\paragraph{Hyperparameters.}
For all algorithms, we tune only the client learning rate $\eta_{\client}$ and server learning rate $\eta_{\server}$. For each setup, we tested 25 different combinations of $\eta_{\client}$ and $\eta_{\server}$.  $\eta_{\client}$ is selected from $\{0.001, 0.003, 0.01, 0.03, 0.1\}$, and $\eta_{\server}$ is selected from $\{0.01, 0.03, 0.1, 0.3, 1\}$. 
We draw 10 clients uniformly at random at each communication round and let the selected clients run local algorithms with batch size 10 for 10 epochs (of its local dataset) for this round. 
We run 5,000 communication rounds in total and evaluate the training loss every 100 rounds.
All methods are tuned to achieve the best averaged training loss on the last 10 checkpoints.

\subsubsection{Experimental Results}
\paragraph{$\ell_1$-Constrained Logistic Regression}
We first test the $\ell_1$-regularized logistic regression. The results are shown in \cref{fig:emnist-lr-l1-1000}. 
We observe that \feddualavg outperforms the other three algorithms by a margin. 
Somewhat surprisingly, we observe that the other three algorithms behave very closely in terms of the three metrics tested. 
This seems to suggest that the client proximal step (in this case projection step) might be saved in \fedmid. 

\paragraph{$\ell_2$-Constrained Logistic Regression}
Next, we test the $\ell_2$-regularized logistic regression. The results are shown in \cref{fig:emnist-lr-l2-10}. 
We observe that \feddualavg outperforms the \fedmid and \fedmid-OSP in all three metrics (note again that \fedmid and \fedmid-OSP share very similar trajectories). 
Interestingly, the \feddualavg-OSP behaves much worse in training loss than the other three algorithms, but the training accuracy and validation accuracy are better. 
We conjecture that this effect might be attributed to the homogeneous property of $\ell_2$-constrained logistic regression which \feddualavg-OSP can benefit from.

\paragraph{$\ell_1$-Constrained Two-Hidden-Layer Neural Network}
Finally, we test on the two-hidden-layer neural network with $\ell_1$-constraints. 
The results are shown in \cref{fig:emnist-2NN-l1-1000}.
We observe that \feddualavg outperforms \fedmid and \fedmid-OSP in all three metrics (once again, note that \fedmid and \fedmid-OSP share similar trajectories). 
On the other hand, \feddualavg-OSP behaves much worse (which is out of the plotting ranges). 
This is not quite surprising because \feddualavg-OSP does not have any theoretical guarantees.
\begin{figure}
    \centering
    \includegraphics[width=\textwidth]{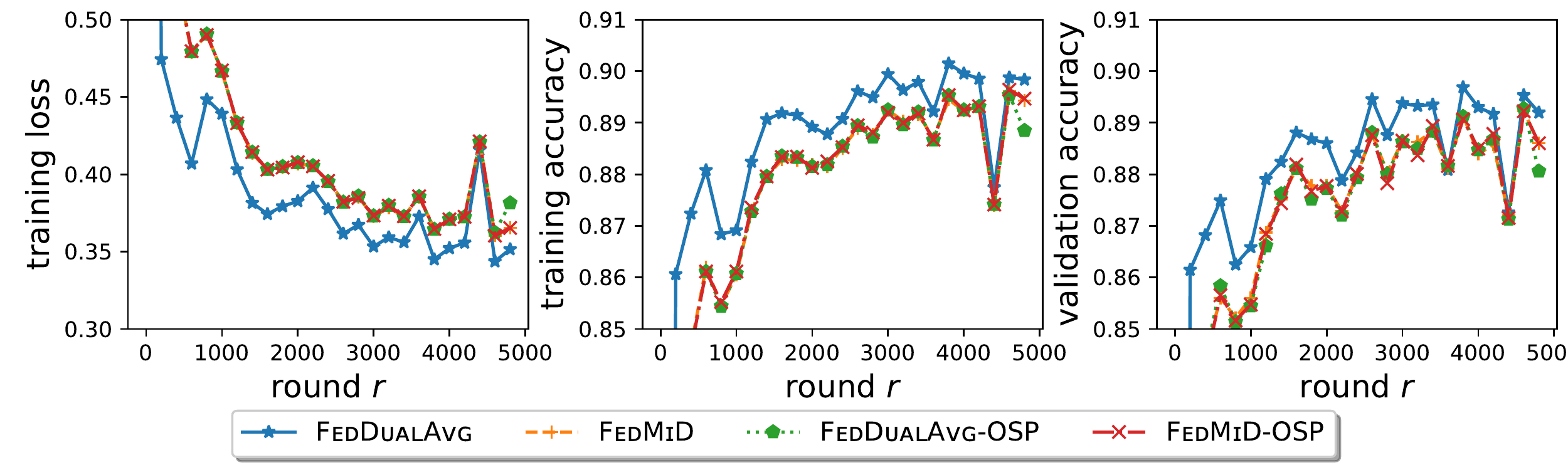}
    \caption{
    \textbf{$\ell_1$-Constrained logistic regression.}
    Dataset: FEMNIST-10.
    Constraint: $\|w\|_1 \leq 1000$.
    }
    \label{fig:emnist-lr-l1-1000}
\end{figure}
\begin{figure}
    \centering
    \includegraphics[width=\textwidth]{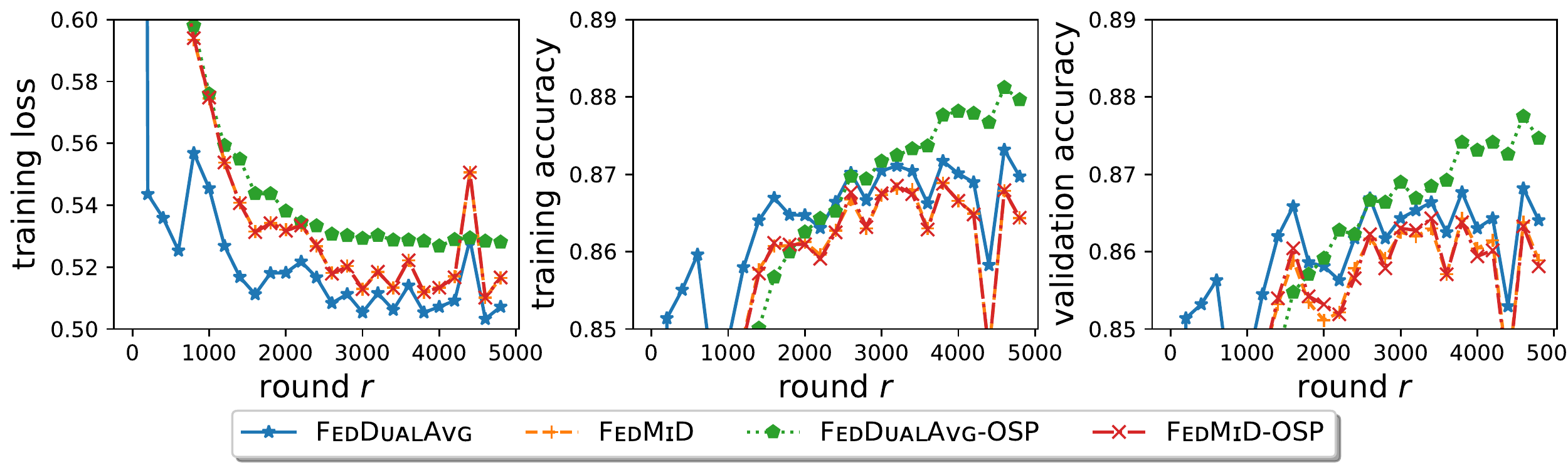}
    \caption{
    \textbf{$\ell_2$-constrained logistic regression.} 
    Dataset: FEMNIST-10.
    Constraint: $\|w\|_2 \leq 10$.
    }
    \label{fig:emnist-lr-l2-10}
\end{figure}
\begin{figure}
    \centering
    \includegraphics[width=\textwidth]{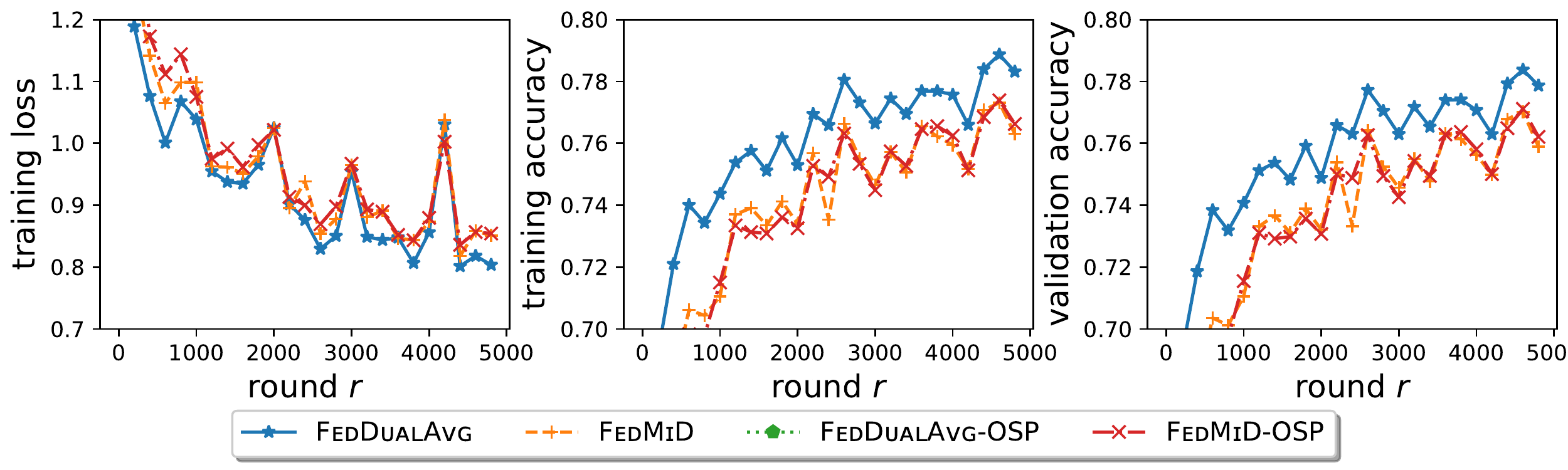}
    \caption{
    \textbf{$\ell_1$-Constrained Two-Hidden-Layer Neural Network.}
  Dataset: FEMNIST-62.
  Constraint: all three dense kernels 
    $w^{[l]}$ simultaneously satisfy $\|w^{[l]}\|_1 \leq 1000$.
    }
    \label{fig:emnist-2NN-l1-1000}
\end{figure}
\clearpage

%% file: alg_fedmid_osp.tex
% !TEX root = main.tex  
\begin{algorithm}
  \caption{\fedmidfull Only Server Proximal (\fedmid-OSP)}
  \begin{algorithmic}[1]
    \label{alg:fedmid-osp}
    \STATE {\textbf{procedure}} \fedmid-OSP ($w_0, \eta_{\client}, \eta_{\server}$)
      \FOR{$r = 0, \ldots, R-1$}
        \STATE sample a subset of clients $\mathcal{S}_r \subseteq [M]$
        \FORALL {$m \in \mathcal{S}_r$ {\bf in parallel}}
          \STATE client initialization $w_{r,0}^m \gets w_{r}$
          \COMMENT{Broadcast primal initialization for round $r$}
          \FOR{$k = 0, \ldots, K-1$}
            \STATE $g_{r,k}^m \gets \nabla f(w_{r,k}^m; \xi_{r,k}^m)$
            \COMMENT{Query gradient}
            \STATE $w_{r,k+1}^m \gets \nabla h^*(\nabla h (w_{r,k}^m) - \eta_{\client}g_{r,k}^m$)
            \COMMENT{Client (primal) update -- proximal operation skipped}
          \ENDFOR
        \ENDFOR
        \STATE $\Delta_r = \frac{1}{|\mathcal{S}_r|} \sum_{m \in \mathcal{S}_r} (w_{r, K}^m - w_{r, 0}^m)$
        \COMMENT{Compute pseudo-anti-gradient}
        \STATE $w_{r+1} \gets \nabla (h + \eta_{\server}\eta_{\client}K \psi)^*(\nabla h(w_{r}) + \eta_{\server} \Delta_r)$
        \COMMENT{Server (primal) update}
      \ENDFOR
  \end{algorithmic}
\end{algorithm}

%% file: alg_feddualavg_osp.tex
% !TEX root = main.tex  
\begin{algorithm}
  \caption{\feddualavgfull Only Server Proximal (\feddualavg-OSP)}
  \begin{algorithmic}[1]
      \label{alg:feddualavg-osp}
    \STATE \textbf{procedure} {\feddualavg-OSP}($w_0, \eta_{\client}, \eta_{\server}$)
      \STATE server initialization $z_{0} \gets \nabla h(w_0)$
      \FOR{$r = 0, \ldots, R-1$}
        \STATE sample a subset of clients $\mathcal{S}_r \subseteq [M]$
        \FORALL{$m \in \mathcal{S}_r$ {\bf in parallel}}
          \STATE client initialization $z_{r,0}^m \gets z_{r}$ 
           \COMMENT{Broadcast dual initialization for round $r$}
          \FOR{$k = 0, \ldots, K-1$}
            \STATE $w_{r,k}^m \gets \nabla h^*(z_{r,k}^m)$
            \COMMENT{Compute primal point $w_{r,k}^m$ -- proximal operation skipped}
            \STATE $g_{r,k}^m \gets \nabla f (w_{r,k}^m; \xi_{r,k}^m) $
            \COMMENT{Query gradient}
            \STATE $z_{r,k+1}^m \gets z_{r,k}^m - \eta_{\client} g_{r,k}^m$
            \COMMENT{Client (dual) update}
          \ENDFOR
        \ENDFOR
        \STATE 
        $\Delta_r = \frac{1}{|\mathcal{S}_r|} \sum_{m \in \mathcal{S}_r} (z_{r, K}^m - z_{r, 0}^m)$
        \COMMENT{Compute pseudo-anti-gradient}
        \STATE $z_{r+1} \gets z_{r} + \eta_{\server} \Delta_r$
        \COMMENT{Server (dual) update}
        \STATE $w_{r+1} \gets \nabla (h + \eta_{\server} \eta_{\client} (r+1) K \psi)^* (z_{r+1})$
        \COMMENT{(Optional) Compute server primal state}
      \ENDFOR
  \end{algorithmic}
\end{algorithm}

%% file: technicalities.tex
\section{Theoretical Background and Technicalities}
\label{sec:background}
In this section, we introduce some definitions and propositions that are necessary for the proof of our theoretical results. Most of the definitions and results are standard and can be found in the classic convex analysis literature (e.g., \citealt{Rockafellar-70,Hiriart-Urruty.Lemarechal-01}), unless otherwise noted.

The following definition of the \emph{effective domain} extends the notion of \emph{domain} (of a finite-valued function) to an extended-valued convex function $\reals^d \to \reals \cup \{+ \infty\}$. 
\begin{definition}[Effective domain]
  Let $g: \reals^d \to \reals \cup \{+ \infty\}$ be an extended-valued convex function. The \textbf{effective domain} of $g$, denoted by $\dom g$, is defined by
  \begin{equation}
    \dom g := \{w \in \reals^d: g(w) < +\infty   \}.
  \end{equation}
\end{definition}
In this work we assume all extended-valued convex functions discussed are \textbf{proper}, namely the effective domain is nonempty.

Next, we formally define the concept of \emph{strict} and \emph{strong convexity}. Note that the strong convexity is parametrized by some parameter $\mu > 0$ and therefore implies strict convexity. 
\begin{definition}[Strict and Strong convexity {\citep[Definition B.1.1.1]{Hiriart-Urruty.Lemarechal-01}}]
  A convex function $g: \reals^d \to \reals \cup \{+ \infty\}$ is \textbf{strictly convex} if for any $w_1, w_2 \in \dom g$, for any $\alpha \in (0,1)$, it is the case that 
  \begin{equation}
    g(\alpha w_1 + (1- \alpha) w_2) < \alpha g (w_1) + (1-\alpha) g(w_2).
  \end{equation}
  Moreover, $g$ is $\mu$-\textbf{strongly convex} with respect to $\|\cdot\|$ norm if for any $w_1, w_2 \in \dom g$, for any $\alpha \in (0,1)$, it is the case that 
  \begin{equation}
    g(\alpha w_1 + (1- \alpha) w_2) \leq \alpha g(w_1) + (1- \alpha) g(w_2) - \frac{1}{2} \mu \alpha (1-\alpha) \|w_2 - w_1\|^2.
  \end{equation}
\end{definition}

The notion of \emph{convex conjugate} (a.k.a. \emph{Legendre-Fenchel transformation}) is defined as follows. The outcome of convex conjugate is always convex and closed.
\begin{definition}[Convex conjugate]
  Let $g: \reals^d \to \reals \cup \{+\infty\}$ be a convex function. The convex conjugate is defined as
  \begin{equation}
    g^*(z) := \sup_{w \in \reals^d} \left\{ \left \langle z, w \right \rangle - g(w) \right\}.
  \end{equation}
\end{definition}

The following result shows that the differentiability of the conjugate function and the strict convexity of the original function is linked.
\begin{proposition}[Differentiability of the conjugate of strictly convex function {\citep[Theorem E.4.1.1]{Hiriart-Urruty.Lemarechal-01}}]
  Let $g$: $\reals^d \to \reals \cup \{+\infty \}$ be a closed, strictly convex function. 
  Then we have $\interior \dom g^* \neq \emptyset$ and $g^*$ is continuously differentiable on $\interior \dom g^*$ (where $\interior$ stands for interior). 
  
  Moreover, for $z \in \interior \dom g^*$, it is the case that
  \begin{equation}
    \nabla g^*(z) = \argmin_{w} \left\{ \left\langle -z, w \right\rangle + g(w)\right\}.
  \end{equation}
  \label{conjugate:strictly:convex}
\end{proposition}

The differentiability in \cref{conjugate:strictly:convex} can be strengthened to smoothness if we further assume the strong convexity of the original function $g$.
\begin{proposition}[Smoothness of the conjugate of strongly convex function {\citep[Theorem E.4.2.1]{Hiriart-Urruty.Lemarechal-01}}]
  \label{strongly:convex:conjugate}
  Let $g: \reals^d \to \reals \cup \{+\infty \}$ be a closed, $\mu$-strongly convex function. 
  Then $g^*$ is continuously differentiable on $\reals^d$, and $g^*$ is $\frac{1}{\mu}$-smooth on $\reals^d$, namely
  $\|\nabla g^*(z) - \nabla g^*(y) \|^* \leq \frac{1}{\mu} \|z-y\|$.
\end{proposition}

Next we define the \emph{Legendre function class}. 
\begin{definition}[Legendre function {\citep[§26]{Rockafellar-70}}]  
  \label{def:legendre}
  A proper, convex, closed function $h: \reals^d \to \reals \cup \{+\infty\}$ is \textbf{of Legendre type} if 
  \begin{enumerate}
    \item [(a)] $h$ is \textbf{strictly convex}.
    \item [(b)] $h$ is \textbf{essentially smooth}, namely $h$ is differentiable on $\interior \dom h$, and $\| \nabla h(w_k) \| \to \infty$ for every sequence $\{w_k\}_{k=0}^{\infty} \subset \interior \dom h$ converging to a boundary point of $\dom h$ as $k \to +\infty$.
  \end{enumerate}
\end{definition}

An important property of the Legendre function is the following proposition \citep{Bauschke.Borwein.ea-JCA97}.
\begin{proposition}[{\citet[Theorem 26.5]{Rockafellar-70}}]
  \label{prop:legendre}
  A convex function $g$ is of Legendre type if and only if its conjugate $g^*$ is. 
  In this case, the gradient mapping $\nabla g$ is a toplogical isomorphism with inverse mapping, namely $(\nabla g)^{-1} = \nabla g^*$.
\end{proposition}

Next, recall the definition of Bregman divergence:
\begin{definition}[Bregman divergence {\citep{Bregman-67}}]
  \label{def:bregman}
  Let $g: \reals^d \to \reals \cup \{+\infty\}$ be a closed, strictly convex function that is differentiable in $\interior \dom g$. The \textbf{Bregman divergence} $D_g(w, u)$ for $w \in \dom g$, $u \in \interior \dom g$ is defined by 
  \begin{equation}
    D_g(w, u) = g(w) - g(u) - \left \langle \nabla g(u), w - u \right \rangle.
  \end{equation}
\end{definition}

Note the definition of Bregman divergence requires the differentiability of the base function $g$.
To extend the concept of Bregman divergence to non-differentiable function $g$, we consider the following generalized Bregman divergence (slightly modified from \citep{Flammarion.Bach-COLT17}). 
The generalized Bregman divergence plays an important role in the analysis of \feddualavg.
\begin{definition}[Generalized Bregman divergence {\citep[slightly modified from][Section B.2]{Flammarion.Bach-COLT17}}]
  \label{def:generalized_bregman}
  Let $g: \reals^d \to \reals \cup \{+\infty\}$ be a closed strictly convex function (which may not be differentiable).
  The \textbf{Generalized Bregman divergence} $\tilde{D}_{g}(w, z)$ for $w \in \dom g$, $z \in \interior \dom g^*$ is defined by 
  \begin{equation}
    \tilde{D}_g(w, u) = g(w) - g(\nabla g^*(z)) - \left \langle z, w - \nabla g^*(z) \right \rangle.
  \end{equation}
  Note that $\nabla g^*$ is well-defined because $g^*$ is differentiable in $\interior \dom g^*$ according to \cref{conjugate:strictly:convex}.
\end{definition}

The generalized Bregman divergence is lower bounded by the ordinary Bregman divergence in the following sense. 
\begin{proposition}[{\citep[Lemma 6]{Flammarion.Bach-COLT17}}]
  \label{generalized:bregman}
  Let $h: \reals^d \to \reals \cup \{ +\infty\}$ be a Legendre function. Let $\psi: \reals^d \to \reals$ be a convex function (which may not be differentiable). Then for any $w \in \dom h$, for any $z \in \interior \dom (h + \psi)^*$, the following inequality holds
  \begin{equation}
    \tilde{D}_{h + \psi}(w, z) \geq D_h (w, \nabla (h + \psi)^*(z)).
  \end{equation}
\end{proposition}
\begin{proof}[Proof of \cref{generalized:bregman}] 
  The proof is very similar to Lemma 6 of \citep{Flammarion.Bach-COLT17}, and we include for completeness.
  By definition of the generalized Bregman divergence (\cref{def:generalized_bregman}),
  \begin{equation}
    \tilde{D}_{h + \psi}(w,z) = (h + \psi)(w) - (h + \psi)(\nabla (h+\psi)^*(z)) - \left\langle z, w - \nabla (h + \psi)^*(z) \right\rangle.
  \end{equation}
  By definition of the (ordinary) Bregman divergence (\cref{def:bregman}),
  \begin{equation}
    D_{h}(w, \nabla (h + \psi)^*(z)) = h(w) - h( \nabla (h+\psi)^*(z)) - \left\langle \nabla h \left( \nabla (h+ \psi)^*(z) \right), w - \nabla (h + \psi)^*(z) \right\rangle.
  \end{equation}
  Taking difference,
  \begin{equation}
    \tilde{D}_{h + \psi}(w, z) - D_{h}(w, \nabla (h + \psi)^*(z))
    =
    \psi(w) - \psi \left( \nabla (h+\psi)^*(z) \right) - \left\langle z - \nabla h \left( \nabla (h+ \psi)^*(z) \right), w - \nabla (h + \psi)^*(z) \right\rangle.
    \label{eq:generalized:bregman:1}
  \end{equation}
  By \cref{conjugate:strictly:convex}, one has $z \in \partial (h + \psi) (\nabla (h + \psi)^*(z))$. 
  Since $h$ is differentiable in $\interior \dom h$, we have (by subgradient calculus)
  \begin{equation}
    z - \nabla h (\nabla (h + \psi)^*(z)) \in \partial \psi (\nabla (h+\psi)^*(z)).
  \end{equation}
  Therefore by the property of subgradient as the supporting hyperplane,
  \begin{equation}
    \psi(w) \geq \psi ( \nabla (h + \psi)^*(z)) + \left\langle z - \nabla h \left( \nabla (h + \psi)^*(z) \right) , w - \nabla \left( h + \psi \right)^*(z) \right\rangle
    \label{eq:generalized:bregman:2}
  \end{equation}
  Combining \cref{eq:generalized:bregman:1} and \cref{eq:generalized:bregman:2} yields

  \begin{equation}
    \tilde{D}_{h + \psi}(w, z) - D_{h}(w, \nabla (h + \psi)^*(z)) \geq 0,
  \end{equation}
  completing the proof.
\end{proof}

%% file: analysis_bdd_grad.tex
% !TEX root = main.tex  
\section{Proof of \cref{thm:1:simplified}: Convergence of \feddualavg Under Bounded Gradient Assumption}
\label{sec:proof:thm:1}
In this section, we provde a complete, non-asymptotic version of \cref{thm:1:simplified} with detailed proof.

We now formally state the assumptions of \cref{thm:1:simplified} for ease of reference.
\begin{assumption}[Bounded gradient]
  \;
  \label{a2} In addition to \cref{a1}, assume 
  that the gradient is $G$-uniformly-bounded, namely 
  \begin{equation}
    \sup_{w \in \dom \psi} \|\nabla f(w, \xi)\|_* \leq G
  \end{equation}
\end{assumption}
This is a standard assumption in analyzing classic distributed composite optimization \citep{Duchi.Agarwal.ea-TACON12}.

% In this section, we study the convergence of \feddualavg under \cref{a2} with unit server learning rate $\eta_{\server} = 1$.
Before we start, we introduce a few more notations to simplify the exposition and analysis throughout this section. 
Let $h_{r,k}(w) = h(w) +  (r K + k) \eta_{\client} \psi(w)$.
Let $\overline{z_{r,k}} := \frac{1}{M} \sum_{m=1}^M z_{r,k}^m$ denote the average over clients, and $\widehat{w_{r,k}} := \nabla h_{r,k}^*(\overline{z_{r,k}})$ denote the primal image of $\overline{z_{r,k}}$.
Formally, we use $\mathcal{F}_{r,k}$ to denote the $\sigma$-algebra generated by $\{z_{\rho,\kappa}^{m}: \rho < r \text{ or } (\rho = r \text{ and } \kappa \leq k), m \in [M]\}$. 

\subsection{Main Theorem and Lemmas}
Now we introduce the full version of \cref{thm:1:simplified} regarding the convergence of \feddualavg with unit server learning rate $\eta_{\server} = 1$ under bounded gradient assumptions. 
\begin{theorem}[Detailed version of \cref{thm:1:simplified}]
  \label{thm:1}
  Assume \cref{a2}, then for any initialization $w_0 \in \dom \psi$, for unit server learning rate $\eta_{\server} = 1$ and any client learning rate $\eta_{\client} \leq \frac{1}{4L}$, \feddualavg yields
  \begin{equation}
    \expt \left[  \Phi \left(  \frac{1}{KR} \sum_{r=0}^{R-1} \sum_{k=1}^K \widehat{w_{r,k}} \right) - \Phi(w^{\star}) \right] 
    \leq
    \frac{B}{\eta_{\client} KR } 
    + \frac{\eta_{\client}  \sigma^2}{M}
    + 4 \eta_{\client}^2 L (K-1)^2 G^2,
    \label{eq:thm:1:1}
  \end{equation}
  where $B := D_h(w^{\star}, w_0)$ is the Bregman divergence between the optimal $w^*$ and the initial $w_0$.

  Particularly for 
  \begin{equation}
      \eta_{\client} = \min \left\{ \frac{1}{4L}, \frac{M^{\frac{1}{2}} B^{\frac{1}{2}}}{\sigma K^{\frac{1}{2}} R^{\frac{1}{2}}} , \frac{B^{\frac{1}{3}}}{L^{\frac{1}{3}} K R^{\frac{1}{3}} G^{\frac{2}{3}}}  \right\},
  \end{equation}
  one has
  \begin{equation}
    \expt \left[  \Phi \left(  \frac{1}{KR} \sum_{r=0}^{R-1} \sum_{k=1}^K \widehat{w_{r,k}} \right) - \Phi(w^{\star}) \right]
    \leq
      \frac{4 L B}{KR} 
      +
      \frac{2\sigma B^{\frac{1}{2}}}{M^{\frac{1}{2}} K^{\frac{1}{2}} R^{\frac{1}{2}}}
      +
      \frac{5 L^{\frac{1}{3}} B^{\frac{2}{3}}  G^{\frac{2}{3}}}{R^{\frac{2}{3}}}.
      \label{eq:thm:1:2}
  \end{equation}
\end{theorem}
The proof of \cref{thm:1} is based on the following two lemmas regarding perturbed convergence and stability respectively. 

\begin{lemma}[Perturbed iterate analysis of \feddualavg]
  \label{pia:general}
  Assume \cref{a1}, 
  then for any initialization $w_0 \in \dom \psi$, 
  for any reference point $w \in \dom \psi$,
  for $\eta_{\server} = 1$, 
  for any $\eta_{\client} \leq \frac{1}{4L}$, \feddualavg yields
  \begin{equation}
    \expt \left[  \Phi \left(  \frac{1}{KR} \sum_{r=0}^{R-1} \sum_{k=1}^K \widehat{w_{r,k}} \right) - \Phi(w) \right] 
    \leq
    \frac{1}{\eta_{\client} KR } D_h(w, w_0) 
    + \frac{\eta_{\client}  \sigma^2}{M}
    + \frac{L}{M KR }   \left[ \sum_{r=0}^{R-1}  \sum_{k=0}^{K-1} \sum_{m=1}^M   \expt\left \| \overline{z_{r,k}} - z_{r,k}^m \right\|_*^2 \right].
  \end{equation}
\end{lemma}
\cref{pia:general} decomposes the convergence of \feddualavg into two parts: the first part $\frac{1}{\eta_{\client} KR} D_{h}(w, w_0) + \frac{\eta_{\client} \sigma^2}{2M} + \frac{L}{MKR}$ corresponds to the convergence rate if all clients were synchronized every iteration. The second part $\frac{L}{MKR} \sum_{r=0}^{R-1} \sum_{k=0}^{K-1} \sum_{m=1}^M \expt \| z_{r,k}^m - \overline{z_{r,k}} \|_*^2$ characterizes the stability of the algorithm. 
Note that \cref{pia:general} only assumes the blanket \cref{a1}.
We defer the proof of \cref{pia:general} to \cref{sec:pia:general}. 

The following \cref{bounded_gradient}  bounds the stability term under the additional bounded gradient assumptions.
\begin{lemma}[Stability of \feddualavg under bounded gradient assumption]
  \label{bounded_gradient}
  In the same settings of \cref{thm:1}, it is the case that
  \begin{equation}
    \frac{1}{M}\sum_{m=1}^M \expt \left\| z_{r,k}^m - \overline{z_{r,k}} \right\|_*^2 \leq 4 \eta_{\client}^2 (K-1)^2 G^2.
  \end{equation}
\end{lemma}
We defer the proof of \cref{bounded_gradient} to \cref{sec:bounded_gradient}.
With \cref{pia:general,bounded_gradient} at hands the proof of \cref{thm:1} is immediate.
\begin{proof}[Proof of \cref{thm:1}]
  \cref{eq:thm:1:1} follows immediately from \cref{pia:general,bounded_gradient} by putting $w = w^{\star}$ in \cref{pia:general}. 

  Now put 
  \begin{equation}
    \eta_{\client} =  \min \left\{ \frac{1}{4L}, 
    \frac{M^{\frac{1}{2}} B^{\frac{1}{2}}}{\sigma K^{\frac{1}{2}} R^{\frac{1}{2}}}, 
    \frac{B^{\frac{1}{3}}}{L^{\frac{1}{3}} K R^{\frac{1}{3}} G^{\frac{2}{3}}}  \right\},
  \end{equation}
  which yields
  \begin{equation}
    \frac{B}{\eta_{\client} KR} = \max \left\{ \frac{4LB }{KR},  
    \frac{\sigma B^{\frac{1}{2}}}{M^{\frac{1}{2}} K^{\frac{1}{2}} R^{\frac{1}{2}}},  
    \frac{L^{\frac{1}{3}} B^{\frac{2}{3}} G^{\frac{2}{3}}}{R^{\frac{2}{3}}}
       \right\}
    \leq
    \frac{4LB }{KR} 
    + \frac{\sigma B^{\frac{1}{2}}}{M^{\frac{1}{2}} K^{\frac{1}{2}} R^{\frac{1}{2}}}
    +   \frac{L^{\frac{1}{3}} B^{\frac{2}{3}} G^{\frac{2}{3}}}{R^{\frac{2}{3}}},
  \end{equation}
  and
  \begin{equation}
    \frac{\eta_{\client} \sigma^2}{2M} \leq {\frac{M^{\frac{1}{2}} B^{\frac{1}{2}}}{\sigma T^{\frac{1}{2}}}} \cdot \frac{\sigma^2}{2M} 
    =
    \frac{\sigma B^{\frac{1}{2}}}{2 M^{\frac{1}{2}} K^{\frac{1}{2}} R^{\frac{1}{2}} } , 
    \qquad
     4 \eta_{\client}^2 L K^2 G^2
    \leq
    4 \left( \frac{B^{\frac{1}{3}}}{L^{\frac{1}{3}} K R^{\frac{1}{3}} G^{\frac{2}{3}}}   \right)^2 LK^2 G^2
    =
    \frac{4 L^{\frac{1}{3}} B^{\frac{2}{3}} G^{\frac{2}{3}}}{R^{\frac{2}{3}}}.
  \end{equation}
  Summarizing the above three inequalities completes the proof of \cref{thm:1}.
\end{proof}

\subsection{Perturbed Iterate Analysis of \feddualavg: Proof of \cref{pia:general}}
\label{sec:pia:general}
In this subsection, we prove \cref{pia:general}. We start by showing the following \cref{one:step:analysis} regarding the one step improvement of the shadow sequence $\overline{z_{r,k}}$. 
\begin{proposition}[One step analysis of \feddualavg]
  \label{one:step:analysis}
  Under the same assumptions of \cref{pia:general}, the following inequality holds
  \begin{align}
    \expt \left[ \tilde{D}_{h_{r,k+1}}(w, \overline{z_{r,k+1}}) \middle| \mathcal{F}_{r,k} \right] 
    \leq & \tilde{D}_{h_{r,k}} (w, \overline{z_{r,k}}) - \eta_{\client} \expt \left[  \Phi(\widehat{w_{r,k+1}}) - \Phi(w) \middle| \mathcal{F}_{r,k} \right] 
    \\ 
    & 
    + \eta_{\client} L \cdot \expt \left[ \frac{1}{M}  \sum_{m=1}^M \left \| \overline{z_{r,k}} - z_{r,k}^m \right\|_*^2 \middle| \mathcal{F}_{r,k} \right] 
    + \frac{\eta_{\client}^2 \sigma^2}{M},
  \end{align}
  where $\tilde{D}$ is the generalized Bregman divergence defined in \cref{def:generalized_bregman}.
\end{proposition}
The proof of \cref{one:step:analysis} relies on the following two claims regarding the deterministic analysis of \feddualavg. We defer the proof of \cref{one:step:analysis:claim:1,one:step:analysis:claim:2} to \cref{sec:one:step:analysis:claim:1,sec:one:step:analysis:claim:2}, respectively.
\begin{claim}
  \label{one:step:analysis:claim:1}
  Under the same assumptions of \cref{pia:general}, the following inequality holds
  \begin{align}
    & \tilde{D}_{h_{r,k+1}}(w, \overline{z_{r,k+1}}) 
    \\
    = &
    \tilde{D}_{h_{r,k}}(w, \overline{z_{r,k}}) - \tilde{D}_{h_{r,k}}(\widehat{w_{r,k+1}}, \overline{z_{r,k}}) - \eta_{\client} (\psi(\widehat{w_{r,k+1}}) - \psi(w))) - \eta_{\client} \left\langle \frac{1}{M} \sum_{m=1}^{M} \nabla f(w_{r,k}^m; \xi_{r,k}^m), \widehat{w_{r,k+1}} - w  \right\rangle.
    \label{eq:one:step:analysis:claim:1}
  \end{align}
\end{claim}
\begin{claim}
  \label{one:step:analysis:claim:2}
  Under the same assumptions of \cref{pia:general}, it is the case that 
  \begin{align}
     & F(\widehat{w_{r,k+1}}) - F(w) \leq \left\langle \frac{1}{M} \sum_{m=1}^M \nabla f(w_{r,k}^m; \xi_{r,k}^m), \widehat{w_{r,k+1}} - w  \right\rangle  
     \\
        & \qquad \qquad + \left\langle \frac{1}{M} \sum_{m=1}^M \left( \nabla F_m(w_{r,k}^m) - \nabla f(w_{r,k}^m; \xi_{r,k}^m) \right), \widehat{w_{r,k+1}} - w  \right\rangle  
        + L \| \widehat{w_{r,k+1}} - \widehat{w_{r,k}} \|^2  
         +  \frac{L}{M} \sum_{m=1}^M \left \| \overline{z_{r,k}} - z_{r,k}^m \right\|_*^2.
         \label{eq:one:step:analysis:claim:2}
  \end{align}
\end{claim}
With \cref{one:step:analysis:claim:1,one:step:analysis:claim:2} at hand we are ready to prove the one step analysis \cref{one:step:analysis}.
\begin{proof}[Proof of \cref{one:step:analysis}]
  Applying \cref{one:step:analysis:claim:1,one:step:analysis:claim:2} yields 
  (summating \cref{eq:one:step:analysis:claim:1} with $\eta_c$ times of \cref{eq:one:step:analysis:claim:2}),
  \begin{align}
    \tilde{D}_{h_{r,k+1}}(w, \overline{z_{r,k+1}})
    \leq &
    \tilde{D}_{h_{r,k}}(w, \overline{z_{r,k}}) - \tilde{D}_{h_{r,k}}(\widehat{w_{r,k+1}}, \overline{z_{r,k}})
    + \eta_{\client} L \| \widehat{w_{r,k+1}} - \widehat{w_{r,k}} \|^2 
    - \eta_{\client} \left( \Phi(\widehat{w_{r,k+1}}) - \Phi(w) \right) 
    \\
    & + \eta_{\client}  \left\langle \frac{1}{M} \sum_{m=1}^M \left( \nabla F_m(w_{r,k}^m) - \nabla f(w_{r,k}^m; \xi_{r,k}^m) \right), \widehat{w_{r,k+1}} - w  \right\rangle \\
    & + \eta_{\client} L \cdot \frac{1}{M} \sum_{m=1}^M \left \| \overline{z_{r,k}} - z_{r,k}^m \right\|_*^2.
      \label{eq:proof:one:step:analysis:1}
  \end{align}
  Note that 
  \begin{equation}
    \tilde{D}_{h_{r,k}}(\widehat{w_{r,k+1}}, \overline{z_{r,k}}) \geq D_h(\widehat{w_{r,k+1}}, \nabla h_{r,k}^* ( \overline{z_{r,k}})) = D_h(\widehat{w_{r,k+1}},  \widehat{w_{r,k}}) \geq \frac{1}{2}\| \widehat{w_{r,k+1}} -  \widehat{w_{r,k}}\|^2,
  \end{equation}
  and 
  \begin{equation}
    \eta_{\client} L \| \widehat{w_{r,k+1}} - \widehat{w_{r,k}} \|^2 
    \leq
    \frac{1}{4} \| \widehat{w_{r,k+1}} - \widehat{w_{r,k}} \|^2,
  \end{equation}
  since $\eta_{\client} \leq \frac{1}{4L}$ by assumption. Therefore,
  \begin{equation}
    - \tilde{D}_{h_{r,k}}(\widehat{w_{r,k+1}}, \overline{z_{r,k}}) 
    +
    \eta_{\client} L \| \widehat{w_{r,k+1}} - \widehat{w_{r,k}} \|^2 
    \leq
    -\frac{1}{4}  \| \widehat{w_{r,k+1}} - \widehat{w_{r,k}} \|^2.
    \label{eq:proof:one:step:analysis:2}
  \end{equation}
  Plugging \cref{eq:proof:one:step:analysis:2} to \cref{eq:proof:one:step:analysis:1} gives
  \begin{align}
    \tilde{D}_{h_{r,k+1}}(w, \overline{z_{r,k+1}})
    \leq &
    \tilde{D}_{h_{r,k}}(w, \overline{z_{r,k}}) - \frac{1}{4} \| \widehat{w_{r,k+1}} - \widehat{w_{r,k}} \|^2 
    - \eta_{\client} \left( \Phi(\widehat{w_{r,k+1}}) - \Phi(w) \right) 
    \\
    & + \eta_{\client}  \left\langle \frac{1}{M} \sum_{m=1}^M \left( \nabla F_m(w_{r,k}^m) - \nabla f(w_{r,k}^m; \xi_{r,k}^m) \right), \widehat{w_{r,k+1}} - w  \right\rangle \\
    & + \eta_{\client} L \cdot \frac{1}{M} \sum_{m=1}^M \left \| \overline{z_{r,k}} - z_{r,k}^m \right\|_*^2.
      \label{eq:proof:one:step:analysis:3}
  \end{align}

  Now we take conditional expectation. Note that
  \begin{align}
        &   \expt \left[ \left \langle \frac{1}{M} \sum_{m=1}^M \nabla F_m(w_{r,k}^m) - \nabla f(w_{r,k}^m; \xi_{r,k}^m), \widehat{w_{r,k+1}} - w  \right\rangle  \middle| \mathcal{F}_{r,k} \right]
        \\
      = &  \expt \left[ \left \langle \frac{1}{M}  \sum_{m=1}^M \nabla F_m(w_{r,k}^m) - \nabla f(w_{r,k}^m; \xi_{r,k}^m), \widehat{w_{r,k+1}} - \widehat{w_{r,k}}  \right\rangle  \middle| \mathcal{F}_{r,k} \right]
      \tag{since $\expt_{\xi_{r,k}^m \sim \mathcal{D}_m} [\nabla f(w_{r,k}^m; \xi_{r,k}^m)] = \nabla F_m(w_{r,k}^m)$ }
      \\
      \leq & \expt \left[ \left\| \frac{1}{M}  \sum_{m=1}^M \nabla F_m(w_{r,k}^m) - \nabla f(w_{r,k}^m; \xi_{r,k}^m) \right\|_* \middle| \mathcal{F}_{r,k} \right]
      \cdot \expt \left[  \left\| \widehat{w_{r,k+1}} - \widehat{w_{r,k}}   \right\| \middle| \mathcal{F}_{r,k} \right]
      \tag{by definition of dual norm $\|\cdot\|_*$}
      \\
      \leq & \frac{\sigma}{\sqrt{M}}  \expt \left[ \left\| \widehat{w_{r,k+1}} - \widehat{w_{r,k}}   \right\| \middle| \mathcal{F}_{r,k} \right].
      \tag{by bounded variance assumption and independence}
  \end{align}
  Plugging the above inequality to \cref{eq:proof:one:step:analysis:3} gives
  \begin{align}
      & \expt \left[ \tilde{D}_{h_{r,k+1}}(w, \overline{z_{r,k+1}}) \middle| \mathcal{F}_{r,k} \right] 
     \\
     \leq & 
     \tilde{D}_{h_{r,k}} (w, \overline{z_{r,k}}) - \eta_{\client} \expt \left[  \Phi(\widehat{w_{r,k+1}}) - \Phi(w) \middle| \mathcal{F}_{r,k} \right] 
     + \eta_{\client} L \cdot \expt \left[ \frac{1}{M}  \sum_{m=1}^M \left \| \overline{z_{r,k}} - z_{r,k}^m \right\|_*^2 \middle| \mathcal{F}_{r,k} \right]
     \\ 
     & + \frac{\eta_{\client} \sigma}{\sqrt{M}}   \expt \left[ \left\| \widehat{w_{r,k+1}} - \widehat{w_{r,k}}   \right\| \middle| \mathcal{F}_{r,k} \right] - \frac{1}{4} \expt \left[ \| \widehat{w_{r,k+1}} - \widehat{w_{r,k}} \|^2 \middle| \mathcal{F}_{r,k} \right]
     \\
     \leq & \tilde{D}_{h_{r,k}} (w, \overline{z_{r,k}}) - \eta_{\client} \expt \left[  \Phi(\widehat{w_{r,k+1}}) - \Phi(w) \middle| \mathcal{F}_{r,k} \right] 
     + \eta_{\client} L \cdot \expt \left[ \frac{1}{M}  \sum_{m=1}^M  \left \| \overline{z_{r,k}} - z_{r,k}^m \right\|_*^2 \middle| \mathcal{F}_{r,k} \right] 
     + \frac{\eta_{\client}^2 \sigma^2}{M},
     \tag{by quadratic maximum}
  \end{align}
  completing the proof of \cref{one:step:analysis}.
\end{proof}

The \cref{pia:general} then follows by telescoping the one step analysis \cref{one:step:analysis}.
\begin{proof}[Proof of \cref{pia:general}]
  Let us first telescope \cref{one:step:analysis} within the same round $r$, from $k=0$ to $K$, which gives
  \begin{align}
    \expt \left[ \tilde{D}_{h_{r,K}}(w, \overline{z_{r,K}}) \middle| \mathcal{F}_{r,0} \right] 
    \leq & \tilde{D}_{h_{r,0}} (w, \overline{z_{r,0}}) - \eta_{\client} \sum_{k=1}^{K} \expt \left[  \Phi(\widehat{w_{r,k}}) - \Phi(w) \middle| \mathcal{F}_{r,0} \right] 
    \\
    &   + \eta_{\client} L \cdot \expt \left[ \frac{1}{M} \sum_{k=0}^{K-1} \sum_{m=1}^M  \left \| \overline{z_{r,k}} - z_{r,k}^m \right\|_*^2\middle| \mathcal{F}_{r,0} \right] 
    + \frac{\eta_{\client}^2 K \sigma^2}{M}.
  \end{align}
  Since server learning rate $\eta_{\server}=1$ we have $\overline{z_{r,K}} = \overline{z_{r+1,0}}$. Therefore we can telescope the round from $r = 0$ to $R$, which gives
  \begin{align}
    \expt \left[ \tilde{D}_{h_{R,0}}(w, \overline{z_{R,0}}) \right] 
    \leq &
    \tilde{D}_{h_{0,0}}(w,  \overline{z_{0,0}}) - \eta_{\client} \sum_{r=0}^{R-1} \sum_{k=1}^K \expt \left[  \Phi(\widehat{w_{r,k}}) - \Phi(w) \right] \\
    & + \eta_c L \cdot \expt \left[ \frac{1}{M}\sum_{r=0}^{R-1}  \sum_{k=0}^{K-1} \sum_{m=1}^M  \left \| \overline{z_{r,k}} - z_{r,k}^m \right\|_*^2 \right] 
    + \frac{\eta_{\client}^2 K  R \sigma^2}{M}.
  \end{align}
  Dividing both sides by $\eta_{\client} \cdot KR$ and rearranging 
  \begin{align}
    \frac{1}{KR }\sum_{r=0}^{R-1} \sum_{k=1}^K \expt \left[  \Phi(\widehat{w_{r,k}}) - \Phi(w) \right] 
    \leq &
    \frac{1}{\eta_{\client} KR }\left( \tilde{D}_{h_{0,0}}(w,  \overline{z_{0,0}}) - \expt \left[ \tilde{D}_{h_{R,0}}(w, \overline{z_{R,0}}) \right]  \right)
    \\
    & + L \cdot \expt \left[ \frac{1}{M KR }\sum_{r=0}^{R-1}  \sum_{k=0}^{K-1} \sum_{m=1}^M  \left \| \overline{z_{r,k}} - z_{r,k}^m \right\|_*^2 \right] 
    + \frac{\eta_{\client}  \sigma^2}{M}.
  \end{align}
  Applying Jensen's inequality on the LHS and dropping the negative term on the RHS yield
  \begin{equation}
    \expt \left[  \Phi \left(  \frac{1}{KR} \sum_{r=0}^{R-1} \sum_{k=1}^K \widehat{w_{r,k}} \right) - \Phi(w) \right] 
    \leq
    \frac{1}{\eta_{\client} KR } \tilde{D}_{h_{0,0}}(w,  \overline{z_{0,0}}) 
    + \frac{L}{M KR }   \left[ \sum_{r=0}^{R-1}  \sum_{k=0}^{K-1} \sum_{m=1}^M   \expt\left \| \overline{z_{r,k}} - z_{r,k}^m \right\|_*^2 \right] 
    + \frac{\eta_{\client}  \sigma^2}{M}.
    \label{eq:pia:1}
  \end{equation}

  Since $\overline{z_{0,0}} = \nabla h (w_0)$ and $w_0 \in \dom \psi$, we have $\nabla h_{0,0}^*(\nabla h(w_0)) = w_0$ by \cref{prop:legendre} since $h$ is assumed to be of Legendre type. Consequently
  \begin{align}
    \tilde{D}_{h_{0,0}}(w, \overline{z_0}) & = h(w) - h(\nabla h_{0,0}^*(\nabla h(w_0))) - \left\langle z_0, w - \nabla h_{0,0}^*(\nabla h(w_0)) \right\rangle
    \\
    & = h(w) - h(w_0) - \left\langle \nabla h(w_0), w - w_0 \right\rangle = D_h(w,w_0).
    \label{eq:pia:2}
  \end{align}
  Plugging \cref{eq:pia:2} back to  \cref{eq:pia:1} completes the proof of \cref{pia:general}.
\end{proof} 

\subsubsection{Deferred Proof of  \cref{one:step:analysis:claim:1}}
\label{sec:one:step:analysis:claim:1}
\begin{proof}[Proof of \cref{one:step:analysis:claim:1}]
  By definition of \feddualavg procedure, for all $m \in [M]$, $k \in \{0, 1, \ldots, K-1\}$, we have
  \begin{equation}
    z_{r,k+1}^m = z_{r,k}^m - \eta_{\client} \nabla f(w_{r,k}^m; \xi_{r,k}^m).
  \end{equation}
  Taking average over $m \in [M]$ gives (recall the overline denotes the average over clients)
  \begin{equation}
    \overline{z_{r,k+1}} = \overline{z_{r,k}} - \eta_{\client} \cdot \frac{1}{M} \sum_{m=1}^{M} \nabla f(w_{r,k}^m; \xi_{r,k}^m).
    \label{eq:one:step:analysis:1}
  \end{equation}
  
  Now we study generalized Bregman divergence $\tilde{D}_{h,k+1}(w, \overline{z_{r,k+1}})$ for any arbitrary pre-fixed $w \in \dom h_{r,k}$ 
  \begin{align}
        & \tilde{D}_{h_{r,k+1}}(w, \overline{z_{r,k+1}})
      \\
      = &
      h_{r,k+1}(w) - h_{r,k+1} \left( \nabla h_{r,k+1}^*(\overline{z_{r,k+1}}) \right) - \left\langle \overline{z_{r,k+1}}, w - \nabla h_{r,k+1}^*(\overline{z_{r,k+1}}) \right\rangle \tag{By definition of $\tilde{D}$}
      \\
      = &
      h_{r,k+1}(w) - h_{r,k+1} \left( \widehat{w_{r,k+1}} \right) - \left\langle \overline{z_{r,k+1}}, w - \widehat{w_{r,k+1}} \right\rangle
      \tag{By definition of $\widehat{w_{r,k+1}}$}
      \\
      = &
      h_{r,k+1}(w) - h_{r,k+1} \left( \widehat{w_{r,k+1}} \right) - \left\langle \overline{z_{r,k}} - \eta_{\client} \cdot \frac{1}{M} \sum_{m=1}^{M} \nabla f(w_{r,k}^m; \xi_{r,k}^m), w - \widehat{w_{r,k+1}} \right\rangle
      \tag{By \cref{eq:one:step:analysis:1}}
      \\
      = & (h_{r,k}(w) + \eta_{\client} \psi(w)) - (h_{r,k}(\widehat{w_{r,k+1}}) + \eta_{\client} \psi(\widehat{w_{r,k+1}}))  - \left\langle \overline{z_{r,k}} - \eta_{\client} \cdot \frac{1}{M} \sum_{m=1}^{M} \nabla f(w_{r,k}^m; \xi_{r,k}^m), w - \widehat{w_{r,k+1}} \right\rangle
      \tag{Since $h_{r,k+1} = h_{r,k} + \eta_{\client}\psi$ by definition of $h_{r,k+1}$} 
      \\
      = & \left[ h_{r,k}(w) - h_{r,k}(\widehat{w_{r,k}}) - \left\langle \overline{z_{r,k}}, w - \widehat{w_{r,k}} \right\rangle \right] - 
      \left[ h_{r,k}(\widehat{w_{r,k+1}}) - h_{r,k}(\widehat{w_{r,k}}) - \left\langle \overline{z_{r,k}}, \widehat{w_{r,k+1}} - \widehat{w_{r,k}}  \right\rangle \right]
      \\
      & - \eta_{\client} \left( \psi(\widehat{w_{r,k+1}}) - \psi(w) \right) - \eta_{\client} \left\langle \frac{1}{M} \sum_{m=1}^{M} \nabla f(w_{r,k}^m; \xi_{r,k}^m), \widehat{w_{r,k+1}} - w  \right\rangle
      \tag{Rearranging}
      \\
      = & \tilde{D}_{h_{r,k}}(w, \overline{z_{r,k}}) - \tilde{D}_{h_{r,k}}(\widehat{w_{r,k+1}}, \overline{z_{r,k}}) - \eta_{\client} (\psi(\widehat{w_{r,k+1}}) - \psi(w))) - \eta_{\client} \left\langle \frac{1}{M} \sum_{m=1}^{M} \nabla f(w_{r,k}^m; \xi_{r,k}^m), \widehat{w_{r,k+1}} - w  \right\rangle,
      \label{eq:feddualavg:2}
  \end{align} 
  where the last equality is by definition of $\tilde{D}$.
  \end{proof}

\subsubsection{Deferred Proof of  \cref{one:step:analysis:claim:2}}
\label{sec:one:step:analysis:claim:2}
\begin{proof}[Proof of \cref{one:step:analysis:claim:2}]
  By smoothness and convexity of $F_m$, we know
    \begin{align}
      F_m(\widehat{w_{r,k+1}}) & \leq F_m(w_{r,k}^m) + \left\langle \nabla F_m(w_{r,k}^m),  \widehat{w_{r,k+1}} - w_{r,k}^m \right\rangle + \frac{L}{2} \| \widehat{w_{r,k+1}} - w_{r,k}^m \|^2 
      \tag{smoothness}
      \\
      & \leq F_m(w) + \left\langle \nabla F_m(w_{r,k}^m),  \widehat{w_{r,k+1}} - w \right\rangle + \frac{L}{2} \| \widehat{w_{r,k+1}} - w_{r,k}^m \|^2.
      \tag{convexity}
  \end{align}
  Taking summation over $m$ gives
  \begin{align}
      &     F(\widehat{w_{r,k+1}}) - F(w)  = \frac{1}{M} \sum_{m=1}^M \left( F_m(\widehat{w_{r,k+1}}) - F_m(w) \right)  
      \\
      \leq & \left\langle \frac{1}{M} \sum_{m=1}^M \nabla F_m(w_{r,k}^m), \widehat{w_{r,k+1}} - w  \right\rangle  + \frac{L}{2M} \sum_{m=1}^M  \| \widehat{w_{r,k+1}} - w_{r,k}^m \|^2
      \\
      = & \left\langle \frac{1}{M} \sum_{m=1}^M \nabla f(w_{r,k}^m; \xi_{r,k}^m), \widehat{w_{r,k+1}} - w  \right\rangle  
      + \left\langle \frac{1}{M} \sum_{m=1}^M \left( \nabla F_m(w_{r,k}^m) - \nabla f(w_{r,k}^m; \xi_{r,k}^m) \right), \widehat{w_{r,k+1}} - w  \right\rangle  
      \\ 
      & \qquad  \qquad      + \frac{L}{2M} \sum_{m=1}^M  \| \widehat{w_{r,k+1}} - w_{r,k}^m \|^2 
      \\
      \leq & \left\langle \frac{1}{M} \sum_{m=1}^M \nabla f(w_{r,k}^m; \xi_{r,k}^m), \widehat{w_{r,k+1}} - w  \right\rangle  
      + \left\langle \frac{1}{M} \sum_{m=1}^M \left( \nabla F_m(w_{r,k}^m) - \nabla f(w_{r,k}^m; \xi_{r,k}^m) \right), \widehat{w_{r,k+1}} - w  \right\rangle  
      \\
          & \qquad \qquad + L \| \widehat{w_{r,k+1}} - \widehat{w_{r,k}} \|^2  +  \frac{L}{M} \sum_{m=1}^M  \| \widehat{w_{r,k}} - w_{r,k}^m \|^2,
      \label{eq:one_step_analysis:claim:2:1}
  \end{align}
  where in the last inequality we applied the triangle inequality (for an arbitrary norm $\|\cdot\|$):
  \begin{equation}
    \| \widehat{w_{r,k+1}} - w_{r,k}^m \|^2 
    \leq
    \left(  \| \widehat{w_{r,k+1}} - \widehat{w_{r,k}}\| + \|\widehat{w_{r,k}} - w_{r,k}^m \| \right)^2 
    \leq
    2 \| \widehat{w_{r,k+1}} - \widehat{w_{r,k}}\|^2 
    + 2\|\widehat{w_{r,k}} - w_{r,k}^m \|^2.
  \end{equation}

  Since $\psi$ is convex and $h$ is $1$-strongly-convex according to \cref{a1}, we know that $h_{r,k} = h + \eta_{\client} (rK + k) \psi$ is also $1$-strongly-convex. Therefore $h_{r,k}^*$ is $1$-smooth by \cref{strongly:convex:conjugate}.
  Consequently,
  \begin{equation}
       \left\| w_{r,k}^m - \widehat{w_{r,k}} \right\|^2 = 
       \left\| \nabla h_{r,k}^*(z_{r,k}^m) - \nabla h_{r,k}^*(\overline{z_{r,k}}) \right\|^2 
      \leq  \left\| z_{r,k}^m - \overline{z_{r,k}} \right\|_*^2,
      \label{eq:one_step_analysis:claim:2:2}
  \end{equation}
  where the first equality is by definition of $w_{r,k}^m$ and $\widehat{w_{r,k}}$ and the second inequality is by $1$-smoothness. Plugging \cref{eq:one_step_analysis:claim:2:2} back to \cref{eq:one_step_analysis:claim:2:1} completes the proof of \cref{one:step:analysis:claim:2}.
\end{proof}

\subsection{Stability of \feddualavg Under Bounded Gradient Assumptions: Proof of \cref{bounded_gradient}}
\label{sec:bounded_gradient}
The proof of \cref{bounded_gradient} is straightforward given the assumption of bounded gradient and the fact that $z_{r,0}^{m_1} = z_{r,0}^{m_2}$ for all $m_1, m_2 \in [M]$.
\begin{proof}[Proof of \cref{bounded_gradient}]
  Let $m_1, m_2 \in [M]$ be two arbitrary clients, then
\begin{align}
  \expt \left[  \|z_{r,k}^{m_1} - z_{r,k}^{m_2}\|_*^2  \middle| \mathcal{F}_{r,0} \right]
  & =
  \eta_{\client}^2 \expt \left[ \left\| \sum_{\kappa=0}^{k-1} \left( \nabla f(w_{r,\kappa}^{m_1}; \xi_{r,\kappa}^{m_1}) - \nabla f(w_{r,\kappa}^{m_2}; \xi_{r,\kappa}^{m_2})  \right) \right\|_*^2 \middle| \mathcal{F}_{r,0} \right]
  \tag{since $z_{r,0}^{m_1} = z_{r,0}^{m_2}$}
  \\
  & \leq
  \eta_{\client}^2 \expt \left[ \left( \sum_{\kappa={0}}^{k-1}  \left\|  \nabla f(w_{r,\kappa}^{m_1}; \xi_{r,\kappa}^{m_1})\right\|_* + \sum_{\kappa={0}}^{k-1}  \left\| \nabla f(w_{r,\kappa}^{m_2}; \xi_{r,\kappa}^{m_2}) \right\|_* \right)^2 
  \middle| \mathcal{F}_{r,0} \right]
  \tag{triangle inequality of $\|\cdot\|_*$} 
  \\
  & \leq
  \eta_{\client}^2 (2 (k-1) G)^2 = 4\eta_{\client}^2 (K-1)^2 G^2.
  \label{eq:bounded_gradient:1}
\end{align}
By convexity of $\|\cdot\|_*$,
\begin{equation}
  \frac{1}{M}\sum_{m=1}^M \expt \left\| z_{r,k}^m - \overline{z_{r,k}} \right\|_*^2 
  \leq \expt \left\| z_{r,k}^{m_1} - z_{r,k}^{m_2} \right\|_*^2
  \leq 4 \eta_{\client}^2 (K-1)^2 G^2,
\end{equation}
completing the proof of \cref{bounded_gradient}.
\end{proof}

%% file: analysis_quad.tex
% !TEX root = main.tex  
\section{Proof of Theorem \ref{thm:2:simplified}: Convergence of \feddualavg Under Bounded Heterogeneity and Quadratic Assumptions}
\label{sec:proof:thm:2}
In this section, we study the convergence of \feddualavg under \cref{a3} (quadraticness) with unit server learning rate $\eta_{\server} = 1$.
We provide a complete, non-asymptotic version of \cref{thm:2:simplified} with detailed proof, which expands the proof sketch in \cref{sec:proof_sketch}. 
We will reuse the notations ($\overline{z_{r,k}}$, $\widehat{w_{r,k}}$, etc.) introduced at the beginning of \cref{sec:proof:thm:1}.

\subsection{Main Theorem and Lemmas}
Now we state the full version of \cref{thm:2:simplified} on \feddualavg with unit server learning rate $\eta_{\server} = 1$ under quadratic assumptions.
\begin{theorem}[Detailed version of \cref{thm:2:simplified}]
  \label{thm:2}
  Assuming \cref{a3}, then for any initialization $w_0 \in \dom \psi$, for unit server learning rate $\eta_{\server} = 1$ and any client learning rate $\eta_{\client} \leq \frac{1}{4L}$, \feddualavg yields
  \begin{equation}
    \expt \left[  \Phi \left(  \frac{1}{KR} \sum_{r=0}^{R-1} \sum_{k=1}^K \widehat{w_{r,k}} \right) - \Phi(w^{\star}) \right] 
  \leq
  \frac{B}{\eta_{\client} KR }
  + \frac{\eta_{\client}  \sigma^2}{M}
  + 7 \eta_{\client}^2 L K \sigma^2
  + 14 \eta_{\client}^2 L K^2  \zeta^2,
  \end{equation}
  where $B := D_{h}(w^{\star}, w_0)$ is the Bregman divergence between the optimal $w^{\star}$ and the initialization $w_0$.

  Particularly for 
  \begin{equation}
    \eta_{\client} = \min \left\{ \frac{1}{4L}, 
    \frac{M^{\frac{1}{2}} B^{\frac{1}{2}}}{\sigma K^{\frac{1}{2}} R^{\frac{1}{2}}} ,
    \frac{B^{\frac{1}{3}}}{L^{\frac{1}{3}} K^{\frac{2}{3}} R^{\frac{1}{3}} \sigma^{\frac{2}{3}}}, 
    \frac{B^{\frac{1}{3}}}{L^{\frac{1}{3}} K R^{\frac{1}{3}} \zeta^{\frac{2}{3}}}  \right\},
  \end{equation}  
  we have
  \begin{equation}
    \expt \left[  \Phi \left(  \frac{1}{KR} \sum_{r=0}^{R-1} \sum_{k=1}^K \widehat{w_{r,k}} \right) - \Phi(w^{\star}) \right] 
  \leq
  \frac{4LB}{KR} 
    +
    \frac{2\sigma B^{\frac{1}{2}}}{M^{\frac{1}{2}} K^{\frac{1}{2}} R^{\frac{1}{2}}}
    +
    \frac{8L^{\frac{1}{3}} B^{\frac{2}{3}} \sigma^{\frac{2}{3}}}{K^{\frac{1}{3}} R^{\frac{2}{3}}}
    +
    \frac{15L^{\frac{1}{3}} B^{\frac{2}{3}} \zeta^{\frac{2}{3}}}{R^{\frac{2}{3}}}.
  \end{equation}
\end{theorem}

The proof of \cref{thm:2} relies on the perturbed iterate analysis \cref{pia:general} of \feddualavg and a stability bound for quadratic objectives, as stated below in \cref{quad:stability}. 
Note that \cref{pia:general} only assumes \cref{a1} and therefore applicable to \cref{thm:2}.
\begin{lemma}
  In the same settings of \cref{thm:2}, the following inequality holds for any $k \in \left\{ 0, 1, \ldots, K\right\}$ and $r \in \left\{ 0, 1, \ldots, R\right\}$,
  \label{quad:stability}
  \begin{equation}
    \frac{1}{M} \sum_{m=1}^M \expt \left[ \left\| \overline{z_{r,k}} - z_{r,k}^{m} \right\|^2_* \right] 
    \leq
    7 \eta_{\client}^2 K \sigma^2 
    + 14 \eta_{\client}^2 K^2  \zeta^2.
  \end{equation}
\end{lemma}
The proof of \cref{quad:stability} is deferred to \cref{sec:quad:stability}. With \cref{quad:stability} at hand we are ready to prove \cref{thm:2}.
\begin{proof}[Proof of \cref{thm:2}]
  Applying \cref{pia:general,quad:stability} one has
  \begin{align}
    & \expt \left[  \Phi \left(  \frac{1}{KR} \sum_{r=0}^{R-1} \sum_{k=1}^K \widehat{w_{r,k}} \right) - \Phi(w^{\star}) \right] 
    \\
    \leq
    & \frac{1}{\eta_{\client} KR } D_h(w^{\star}, w_0) 
    + \frac{\eta_{\client}  \sigma^2}{M}
    + \frac{L}{M KR }   \left[ \sum_{r=0}^{R-1}  \sum_{k=0}^{K-1} \sum_{m=1}^M   \expt\left \| \overline{z_{r,k}} - z_{r,k}^m \right\|_*^2 \right] 
    \tag{by \cref{pia:general}}
    \\
    \leq & 
    \frac{1}{\eta_{\client} KR } D_h(w^{\star}, w_0) 
    + \frac{\eta_{\client}  \sigma^2}{M}
    + L \cdot \left(  7 \eta_{\client}^2 K \sigma^2 
    + 14 \eta_{\client}^2 K^2  \zeta^2 \right)
    \tag{by \cref{quad:stability}}
    \\
    = &   \frac{B}{\eta_{\client} KR }
    + \frac{\eta_{\client}  \sigma^2}{M}
    + 7 \eta_{\client}^2 L K \sigma^2
    + 14 \eta_{\client}^2 L K^2  \zeta^2,
    \label{eq:proof:thm:2:1}
  \end{align}
  which gives the first inequality in \cref{thm:2}.

  Now set 
  \begin{equation}
    \eta_{\client} = \min \left\{ \frac{1}{4L}, 
    \frac{M^{\frac{1}{2}} B^{\frac{1}{2}}}{\sigma K^{\frac{1}{2}} R^{\frac{1}{2}}} ,
    \frac{B^{\frac{1}{3}}}{L^{\frac{1}{3}} K^{\frac{2}{3}} R^{\frac{1}{3}} \sigma^{\frac{2}{3}}}, 
    \frac{B^{\frac{1}{3}}}{L^{\frac{1}{3}} K R^{\frac{1}{3}} \zeta^{\frac{2}{3}}}  \right\}.
  \end{equation}  
  We have
  \begin{equation}
    \frac{B}{\eta_{\client} KR }
    \leq
    \max \left\{  \frac{4LB}{KR} 
    ,
    \frac{\sigma B^{\frac{1}{2}}}{M^{\frac{1}{2}} K^{\frac{1}{2}} R^{\frac{1}{2}}}
    ,
    \frac{L^{\frac{1}{3}} B^{\frac{2}{3}} \sigma^{\frac{2}{3}}}{K^{\frac{1}{3}} R^{\frac{2}{3}}}
    ,
    \frac{L^{\frac{1}{3}} B^{\frac{2}{3}} \zeta^{\frac{2}{3}}}{R^{\frac{2}{3}}}\right\},
  \end{equation}
  and 
  \begin{equation}
    \frac{\eta_{\client} \sigma^2}{M} \leq  \frac{\sigma B^{\frac{1}{2}}}{M^{\frac{1}{2}} K^{\frac{1}{2}} R^{\frac{1}{2}}}
    ,
    \qquad
    7 \eta_{\client}^2 L K \sigma^2 \leq  \frac{7L^{\frac{1}{3}} B^{\frac{2}{3}} \sigma^{\frac{2}{3}}}{K^{\frac{1}{3}} R^{\frac{2}{3}}}
    ,
    \qquad
    14 \eta_{\client}^2 L K^2  \zeta^2 \leq   \frac{14L^{\frac{1}{3}} B^{\frac{2}{3}} \zeta^{\frac{2}{3}}}{R^{\frac{2}{3}}}.
  \end{equation}
  Consequently
  \begin{equation}
    \frac{B}{\eta_{\client} KR }
    + \frac{\eta_{\client}  \sigma^2}{M}
    + 7 \eta_{\client}^2 L K \sigma^2
    + 14 \eta_{\client}^2 L K^2  \zeta^2
    \leq
    \frac{4LB}{KR} 
    +
    \frac{2\sigma B^{\frac{1}{2}}}{M^{\frac{1}{2}} K^{\frac{1}{2}} R^{\frac{1}{2}}}
    +
    \frac{8L^{\frac{1}{3}} B^{\frac{2}{3}} \sigma^{\frac{2}{3}}}{K^{\frac{1}{3}} R^{\frac{2}{3}}}
    +
    \frac{15L^{\frac{1}{3}} B^{\frac{2}{3}} \zeta^{\frac{2}{3}}}{R^{\frac{2}{3}}},
  \end{equation}
  completing the proof of \cref{thm:2}.
\end{proof}

\subsection{Stability of \feddualavg Under Quadratic Assumptions: Proof of \cref{quad:stability}}
\label{sec:quad:stability}
In this subsection, we prove \cref{quad:stability} on the stability of \feddualavg for quadratic $F$. We first state and prove the following \cref{quad:one_step_stab} on the one-step analysis of stability. 
\begin{proposition}
  In the same settings of \cref{thm:2}, let $m_1, m_2 \in [M]$ be two arbitrary clients. Then the following inequality holds
  \label{quad:one_step_stab}
  \begin{equation}
    \expt \left[ \left\|z_{r,k+1}^{m_1} - z_{r,k+1}^{m_2} \right\|_{Q^{-1}}^2  \middle| \mathcal{F}_{r,k} \right]
    \leq
    \left( 1 + \frac{1}{K} \right)  
    \left\| z_{r,k}^{m_1} - z_{r,k}^{m_2} \right\|_{Q^{-1}}^2 
    + 2 \left( 1 + \frac{1}{K} \right)   \eta_{\client}^2 \sigma^2 \|Q\|_2^{-1}
    + 4 ( 1+ K) \eta_{\client}^2 \zeta^2 \|Q\|_2^{-1}.
  \end{equation}
\end{proposition}

The proof of \cref{quad:one_step_stab} relies on the following three claims.
To simplify the exposition we introduce two more notations for this subsection. For any $r, k, m$, let
\begin{equation}
  \varepsilon_{r,k}^{m} :=  \nabla f(w_{r,k}^m;\xi_{r,k}^{m}) - \nabla F_m(w_{r,k}^m),
  \qquad
  \delta_{r,k}^m := \nabla F_m(w_{r,k}^m) - \nabla  F(w_{r,k}^m).
\end{equation}

The following claim upper bounds the growth of $\left\|z_{r,k+1}^{m_1} - z_{r,k+1}^{m_2} \right\|_{Q^{-1}}^2$. The proof of \cref{quad:one_step_stab:claim:1} is deferred to \cref{sec:proof:quad:one_step_stab:claim:1}.
\begin{claim}
  \label{quad:one_step_stab:claim:1}
  In the same settings of \cref{quad:one_step_stab}, the following inequality holds
  \begin{align}
    \left\|z_{r,k+1}^{m_1} - z_{r,k+1}^{m_2} \right\|_{Q^{-1}}^2 
    \leq
    & \left( 1 + \frac{1}{K} \right) 
    \left\| z_{r,k}^{m_1} - z_{r,k}^{m_2} - \eta_{\client} \cdot Q \left( w_{r,k}^{m_1} - w_{r,k}^{m_2} \right) - \eta_{\client} \left( \varepsilon_{r,k}^{m_1} - \varepsilon_{r,k}^{m_2} \right) \right\|_{Q^{-1}}^2
    \\
    & + 
    \left( 1 + K \right) \eta_{\client}^2 \left\| \delta_{r,k}^{m_1} - \delta_{r,k}^{m_2}  \right\|_{Q^{-1}}^2.
    \label{eq:quad:one_step_stab:claim:1}
  \end{align}
\end{claim}

The next claim upper bounds the growth of the first term in \cref{eq:quad:one_step_stab:claim:1} in conditional expectation.
We extend the stability technique in \citep{Flammarion.Bach-COLT17} to bound this term. 
The proof of \cref{quad:one_step_stab:claim:2} is deferred to \cref{sec:proof:quad:one_step_stab:claim:2}.
\begin{claim}
  \label{quad:one_step_stab:claim:2}
  In the same settings of \cref{quad:one_step_stab}, the following inequality holds
  \begin{equation}
    \expt \left[ \left\| z_{r,k}^{m_1} - z_{r,k}^{m_2} - \eta_{\client} Q \left( w_{r,k}^{m_1} - w_{r,k}^{m_2} \right) - \eta_{\client} \left( \varepsilon_{r,k}^{m_1} - \varepsilon_{r,k}^{m_2} \right) \right\|_{Q^{-1}}^2 \middle| \mathcal{F}_{r,k} \right]
    \leq
    \left\| z_{r,k}^{m_1} - z_{r,k}^{m_2} \right\|_{Q^{-1}}^2 + 2 \eta_{\client}^2 \sigma^2 \|Q\|_2^{-1}.
  \end{equation}
\end{claim}

The third claim upper bounds the growth of the second term in \cref{eq:quad:one_step_stab:claim:1} under conditional expectation. This is a result of the bounded heterogeneity assumption (\cref{a3}(c)). 
The proof of \cref{quad:one_step_stab:claim:3} is deferred to \cref{sec:proof:quad:one_step_stab:claim:3}.
\begin{claim}
  \label{quad:one_step_stab:claim:3}
  In the same settings of \cref{quad:one_step_stab}, the following inequality holds
  \begin{equation}
    \expt \left[ \left\| \delta_{r,k}^{m_1} - \delta_{r,k}^{m_2}  \right\|_{Q^{-1}}^2  \middle| \mathcal{F}_{r,k} \right]
    \leq
    4 \|Q\|_2^{-1} \zeta^2.
  \end{equation}
\end{claim}
The proof of the above claims as well as the main lemma require the following helper claim which we also state here. The proof is also deferred to \cref{sec:proof:quad:one_step_stab:claim:3}.
\begin{claim}
  \label{quad:one_step_stab:claim:4}
  In the same settings of \cref{quad:one_step_stab}, the dual norm $\|\cdot\|_*$ corresponds to the $\|Q\|_2 \cdot {Q^{-1}}$-norm, namely $    \|z\|_* = \sqrt{ \|Q\|_2 \cdot z^{\top} Q^{-1} z}$. 
\end{claim}

The proof of \cref{quad:one_step_stab} is immediate once we have \cref{quad:one_step_stab:claim:1,quad:one_step_stab:claim:2,quad:one_step_stab:claim:3}.
\begin{proof}[Proof of \cref{quad:one_step_stab}]
  By  \cref{quad:one_step_stab:claim:1,quad:one_step_stab:claim:2,quad:one_step_stab:claim:3},
  \begin{align}
    & \expt \left[ \left\|z_{r,k+1}^{m_1} - z_{r,k+1}^{m_2} \right\|_{Q^{-1}}^2 \middle| \mathcal{F}_{r,k} \right]
    \\
    \leq
    & \left( 1 + \frac{1}{K} \right) 
    \expt \left[ \left\| z_{r,k}^{m_1} - z_{r,k}^{m_2} - \eta_{\client} Q \left( w_{r,k}^{m_1} - w_{r,k}^{m_2} \right) - \eta_{\client} \left( \varepsilon_{r,k}^{m_1} - \varepsilon_{r,k}^{m_2} \right) \right\|_{Q^{-1}}^2 \middle| \mathcal{F}_{r,k} \right]
    \\
    & \qquad
    + 
    \left( 1 + K \right) \eta_{\client}^2 \expt \left[ \left\| \delta_{r,k}^{m_1} - \delta_{r,k}^{m_2}  \right\|_{Q^{-1}}^2 \middle| \mathcal{F}_{r,k} \right]
    \tag{by \cref{quad:one_step_stab:claim:1}}
    \\
    \leq &  \left( 1 + \frac{1}{K} \right)  
    \left\| z_{r,k}^{m_1} - z_{r,k}^{m_2} \right\|_{Q^{-1}}^2 
    + 2 \left( 1 + \frac{1}{K} \right)   \eta^2 \sigma^2 \|Q\|_2^{-1}
    + 4 ( 1+ K) \eta_{\client}^2 \zeta^2 \|Q\|_2^{-1},
    \tag{by \cref{quad:one_step_stab:claim:2,quad:one_step_stab:claim:3}}
  \end{align}
  completing the proof of \cref{quad:one_step_stab}.
\end{proof}

The main \cref{quad:stability} then follows by telescoping \cref{quad:one_step_stab}.
\begin{proof}[Proof of \cref{quad:stability}]
Let $m_1, m_2$ be two arbitrary clients. 
Telescoping  \cref{quad:one_step_stab} from $\mathcal{F}_{r,0}$ to $\mathcal{F}_{r,k}$ gives
\begin{align}
         & \expt \left[ \left\| z_{r,k}^{m_1} - z_{r,k}^{m_2} \right\|_{Q^{-1}}^2 \right]
    \\
    \leq & 
    \frac{\left( 1 + \frac{1}{K} \right)^k - 1}{\frac{1}{K}}
     \left( 2 \left( 1 + \frac{1}{K} \right)   \eta_{\client}^2 \sigma^2 \|Q\|_2^{-1}
    + 4 ( 1+ K) \eta_{\client}^2 \zeta^2 \|Q\|_2^{-1} \right) 
    \tag{telescoping of \cref{quad:one_step_stab}}
    \\
    \leq & 
    (\euler - 1) K
     \left( 2 \left( 1 + \frac{1}{K} \right)   \eta_{\client}^2 \sigma^2 \|Q\|_2^{-1}
    + 4 ( 1+ K) \eta_{\client}^2 \zeta^2 \|Q\|_2^{-1} \right) 
    \tag{since $(1 + \frac{1}{K})^k \leq (1 + \frac{1}{K})^K < \euler$}
    \\
    \leq & (\euler - 1) K \left( 4 \eta_{\client}^2 \sigma^2 \|Q\|_2^{-1} + 8K \eta_{\client}^2 \zeta^2 \|Q\|_2^{-1} \right)
    \tag{since $1 + \frac{1}{K} \leq 2$ and $1 + K \leq 2K$}
    \\
    \leq & 7 \eta_{\client}^2  K \sigma^2 \|Q\|_2^{-1} + 14 \eta_{\client}^2 K^2  \zeta^2 \|Q\|_2^{-1}
    \tag{since $4 (\euler - 1) < 7$ and $8 (\euler - 1) < 14$ }
\end{align}

By convexity of $\|\cdot\|^2_{Q^{-1}}$ and \cref{quad:one_step_stab} one has
\begin{equation}
  \frac{1}{M} \sum_{m=1}^M \expt \left[ \left\| \overline{z_{r,k}} - z_{r,k}^{m} \right\|_{Q^{-1}}^2 \right] 
  \leq 
  \expt \left[ \left\| z_{r,k}^{m_1} - z_{r,k}^{m_2} \right\|_{Q^{-1}}^2 \right]
  \leq 
  7 \eta_{\client}^2 K \sigma^2 \|Q\|_2^{-1} + 14 \eta_{\client}^2 K^2  \zeta^2 \|Q\|_2^{-1}.
\end{equation}
Finally, we switch back to $\|\cdot\|_*$ norm following \cref{quad:one_step_stab:claim:4} 
\begin{equation}
  \frac{1}{M} \sum_{m=1}^M \expt \left[ \left\| \overline{z_{r,k}} - z_{r,k}^{m} \right\|_*^2 \right] 
  \leq
  7 \eta_{\client}^2 K \sigma^2  + 14 \eta_{\client}^2 K^2  \zeta^2,
\end{equation}
completing the proof of \cref{quad:stability}.
\end{proof}

\subsubsection{Deferred Proof of  \cref{quad:one_step_stab:claim:1}}
\label{sec:proof:quad:one_step_stab:claim:1}
\begin{proof}[Proof of \cref{quad:one_step_stab:claim:1}]
  By definition of \feddualavg procedure one has
  \begin{align}
    z_{r,k+1}^{m} & = z_{r,k}^m - \eta_{\client} \nabla f(w_{r,k}^m; \xi_{r,k}^m) 
    \\
    & = z_{r,k}^m - \eta_{\client} \nabla F(w_{r,k}^m) 
    + 
    \eta_{\client} \left( \nabla F_m(w_{r,k}^m) - \nabla F(w_{r,k}^m) \right) 
    +
    \eta_{\client} \left( \nabla f(w_{r,k}^m;\xi_{r,k}^{m}) - \nabla F_m(w_{r,k}^m)  \right) 
    \\
    & =   z_{r,k}^m - \eta_{\client} \nabla F(w_{r,k}^m) - \eta_{\client} \varepsilon_{r,k}^m - \eta_{\client} \delta_{r,k}^m,
    \label{eq:quad:one_step_stab:1}
  \end{align}
  where the last equality is by definition of $\varepsilon_{r,k}^m$ and $\delta_{r,k}^m$.
  Therefore
  \begin{align}
    &  \left\|z_{r,k+1}^{m_1} - z_{r,k+1}^{m_2} \right\|_{Q^{-1}}^2 
    \\
    = & \left\| z_{r,k}^{m_1} - z_{r,k}^{m_2} - \eta_{\client} Q \left( w_{r,k}^{m_1} - w_{r,k}^{m_2} \right) - \eta_{\client} \left( \varepsilon_{r,k}^{m_1} - \varepsilon_{r,k}^{m_2} \right) 
    - \eta_{\client} \left( \delta_{r,k}^{m_1} - \delta_{r,k}^{m_2} \right) 
    \right\|_{Q^{-1}}^2 
    \tag{by \cref{eq:quad:one_step_stab:1}}
    \\
    = & \left\| z_{r,k}^{m_1} - z_{r,k}^{m_2} - \eta_{\client} Q \left( w_{r,k}^{m_1} - w_{r,k}^{m_2} \right) - \eta_{\client} \left( \varepsilon_{r,k}^{m_1} - \varepsilon_{r,k}^{m_2} \right) \right\|_{Q^{-1}}^2 
    + \eta_{\client}^2 \left\| \delta_{r,k}^{m_1} - \delta_{r,k}^{m_2}  \right\|_{Q^{-1}}^2 
    \\
    & \quad
    + 
    2 \left\langle z_{r,k}^{m_1} - z_{r,k}^{m_2} - \eta_{\client} Q \left( w_{r,k}^{m_1} - w_{r,k}^{m_2} \right) - \eta_{\client} \left( \varepsilon_{r,k}^{m_1} - \varepsilon_{r,k}^{m_2}  \right), \eta_{\client} Q^{-1} \left(  \delta_{r,k}^{m_1} - \delta_{r,k}^{m_2}  \right) \right\rangle.
    % \tag{expansion of $\|\cdot\|_{Q^{-1}}^2$}
    \label{eq:quad:one_step_stab:2}
  \end{align}
  By Cauchy-Schwartz inequality and AM-GM inequality one has (for any $\gamma > 0$)
  \begin{align}
    & \left\langle z_{r,k}^{m_1} - z_{r,k}^{m_2} - \eta_{\client} Q \left( w_{r,k}^{m_1} - w_{r,k}^{m_2} \right) - \eta_{\client} \left( \varepsilon_{r,k}^{m_1} - \varepsilon_{r,k}^{m_2}  \right), \eta_{\client} Q^{-1} \left(  \delta_{r,k}^{m_1} - \delta_{r,k}^{m_2}  \right) \right\rangle 
    \\
    \leq & \left\| z_{r,k}^{m_1} - z_{r,k}^{m_2} - \eta_{\client} Q \left( w_{r,k}^{m_1} - w_{r,k}^{m_2} \right) - \eta_{\client} \left( \varepsilon_{r,k}^{m_1} - \varepsilon_{r,k}^{m_2} \right) \right\|_{Q^{-1}}
    \left\| \eta_{\client}  \left( \delta_{r,k}^{m_1} - \delta_{r,k}^{m_2} \right)  \right\|_{Q^{-1}}
    \tag{Cauchy-Schwarz inequality}
    \\
    \leq & \frac{1}{2\gamma} \left\| z_{r,k}^{m_1} - z_{r,k}^{m_2} - \eta_{\client} Q \left( w_{r,k}^{m_1} - w_{r,k}^{m_2} \right) - \eta_{\client} \left( \varepsilon_{r,k}^{m_1} - \varepsilon_{r,k}^{m_2} \right) \right\|^2_{Q^{-1}} 
    +
    \frac{1}{2}  \gamma  \left\| \eta_{\client}  \left( \delta_{r,k}^{m_1} - \delta_{r,k}^{m_2} \right)  \right\|_{Q^{-1}}^2.
    \tag{AM-GM inequality}
    \\
    \leq & \frac{1}{2\gamma} \left\| z_{r,k}^{m_1} - z_{r,k}^{m_2} - \eta_{\client} Q \left( w_{r,k}^{m_1} - w_{r,k}^{m_2} \right) - \eta_{\client} \left( \varepsilon_{r,k}^{m_1} - \varepsilon_{r,k}^{m_2} \right) \right\|^2_{Q^{-1}} 
    +
    \frac{1}{2}  \gamma \eta_{\client}^2 \left\|  \left( \delta_{r,k}^{m_1} - \delta_{r,k}^{m_2} \right)  \right\|_{Q^{-1}}^2.
    \label{eq:quad:one_step_stab:3}
  \end{align}
  Plugging \cref{eq:quad:one_step_stab:3}
     to \cref{eq:quad:one_step_stab:2} with $\gamma = K$ gives
  \begin{align}
    &  \left\|z_{r,k+1}^{m_1} - z_{r,k+1}^{m_2} \right\|_{Q^{-1}}^2 
    \\
    \leq & \left( 1 + \frac{1}{K} \right) 
    {\left\| z_{r,k}^{m_1} - z_{r,k}^{m_2} - \eta_{\client} Q \left( w_{r,k}^{m_1} - w_{r,k}^{m_2} \right) - \eta_{\client} \left( \varepsilon_{r,k}^{m_1} - \varepsilon_{r,k}^{m_2} \right) \right\|_{Q^{-1}}^2}
    + \left( 1 + K \right) \eta_{\client}^2 
    {\left\| \delta_{r,k}^{m_1} - \delta_{r,k}^{m_2}  \right\|_{Q^{-1}}^2},
  \end{align}
  completing the proof of \cref{quad:one_step_stab:claim:1}.
\end{proof}

\subsubsection{Deferred Proof of  \cref{quad:one_step_stab:claim:2}}
\label{sec:proof:quad:one_step_stab:claim:2}
The proof technique of this claim is similar to \citep[Lemma 8]{Flammarion.Bach-COLT17} which we adapt to fit into our settings.
\begin{proof}[Proof of \cref{quad:one_step_stab:claim:2}]
  Let us first expand the $\|\cdot\|_{Q^{-1}}^2$:
\begin{align}
   &  \left\| z_{r,k}^{m_1} - z_{r,k}^{m_2} - \eta_{\client} Q \left( w_{r,k}^{m_1} - w_{r,k}^{m_2} \right) - \eta_{\client} \left( \varepsilon_{r,k}^{m_1} - \varepsilon_{r,k}^{m_2} \right) \right\|_{Q^{-1}}^2  
   \\
  = & \left\| z_{r,k}^{m_1} - z_{r,k}^{m_2} \right\|_{Q^{-1}}^2 
  + \left\| \eta Q \left( w_{r,k}^{m_1} - w_{r,k}^{m_2} \right)  \right\|_{Q^{-1}}^2
  + \left\| \eta  \left( \varepsilon_{r,k}^{m_1} - \varepsilon_{r,k}^{m_2} \right)  \right\|_{Q^{-1}}^2 
  + 2 \left\langle \eta \left( w_{r,k}^{m_1} - w_{r,k}^{m_2} \right),   \eta \left( \varepsilon_{r,k}^{m_1} - \varepsilon_{r,k}^{m_2} \right)  \right\rangle
  \\
  & \qquad 
  + 2 \left\langle  z_{r,k}^{m_1} - z_{r,k}^{m_2}, - \eta \left( w_{r,k}^{m_1} - w_{r,k}^{m_2} \right)   \right\rangle 
  + 2 \left\langle z_{r,k}^{m_1} - z_{r,k}^{m_2}, - \eta Q^{-1} \left( \varepsilon_{r,k}^{m_1} - \varepsilon_{r,k}^{m_2} \right)  \right\rangle.
\end{align}
Now we take conditional expectation. 
Note that by bounded variance assumption one has
\begin{equation}
  \expt \left[ \left\| \eta_{\client}  \left( \varepsilon_{r,k}^{m_1} - \varepsilon_{r,k}^{m_2} \right)  \right\|_{Q^{-1}}^2 \middle| \mathcal{F}_{r,k} \right] = 
  \|Q\|_2^{-1} \cdot \expt \left[ \left\| \eta_{\client}  \left( \varepsilon_{r,k}^{m_1} - \varepsilon_{r,k}^{m_2} \right)  \right\|_*^2 \middle| \mathcal{F}_{r,k} \right]
  \leq 
  2 \eta_{\client}^2 \sigma^2  \|Q\|_2^{-1},
\end{equation}
where in the first equality we applied \cref{quad:one_step_stab:claim:4}.

By unbiased and independence assumptions
\begin{equation}
  \expt \left[ \varepsilon_{r,k}^{m_1} - \varepsilon_{r,k}^{m_2} \middle| \mathcal{F}_{r,k} \right] = 0.
\end{equation}
Thus
\begin{align}
  & \expt \left[ \left\| z_{r,k}^{m_1} - z_{r,k}^{m_2} - \eta_{\client} Q \left( w_{r,k}^{m_1} - w_{r,k}^{m_2} \right) - \eta_{\client} \left( \varepsilon_{r,k}^{m_1} - \varepsilon_{r,k}^{m_2} \right) \right\|_{Q^{-1}}^2  \middle | \mathcal{F}_{r,k} \right]
  \\
  \leq & 
  \left\| z_{r,k}^{m_1} - z_{r,k}^{m_2} \right\|_{Q^{-1}}^2 
  \underbrace{+  \eta_{\client}^2 \left\|  Q \left( w_{r,k}^{m_1} - w_{r,k}^{m_2} \right)  \right\|_{Q^{-1}}^2}_{\text{(I)}}
  \underbrace{- 2  \eta_{\client} \left\langle  z_{r,k}^{m_1} - z_{r,k}^{m_2},  w_{r,k}^{m_1} - w_{r,k}^{m_2} \right\rangle}_{\text{(II)}} 
  + 2 \eta_{\client}^2 \sigma^2 \|Q\|_2^{-1}.
  \label{eq:quad:1}
\end{align}

Now we analyze (I), (II) in \cref{eq:quad:1}. First note that
\begin{align}
  & 
  \text{(I)} = \eta_{\client}^2 \left\|  Q \left( w_{r,k}^{m_1} - w_{r,k}^{m_2} \right)  \right\|_{Q^{-1}}^2
  \\
  = & \eta_{\client}^2 \left\langle w_{r,k}^{m_1} - w_{r,k}^{m_2} , Q \left( w_{r,k}^{m_1} - w_{r,k}^{m_2} \right)  \right\rangle 
  \tag{by definition of $\|\cdot\|_{Q^{-1}}^2$}
  \\
  = & \eta_{\client} \left\langle w_{r,k}^{m_1} - w_{r,k}^{m_2} , \eta_{\client} \left( \nabla F(w_{r,k}^{m_1}) - \nabla F(w_{r,k}^{m_2}) \right) \right\rangle
  \tag{since $F$ is quadratic}
  \\
  = & \eta_{\client} \left\langle w_{r,k}^{m_1} - w_{r,k}^{m_2} ,  \nabla (\eta_{\client} F - 2 h) (w_{r,k}^{m_1}) - \nabla (\eta_{\client} F - 2h)(w_{r,k}^{m_2})  \right\rangle
  + 2 \eta_{\client} \left\langle w_{r,k}^{m_1} - w_{r,k}^{m_2} ,  \nabla h (w_{r,k}^{m_1}) - \nabla h(w_{r,k}^{m_2})  \right\rangle
\end{align}
By $L$-smoothness of $F_m$ (\cref{a1}(c)) we know that $F := \frac{1}{M} \sum_{m=1}^M F_m$ is also $L$-smooth. Thus $\eta_{\client} F$ is $\frac{1}{4}$-smooth since $\eta_{\client} \leq \frac{1}{4L}$. Thus $\eta_{\client} F - 2h$ is concave since $h$ is $1$-strongly convex, which implies
\begin{equation}
   \left\langle w_{r,k}^{m_1} - w_{r,k}^{m_2} ,  \nabla (\eta_{\client} F - 2 h) (w_{r,k}^{m_1}) - \nabla (\eta_{\client} F - 2h)(w_{r,k}^{m_2})  \right\rangle \leq 0.
\end{equation}
We obtain
\begin{equation}
  \text{(I)} \leq 2 \eta_{\client} \left\langle w_{r,k}^{m_1} - w_{r,k}^{m_2} ,  \nabla h (w_{r,k}^{m_1}) - \nabla h(w_{r,k}^{m_2})  \right\rangle.
  \label{eq:quad:2}
\end{equation}

Now we study (I)+(II) in \cref{eq:quad:1}:
\begin{align}
  & \text{(I) + (II)} =  
   \eta_{\client}^2 \left\|  Q \left( w_{r,k}^{m_1} - w_{r,k}^{m_2} \right)  \right\|_{Q^{-1}}^2
   - 2  \eta_{\client} \left\langle  z_{r,k}^{m_1} - z_{r,k}^{m_2},  w_{r,k}^{m_1} - w_{r,k}^{m_2} \right\rangle 
  \\
  \leq & 
  2 \eta_{\client} \left\langle  w_{r,k}^{m_1} - w_{r,k}^{m_2}  ,  \nabla h (w_{r,k}^{m_1}) - \nabla h(w_{r,k}^{m_2})   \right\rangle
  - 2  \eta_{\client} \left\langle  w_{r,k}^{m_1} - w_{r,k}^{m_2} , z_{r,k}^{m_1} - z_{r,k}^{m_2} \right\rangle 
  \tag{by inequality \cref{eq:quad:2}}
  \\
  = &   - 2  \eta_{\client} \left\langle  w_{r,k}^{m_1} - w_{r,k}^{m_2} , \left(z_{r,k}^{m_1} - \nabla h(w_{r,k}^{m_1}) \right) - \left(z_{r,k}^{m_2} - \nabla h(w_{r,k}^{m_2}) \right)  \right\rangle 
  \label{eq:quad:3}
\end{align}

On the other hand, by definition of $w_{r,k}^m$ we have
\begin{equation}
  w_{r,k}^{m} = \nabla (h + (rK + k)\eta_{\client} \psi )^* (z_{r,k}^m)
  =
  \argmin_{w} \left\{ \left\langle -z_{r,k}^m, w  \right\rangle + (rK+k)\eta_{\client} \psi (w) + h(w)   \right\}.
\end{equation}
By subdifferential calculus one has
\begin{equation}
  z_{r,k}^m - \nabla h(w_{r,k}^m)  \in \partial \left[ \eta_c (rK+k) \psi(w_{r,k}^m) \right].
\end{equation}
By monotonicity of subgradients one has
\begin{equation}
  \left\langle  w_{r,k}^{m_1} - w_{r,k}^{m_2} , \left(z_{r,k}^{m_1} - \nabla h(w_{r,k}^{m_1}) \right) - \left(z_{r,k}^{m_2} - \nabla h(w_{r,k}^{m_2}) \right)  \right\rangle \geq 0.
  \label{eq:quad:4}
\end{equation}
Combining \cref{eq:quad:3,eq:quad:4} gives
\begin{equation}
  \text{(I) + (II)} \leq 0.
  \label{eq:quad:5}
\end{equation}
Combining \cref{eq:quad:1,eq:quad:5} completes the proof as 
\begin{equation}
  \expt \left[ \left\| z_{r,k}^{m_1} - z_{r,k}^{m_2} - \eta_{\client} Q \left( w_{r,k}^{m_1} - w_{r,k}^{m_2} \right) - \eta_{\client} \left( \varepsilon_{r,k}^{m_1} - \varepsilon_{r,k}^{m_2} \right) \right\|_{Q^{-1}}^2  \middle | \mathcal{F}_{r,k} \right]
  \leq
  \left\| z_{r,k}^{m_1} - z_{r,k}^{m_2} \right\|_{Q^{-1}}^2 
  + 2 \eta_{\client}^2 \sigma^2  \|Q\|_2^{-1}.
\end{equation}
\end{proof}

\subsubsection{Deferred Proof of  \cref{quad:one_step_stab:claim:3,quad:one_step_stab:claim:4}}
\label{sec:proof:quad:one_step_stab:claim:3}
\begin{proof}[Proof of \cref{quad:one_step_stab:claim:3}]
By triangle inequality and AM-GM inequality,
\begin{align}
  & \expt \left[  \left\|\delta_{r,k}^{m_1} - \delta_{r,k}^{m_2} \right\|_{Q^{-1}}^2 | \mathcal{F}_{r,k} \right]  
  \\
  \leq &
  \expt \left[ \left( \| \delta_{r,k}^{m_1} \|_{Q^{-1}} + \| \delta_{r,k}^{m_2} \|_{Q^{-1}} \right)^2  \middle| \mathcal{F}_{r,k} \right]\tag{triangle inequality} 
  \\
  \leq & 2 \expt \left[ \| \delta_{r,k}^{m_1} \|_{Q^{-1}}^2 + \| \delta_{r,k}^{m_2} \|_{Q^{-1}}^2  \middle| \mathcal{F}_{r,k} \right].
  \tag{AM-GM inequality}
\end{align}

By \cref{quad:one_step_stab:claim:4}, 
\begin{equation}
  \expt \left[  \left\|\delta_{r,k}^{m_1} - \delta_{r,k}^{m_2} \right\|_{Q^{-1}}^2 \middle| \mathcal{F}_{r,k} \right]  
  \leq
  2 \|Q\|_2^{-1} \expt \left[ \| \delta_{r,k}^{m_1} \|_*^2 + \| \delta_{r,k}^{m_2} \|_*^2  \middle| \mathcal{F}_{r,k} \right]
  \leq
  4 \|Q\|_2^{-1} \zeta^2,
\end{equation}
where the last inequality is due to bounded heterogeneity  \cref{a3}(c). This completes the proof of \cref{quad:one_step_stab:claim:3}.
\end{proof}

\begin{proof}[Proof of \cref{quad:one_step_stab:claim:4}]
  Since the primal norm $\|\cdot\|$ is $(\|Q\|_2^{-1} \cdot Q)$-norm by \cref{a3}(b), the dual norm $\|\cdot\|_*$ is $ \left( \|Q\|_2^{-1} \cdot Q \right)^{-1} = \|Q\|_2 \cdot Q^{-1}$-norm.
\end{proof}

%% file: analysis_small_lr.tex
% !TEX root = main.tex  
\section{Proof of \cref{thm:0}}
\label{sec:small_lr}
In this section, we state and prove \cref{thm:0} on the convergence of \feddualavg for small client learning rate $\eta_{\client}$. 
The intuition is that for sufficiently small client learning rate, \feddualavg is almost as good as stochastic mini-batch with $R$ iterations and batch-size $MK$.
The proof technique is very similar to the above sections and \citep{Karimireddy.Kale.ea-ICML20} so we skip a substantial amount of the proof details. 
We present the proof for \feddualavg only since the analysis of \fedmid is very similar.

To facilitate the analysis we re-parameterize the hyperparameters by letting $\eta := \eta_\server \eta_{\client}$, and we treat $(\eta, \eta_{\client})$ as independent hyperparameters (rather than $(\eta_{\client}, \eta_{\server})$).
We use the notation $h_{r,k} := h + \tilde{\eta}_{r,k} \cdot \psi  = h + (\eta r K + \eta_{\client} k) \psi$, $\overline{z_{r,k}} := \frac{1}{M} \sum_{m=1}^M z_{r,k}^m$, and $\widehat{w_{r,k}} := \nabla h_{r,k}^* (\overline{z_{r,k}})$. 
Note that $\widehat{w_{r,0}} = w_{r,0}^m$ for all $m \in [M]$ by definition.
\subsection{Main Theorem and Lemmas}
Now we state the full version of \cref{thm:0} on \feddualavg with small client learning rate $\eta_{\client}$.
\begin{theorem}[Detailed version of \cref{thm:0}]
  \label{small_lr}
  Assuming \cref{a1}, then for any $\eta \in (0, \frac{1}{4KL}]$, for any initialization $w_0 \in \dom \psi$,
  there exists an $\eta_{\client}^{\max} > 0$ (which may depend on $\eta$ and $w_0$) such that for any $\eta_{\client} \in (0, \eta_{\client}^{\max}]$, \feddualavg yields
  \begin{equation} 
    \expt \left[ \Phi\left( \frac{1}{R}  \sum_{r=1}^{R}  \widehat{w_{r,0}} \right) - \Phi(w^{\star}) \right]
    \leq
    \frac{B}{\eta KR}
    +
    \frac{3 \eta \sigma^2}{M},
  \end{equation}
  where $B := D_h(w^{\star}, w_0)$ is the Bregman divergence between the optimal $w^{\star}$ and the initialization $w_0$.

  In particular for 
  \begin{equation}
    \eta = \min \left\{ \frac{1}{4KL}, \frac{B^{\frac{1}{2}} M^{\frac{1}{2}}}{K^{\frac{1}{2}} R^{\frac{1}{2}} \sigma} \right\},
  \end{equation}
  one has
  \begin{equation}
    \expt \left[ \Phi\left( \frac{1}{R}  \sum_{r=1}^{R}  \widehat{w_{r,0}} \right) - \Phi(w^{\star}) \right]
     \leq
    \frac{4 L B}{R}
    +
    \frac{4 \sigma B^{\frac{1}{2}}}{ M^{\frac{1}{2}} K^{\frac{1}{2}} R^{\frac{1}{2}}}.
  \end{equation}
\end{theorem}

The proof of \cref{small_lr} relies on the following lemmas. 

The first \cref{small_lr:1} analyzes $\tilde{D}_{h_{r+1,0}} (w, \overline{z_{r+1,0}})$. The proof of \cref{small_lr:1} is deferred to \cref{sec:proof:small_lr:1}.
\begin{lemma}
  Under the same settings of \cref{small_lr}, the following inequality holds.
  \label{small_lr:1}
\begin{align}
  & \tilde{D}_{h_{r+1,0}} (w, \overline{z_{r+1,0}}) -  \tilde{D}_{h_{r,0}} (w, \overline{z_{r,0}}) 
  \\
  \leq & - \tilde{D}_{h_{r,0}} ( \widehat{w_{r+1,0}}, \overline{z_{r,0}})  - \eta K \left( \Phi( \widehat{w_{r+1,0}}) - \Phi(w)  \right) + \frac{L}{2} \eta K  \|\widehat{w_{r+1,0}}-\widehat{w_{r,0}}\|^2 
  \\
  & + \eta  K 
  \left\langle  \nabla F (\widehat{w_{r,0}}) - \frac{1}{MK} \sum_{m=1}^M \sum_{k=0}^{K-1} \nabla f(w_{r,k}^m; \xi_{r,k}^m), \widehat{w_{r+1,0}}-w  \right\rangle
\end{align}
\end{lemma}

The second lemma analyzes $\tilde{D}_{h_{r+1,0}} (w, \overline{z_{r+1,0}})$ under conditional expectation. The proof of \cref{small_lr:2} is deferred to \cref{sec:proof:small_lr:2}.
\begin{lemma}
  \label{small_lr:2}
  Under the same settings of \cref{small_lr},   there exists an $\eta_{\client}^{\max} > 0$ (which may depend on $\eta$ and $w_0$) such that for any $\eta_{\client} \in (0, \eta_{\client}^{\max}]$, \feddualavg yields
  \begin{align}
    \expt \left[ \tilde{D}_{h_{r+1,0}} (w, \overline{z_{r+1,0}}) \middle | \mathcal{F}_{r,0} \right] -  \tilde{D}_{h_{r,0}} (w, \overline{z_{r,0}}) 
    \leq 
    - \eta K \expt \left[ \left( \Phi( \widehat{w_{r+1,0}}) - \Phi(w)  \right) \middle| \mathcal{F}_{r,0} \right]
    + \frac{3 \eta^2 K \sigma^2}{M}.
  \end{align}
\end{lemma}

With \cref{small_lr:1,small_lr:2} at hand we are ready to prove \cref{small_lr}.
\begin{proof}[Proof of \cref{small_lr}]
  Telescoping \cref{small_lr:2} and dropping the negative terms gives
\begin{equation} 
  \frac{1}{R} \sum_{r=1}^{R} \expt \left[ \Phi( \widehat{w_{r,0}}) - \Phi(w) \right]
  \leq
  \frac{1}{\eta KR}\tilde{D}_{h_{r,0}} (w, \overline{z_{r,0}})
  +
  \frac{3 \eta \sigma^2}{M}
  =
  \frac{B}{\eta KR}
  +
  \frac{3 \eta \sigma^2}{M}.
\end{equation}
The second inequality of \cref{small_lr} follows immediately once we plug in the specified $\eta$.
\end{proof}

\subsection{Deferred Proof of \cref{small_lr:1}}
\label{sec:proof:small_lr:1}
\begin{proof}[Proof of \cref{small_lr:1}]
The proof of this lemma is very similar to \cref{one:step:analysis:claim:1,one:step:analysis:claim:2} so we skip most of the details.

  We start by analyzing $\tilde{D}_{h_{r+1,0}}(w, \overline{z_{r+1,0}})$. 
  \begin{align}
    & \tilde{D}_{h_{r+1,0}}(w, \overline{z_{r+1,0}})
    \\
    = & h_{r+1,0}(w) - h_{r+1,0} \left( \nabla h_{r+1,0}^*(\overline{z_{r+1,0}}) \right)
    - 
    \left\langle \overline{z_{r+1,0}}, w - \nabla h_{r+1,0}^* (\overline{z_{r+1,0}})  \right\rangle 
    \tag{By definition of generalized Bregman divergence $\tilde{D}$}
    \\
    = & h_{r+1,0}(w) - h_{r+1,0}( \widehat{w_{r+1,0}}) - \left\langle \overline{z_{r+1,0}}, w - \widehat{w_{r+1,0}}\right\rangle 
    \tag{By definition of $\widehat{w_{r+1,0}}$}
    \\
    = & h_{r+1,0}(w) - h_{r+1,0}( \widehat{w_{r+1,0}}) - \left\langle \overline{z_{r,0}} - \eta K \cdot \frac{1}{MK} \sum_{m=1}^M \sum_{k=0}^{K-1} \nabla f(w_{r,k}^m; \xi_{r,k}^m) , w - \widehat{w_{r+1,0}} \right\rangle 
    \tag{By \feddualavg procedure}
    \\
    = & \left( h_{r,0}(w) + \eta K \psi (w) \right)
    - \left(h_{r,0}( \widehat{w_{r+1,0}}) + \eta K \psi (\widehat{w_{r+1,0}})  \right)
    - \left\langle \overline{z_{r,0}} - \eta K \cdot \frac{1}{MK} \sum_{m=1}^M \sum_{k=0}^{K-1} \nabla f(w_{r,k}^m; \xi_{r,k}^m) , w - \widehat{w_{r+1,0}} \right\rangle 
    \tag{By definition of $h_{r+1, 0}$}
    \\
    = & 
    \left( h_{r,0}(w) - h_{r,0}( \widehat{w_{r,0}}) - \left\langle \overline{z_{r,0}}, w - \widehat{w_{r,0}}  \right\rangle \right)
    - 
    \left( h_{r,0}(\widehat{w_{r+1,0}}) - h_{r,0}( \widehat{w_{r,0}}) - \left\langle \overline{z_{r,0}},  \widehat{w_{r+1,0}}  - \widehat{w_{r,0}}  \right\rangle \right)
    \\
    & - \eta K \left( \psi ( \widehat{w_{r+1,0}}) -\psi(w) \right) 
    - \eta K \left\langle \frac{1}{MK} \sum_{m=1}^M \sum_{k=0}^{K-1} \nabla f(w_{r,k}^m; \xi_{r,k}^m) , \widehat{w_{r+1,0}} - w\right\rangle 
    \tag{Rearranging}
    \\
    = & \tilde{D}_{h_{r,0}} (w, \overline{z_{r,0}}) - \tilde{D}_{h_{r,0}} ( \widehat{w_{r+1,0}}, \overline{z_{r,0}}) 
    - \eta K \left( \psi ( \widehat{w_{r+1,0}}) -\psi(w) \right) 
    - \eta K \left\langle \frac{1}{MK} \sum_{m=1}^M \sum_{k=0}^{K-1} \nabla f(w_{r,k}^m; \xi_{r,k}^m) , \widehat{w_{r+1,0}} - w\right\rangle 
    \tag{By definition of $\tilde{D}$}
  \end{align}
  By smoothness and convexity of $F$ we have
  \begin{align}
    F(\widehat{w_{r+1,0}})
    \leq &
    F(\widehat{w_{r,0}}) 
    + \left\langle \nabla F (\widehat{w_{r,0}}), \widehat{w_{r+1,0}}-\widehat{w_{r,0}}  \right\rangle + \frac{L}{2} \|\widehat{w_{r+1,0}}-\widehat{w_{r,0}}\|^2
    \tag{by $L$-smoothness of $F$}
    \\
    \leq & 
    F(w) 
    + \left\langle \nabla F (\widehat{w_{r,0}}), \widehat{w_{r+1,0}}-w  \right\rangle + \frac{L}{2} \|\widehat{w_{r+1,0}}-\widehat{w_{r,0}}\|^2
    \tag{by convexity of $F$}
    \\
  \end{align}
  Combining the above two (in)equalities gives
  \begin{align}
   \tilde{D}_{h_{r+1,0}} (w, \overline{z_{r+1,0}}) -  \tilde{D}_{h_{r,0}} (w, \overline{z_{r,0}}) 
    \leq & - \tilde{D}_{h_{r,0}} ( \widehat{w_{r+1,0}}, \overline{z_{r,0}})  - \eta K \left( \Phi( \widehat{w_{r+1,0}}) - \Phi(w)  \right) + \frac{L}{2} \eta K  \|\widehat{w_{r+1,0}}-\widehat{w_{r,0}}\|^2 
    \\
    & + \eta K 
    \left\langle  \nabla F (\widehat{w_{r,0}}) - \frac{1}{MK} \sum_{m=1}^M \sum_{k=0}^{K-1} \nabla f(w_{r,k}^m; \xi_{r,k}^m), \widehat{w_{r+1,0}}-w  \right\rangle.
  \end{align}
  \end{proof}

\subsection{Deferred Proof of \cref{small_lr:2}}
  \label{sec:proof:small_lr:2}
\begin{proof}[Proof of \cref{small_lr:2}]
We start by splitting the inner product term in the inequality of \cref{small_lr:1}:
\begin{align}
  & \left\langle  \nabla F (\widehat{w_{r,0}}) - \frac{1}{MK} \sum_{m=1}^M \sum_{k=0}^{K-1} \nabla f(w_{r,k}^m; \xi_{r,k}^m), \widehat{w_{r+1,0}}-w  \right\rangle
  \\
  = &
  \underbrace{
  \left\langle  \nabla F (\widehat{w_{r,0}}) - \frac{1}{MK} \sum_{m=1}^M \sum_{k=0}^{K-1} \nabla f(\widehat{w_{r,0}}; \xi_{r,k}^m), \widehat{w_{r,0}}-w  \right\rangle}_{\text{(I)}}
  \\ &
   + 
  \underbrace{
    \left\langle  \nabla F (\widehat{w_{r,0}}) - \frac{1}{MK} \sum_{m=1}^M \sum_{k=0}^{K-1} \nabla f(\widehat{w_{r,0}}; \xi_{r,k}^m), \widehat{w_{r+1,0}}-\widehat{w_{r,0}}  \right\rangle}_{\text{(II)}}
  \\ &
  +
  \underbrace{\frac{1}{MK} \sum_{m=1}^M \sum_{k=0}^{K-1} 
  \left\langle  \nabla f(\widehat{w_{r,0}}; \xi_{r,k}^m) - \nabla f(w_{r,k}^m; \xi_{r,k}^m), \widehat{w_{r+1,0}}-w  \right\rangle}_{\text{(III)}}.
\end{align}
Now we investigate the terms (I)-(III). 
By conditional independence we know $\expt[\text{(I)} | \mathcal{F}_{r,0}] = 0$. For (II), we know that
\begin{align}
 \expt \left[ \text{(II)} \middle| \mathcal{F}_{r,0} \right] 
 \leq &
 \expt \left[ \left\|\nabla F (\widehat{w_{r,0}}) - \frac{1}{MK} \sum_{m=1}^M \sum_{k=0}^{K-1} \nabla f(\widehat{w_{r,0}}; \xi_{r,k}^m) \right\|_*\middle | \mathcal{F}_{r,0} \right]
 \expt \left[  \left\|  \widehat{w_{r+1,0}}-\widehat{w_{r,0}}   \right\|\middle| \mathcal{F}_{r,0} \right]
 \\
 \leq & 
 \frac{\sigma}{\sqrt{MK}} \cdot \expt \left[  \left\|  \widehat{w_{r+1,0}}-\widehat{w_{r,0}}   \right\|\middle| \mathcal{F}_{r,0} \right]
\end{align}
For (III) we observe that (by smoothness assumption)
\begin{align}
  \text{(III)} \leq &  \frac{1}{MK} \sum_{m=1}^M \sum_{k=0}^{K-1}  \left\| \nabla f(\widehat{w_{r,0}}; \xi_{r,k}^m) - \nabla f(w_{r,k}^m; \xi_{r,k}^m) \right\|_*  \|\widehat{w_{r+1,0}} - w\| 
  \\
  \leq & \frac{L}{MK} \sum_{m=1}^M \sum_{k=0}^{K-1}  \left\| \widehat{w_{r,0}} - w_{r,k}^m \right\|  \|\widehat{w_{r+1,0}} - w\|.
\end{align}
Taking conditional expectation,
\begin{align}
  \expt   \left[ \text{(III)} \middle| \mathcal{F}_{r,0} \right] \leq &  
  \frac{1}{MK} \sum_{m=1}^M \sum_{k=0}^{K-1}  \expt \left[ \left\| \nabla f(\widehat{w_{r,0}}; \xi_{r,k}^m) - \nabla f(w_{r,k}^m; \xi_{r,k}^m) \right\|_* \middle| \mathcal{F}_{r,0} \right]
  \expt \left[ \|\widehat{w_{r+1,0}} - w\| \middle| \mathcal{F}_{r,0} \right]
  \\ 
  \leq & \frac{L}{MK} \left(  \sum_{m=1}^M \sum_{k=0}^{K-1} \expt \left[ \|\widehat{w_{r,0}} - w_{r,k}^m \| \middle| \mathcal{F}_{r,0} \right] \right) 
  \expt \left[ \|\widehat{w_{r+1,0}} - w\| \middle| \mathcal{F}_{r,0} \right]
  \\
\end{align}
Combining the above inequalities with \cref{small_lr:1} gives
\begin{align}
   & \expt \left[ \tilde{D}_{h_{r+1,0}} (w, \overline{z_{r+1,0}}) \middle | \mathcal{F}_t \right] -  \tilde{D}_{h_{r,0}} (w, \overline{z_{r,0}}) 
   \\
   \leq & 
   - \eta K \expt \left[ \left( \Phi( \widehat{w_{r+1,0}}) - \Phi(w)  \right) \middle| \mathcal{F}_{r,0} \right]
  - \left( \frac{1}{2} - \frac{L}{2} \eta K  \right) \expt \left[  \|\widehat{w_{r+1,0}}-\widehat{w_{r,0}}\|^2  \middle| \mathcal{F}_{r,0} \right]
   \\ & + \frac{\eta \sigma \sqrt{K}}{\sqrt{M}}  \cdot \expt \left[  \left\|  \widehat{w_{r+1,0}}-\widehat{w_{r,0}}   \right\|\middle| \mathcal{F}_{r,0} \right] 
   +  \frac{\eta L}{M} \left(  \sum_{m=1}^M \sum_{k=0}^{K-1} \expt \left[ \|\widehat{w_{r,0}} - w_{r,k}^m \| \middle| \mathcal{F}_{r,0} \right] \right) 
   \expt \left[ \|\widehat{w_{r+1,0}} - w\| \middle| \mathcal{F}_{r,0} \right]
 \end{align}
 Note that 
 \begin{align}
  & - \left( \frac{1}{2} - \frac{L}{2} \eta K  \right) \expt \left[  \|\widehat{w_{r+1,0}}-\widehat{w_{r,0}}\|^2  \middle| \mathcal{F}_{r,0} \right]
  + \frac{\eta \sigma \sqrt{K}}{\sqrt{M}} \cdot \expt \left[  \left\|  \widehat{w_{r+1,0}}-\widehat{w_{r,0}}   \right\|\middle| \mathcal{F}_{r,0} \right] 
  \\
  \leq & - \frac{3}{8} \expt \left[  \|\widehat{w_{r+1,0}}-\widehat{w_{r,0}}\|^2  \middle| \mathcal{F}_{r,0} \right]
  + \frac{\eta \sigma \sqrt{K}}{\sqrt{M}} \cdot \expt \left[  \left\|  \widehat{w_{r+1,0}}-\widehat{w_{r,0}}   \right\|\middle| \mathcal{F}_{r,0} \right]    \tag{since $\eta \leq \frac{1}{4KL}$}
  \\
  \leq & - \frac{1}{4} \expt \left[  \|\widehat{w_{r+1,0}}-\widehat{w_{r,0}}\|^2  \middle| \mathcal{F}_{r,0} \right] + \frac{2\eta^2 K \sigma^2}{M}.
  \tag{by quadratic optimum}
 \end{align}
 Therefore
 \begin{align}
  & \expt \left[ \tilde{D}_{h_{r+1,0}} (w, \overline{z_{r+1,0}}) \middle | \mathcal{F}_t \right] -  \tilde{D}_{h_{r,0}} (w, \overline{z_{r,0}}) 
  \\
  \leq & 
  - \eta K \expt \left[ \left( \Phi( \widehat{w_{r+1,0}}) - \Phi(w)  \right) \middle| \mathcal{F}_{r,0} \right]
  - \frac{1}{4} \expt \left[  \|\widehat{w_{r+1,0}}-\widehat{w_{r,0}}\|^2  \middle| \mathcal{F}_{r,0} \right] 
    + \frac{2 \eta^2 K \sigma^2}{M} 
  \\
  & +  \frac{\eta L}{M} \left(  \sum_{m=1}^M \sum_{k=0}^{K-1} \expt \left[ \|\widehat{w_{r,0}} - w_{r,k}^m \| \middle| \mathcal{F}_{r,0} \right] \right) 
  \expt \left[ \|\widehat{w_{r+1,0}} - w\| \middle| \mathcal{F}_{r,0} \right].
\end{align}
Since $w_{r,k}^m$ is generated by running local dual averaging with learning rate $\eta_{\client}$, one has
\begin{equation}
  \lim_{\eta_{\client} \downarrow 0} \left[ \left(  \sum_{m=1}^M \sum_{k=0}^{K-1} \expt \left[ \|\widehat{w_{r,0}} - w_{r,k}^m \| \middle| \mathcal{F}_{r,0} \right] \right) 
  \expt \left[ \|\widehat{w_{r+1,0}} - w\| \middle| \mathcal{F}_{r,0} \right] \right] = 0.
  \label{eq:limit}
\end{equation}
There exists an upper bound $\eta_{\client}^{\max}$ such that for any $\eta_{\client} \in (0, \eta_{\client}^{\max}]$, it is the case that 
\begin{equation}
   \left(  \sum_{m=1}^M \sum_{k=0}^{K-1} \expt \left[ \|\widehat{w_{r,0}} - w_{r,k}^m \| \middle| \mathcal{F}_{r,0} \right] \right) 
  \expt \left[ \|\widehat{w_{r+1,0}} - w\| \middle| \mathcal{F}_{r,0} \right] 
  \leq
  \frac{\eta K \sigma^2}{L}.
\end{equation}
Therefore, for any $\eta_{\client} \in (0, \eta_{\client}^{\max}]$,
\begin{align}
   \expt \left[ \tilde{D}_{h_{r+1,0}} (w, \overline{z_{r+1,0}}) \middle | \mathcal{F}_t \right] -  \tilde{D}_{h_{r,0}} (w, \overline{z_{r,0}}) 
  \leq
  - \eta K \expt \left[ \Phi( \widehat{w_{r+1,0}}) - \Phi(w)  \middle| \mathcal{F}_{r,0} \right]
  + \frac{3 \eta^2 K \sigma^2}{M}.
\end{align}
\end{proof}